\documentclass[twoside]{article}
\usepackage{amsmath}
\usepackage[preprint]{jmlr2e}
\usepackage{epsfig}
\usepackage{natbib}
\usepackage{subcaption}         
\usepackage{amsfonts}      
\usepackage{xcolor}      
\usepackage{hyperref}
\usepackage{amsmath}  %
\usepackage{amssymb}  %
\usepackage{bm}       %
\usepackage{graphicx} %
\usepackage{hyperref} %
\usepackage{siunitx}  %
\usepackage{float}    %
\usepackage{caption}  %
\usepackage{subcaption} %
\usepackage{enumitem} %
\usepackage{booktabs} %
\usepackage{verbatim} %
\usepackage{multimedia}
\usepackage{xcolor}
\usepackage{algorithm}
\usepackage{algorithmic}
\usepackage{multirow}
\usepackage{stmaryrd}
\usepackage{tikz-cd}
\usepackage{bbm}
\usepackage{xstring}

\usepackage{pifont} %

\renewcommand{\hat}{\widehat}
\renewcommand{\bar}{\overline}

\newcommand{\normal}{\mathsf{N}}

\newcommand{\sA}{\mathsf{A}}
\newcommand{\E}{\mathbb{E}}

\newcommand{\sT}{\mathsf{T}}

\newcommand{\sFST}{\mathsf{F}^{\textrm {ST}}}

\newcommand{\pp}{\mathfrak{p}}

\newcommand{\fA}{\mathfrak{A}}
\newcommand{\fT}{\mathfrak{T}}
\newcommand{\fG}{\mathfrak{G}}

\newcommand{\sB}{\mathsf{B}}
\newcommand{\sG}{\mathsf{G}}

\renewcommand{\epsilon}{\varepsilon}

\newcommand{\cP}{\mathcal{P}}

\newcommand{\bbR}{\mathbb{R}}
\newcommand{\bbP}{\mathbb{P}}

\newcommand{\bbN}{\mathbb{N}}
\newcommand{\bbE}{\mathbb{E}}

\definecolor{light}{rgb}{0.5, 0.5, 0.5}

\definecolor{mypink1}{rgb}{0.858, 0.188, 0.478}

\definecolor{darkred}{rgb}{0.5,0,0}
\definecolor{darkgreen}{rgb}{0,0.75,0}

\newtheorem{assumption}{Assumption}

\newcommand{\lemusage}[1]{\par\noindent{\textbf{Usage note.} #1}}

\usepackage{lastpage}

\ShortHeadings{Measure Neural Mappings}{}
\firstpageno{1}

\title{Measure Neural Mappings}
\author{\name Ricardo Baptista \email r.baptista@utoronto.ca \\
       \addr 
    University of Toronto
       \AND
\name Edoardo Calvello \email ecalvello@lbl.gov \\
       \addr 
       Lawrence Berkeley National Laboratory
       \AND
       \name Bohan Chen \email bhchen@caltech.edu \\
       \addr 
       California Institute of Technology
       \AND
       \name Andrew Stuart \email astuart@caltech.edu \\
       \addr 
       California Institute of Technology
       \AND
       \name Urbain Vaes \email urbain.vaes@inria.fr \\
       \addr
       Inria Paris, \'Ecole des Ponts
       }

\editor{}

\begin{document}
\maketitle

\begin{abstract}%
We introduce and define measure neural mappings (MNMs), which combine measure-dependent neural transport with {ensemble consistency} for learning operators on probability measures. {Ensemble consistency requires outputs from growing ensembles sampled from an input law to converge to the specified population output in probability.} We establish quantitative continuum limits for attention on noncompact laws under strong sub-Gaussian and tail-integrability assumptions. These laws accommodate unbounded state spaces in scientific machine learning, including Bayesian inference and uncertainty quantification. For fixed self-attention transformers, the expected output Wasserstein error is bounded by the empirical input Wasserstein error plus an additional $N^{-1/2}$ attention term. {These self-attention transformers are MNMs under the stated assumptions: pre-norm architectures follow from compact stability after normalization, while post-norm architectures are covered under the stated tail conditions.} We prove corresponding cross-attention estimates and extend the results to multihead attention and grouped transformer evaluation, with convergence as the group size grows. For Set Transformers viewed as measure-to-vector mappings, we obtain {an $N^{-1/2}$ bound on the expected Euclidean output error under the stated assumptions}. Experiments on seven operators in dimensions $d=1,2,3,5,10$ support the predicted particle-count dependence, using one-dimensional Wasserstein computations and Sinkhorn approximations in higher dimensions. The observed Wasserstein and sliced Wasserstein decay rates are close to their respective theoretical sampling benchmarks.
\end{abstract}

\section{Introduction}
\label{sec:intro}

Ensemble data assimilation and Bayesian inverse problems require operators that transform probability laws. In computation, these laws are often represented by finite particle ensembles. Learning an operator from such data therefore requires consistency across different samples and ensemble sizes: a change in the particle representation should not produce a persistent change in the predicted law. Neural operators address a related requirement for discretizations of functions. Here we formulate and analyze its counterpart for probability measures.

We introduce measure neural mappings (MNMs) for learning nonlinear operators on probability measures on Euclidean spaces. Their neural transport maps depend on the input law, allowing nonlinear transformations of measures. {MNMs require ensemble consistency: output measures from growing ensembles sampled from an input law converge to the specified population output in probability.} These requirements permit a single parameterization to act consistently on probability laws and their particle representations.

{Our central result is that fixed finite, permutation-equivariant transformer stacks that transport the input particles are MNMs under the stated architecture and input-law assumptions.} We establish quantitative convergence as the ensemble size increases. Our main extension beyond compact-support theory treats noncompact input laws under explicit regularity and tail assumptions. Such laws arise naturally in scientific machine learning: Gaussian priors, uncertain initial conditions, and filtering distributions on unbounded state spaces cannot be enclosed in a fixed compact set without truncation. {For pre-norm stacks with normalized query, key, and value branches and Lipschitz feed-forward maps, finite $p$th moments suffice, by compact stability after normalization (Remark~\ref{rem:post_pre_norm}). Post-norm stacks are covered under the stated strong sub-Gaussian and tail-integrability conditions (Theorem~\ref{thm:transformer_stack_limit}).} We also study grouped evaluation and measure-to-vector Set Transformers, and compare the predicted sampling rates with numerical experiments.

\subsection{Literature Review}
\label{subsec:litreview}

Attention enters the present problem at two levels. For a finite ensemble, it defines interactions among samples. At the population level, identifying the ensemble with its empirical measure turns the same computation into a measure-dependent operator. We first review attention and its main applications, then discuss formulations on spaces of probability measures, and finally summarize the theoretical results most closely related to the continuum limit studied here.

\subsubsection{Attention and Transformer Architectures}
\label{subsubsec:attention}

Neural attention was introduced in sequence-to-sequence models as a data-dependent aggregation of encoder states \citep{bahdanau2015neural}. The Transformer made self-attention the main interaction mechanism and removed recurrence from the architecture \citep{vaswani2017attention}. Applications include language representation, generation, and instruction following \citep{devlin2019bert,brown2020language,ouyang2022training,openai2022chatgpt,openai2023gpt4,touvron2023llama}, as well as image recognition, object detection, vision--language learning, and image segmentation \citep{dosovitskiy2021image,liu2021swin,carion2020end,radford2021learning,kirillov2023segment}.

For unordered inputs, the relevant structural property is permutation symmetry. Deep Sets and Set Transformers provide architectures for permutation-invariant and permutation-equivariant learning on finite collections \citep{zaheer2017deep,lee2019set}. In particular, self-attention without positional information is permutation equivariant and can be evaluated on different numbers of tokens. This makes it suitable for processing particle ensembles, although permutation equivariance alone does not guarantee convergence as the ensemble grows.

\subsubsection{Attention and Operators on Probability Measures}
\label{subsubsec:measure_operators}

Once a token collection is represented by its empirical measure, both the attention weights and the token updates depend on the full empirical law. Two early formulations made this observation explicit: \citet{vuckovic2020mathematical} treated attention as a nonlinear operator on probability measures, while \citet{sander2022sinkformers} defined softmax and Sinkhorn attention for arbitrary input measures and recovered the usual finite-token formulas on empirical measures.

Related constructions learn solution, transport, or evolution operators indexed by probability laws in distribution-on-distribution regression, meta optimal transport, mean-field games, optimal transport, and Wasserstein gradient flows \citep{ghodrati2022distribution,amos2023meta,huang2025unsupervised,cole2026incontext,feng2026learn,geshkovski2026measure,vandergrift2026measure}. Ensemble Bayesian filtering provides a direct application because both the forecast and analysis distributions are represented by finite ensembles; recent transformer-based filters therefore use a single permutation-equivariant parameterization across different ensemble sizes \citep{bach2026learning,bach2026proper}.

\subsubsection{Related Theoretical Results}
\label{subsubsec:theory}

One group of results concerns stability in the input law and convergence from empirical to population attention. Under compactness and regularity assumptions, the Wasserstein stability estimate of \citet{vuckovic2020mathematical}, together with standard convergence of empirical measures, gives the compact-support case of the continuum limit considered here. Quantitative Lipschitz and propagation-of-chaos estimates have since been developed for individual attention layers and deep contextual flows \citep{castin2024smooth,chen2026propagation}. Closely related sampling-consistency results appear in attention-based solution operator learning and in the mean-field analysis of learned filtering objectives \citep{huang2025unsupervised,bach2026proper}.

{
\citet{bohbot2026token} study how attention parameters and token geometry determine token sample complexity, with convergence estimates for attention maps and output statistics under sub-Gaussian laws. Our results concern convergence of the full empirical output measures in expected $W_p$, with separate input-sampling and $N^{-1/2}$ attention terms under explicit tail and regularity assumptions on noncompact inputs. The estimates follow the same sampled ensemble through a transformer stack, accounting for the dependence introduced by successive attention layers.\par
}

{In concurrent work, \citet{cole2026stability} establish $W_1$ bounds for measure-to-measure transformers that are closely related to our results; we additionally analyze full transformer architectures incorporating layer normalization and provide numerical experiments using Wasserstein and sliced Wasserstein distances.}

A second line of theory treats network depth as a time variable and interprets self-attention as an interacting-particle system. The corresponding mean-field equations have been used to study clustering, metastability, normalization, long-time behavior, and different attention scalings \citep{geshkovski2023emergence,bruno2025metastable,chen2025quantitative,burger2025analysis,castin2025unified,rigollet2025meanfield,geshkovski2025mathematical}. These works use a population-level description similar to the one considered here, but their main asymptotic variable is depth or time rather than the number of samples representing the input law.

A third line of work concerns expressivity. Transformers can approximate continuous in-context maps and broad classes of mean-field vector fields \citep{furuya2025universal,biswal2025universal}, and can approximately interpolate finitely many prescribed pairs of probability measures \citep{geshkovski2026measure}. In concurrent work, \citet{furuya2026pushforwards} establish universal approximation by measure-dependent pushforwards under suitable input assumptions. More generally, \citet{lavenant2026continuous} give conditions under which a continuous transformation of probability measures admits a measure-dependent transport representation, which is the structural form used by the pushforward operators studied in this paper.

The broader operator-learning literature provides a useful comparison. Random-feature operators and neural operators approximate nonlinear maps between function spaces while sharing parameters across discretizations \citep{nelsen2021random,kovachki2023neural,calvello2025continuum}. Complementary approximation and statistical results for transport between probability distributions have been developed using neural differential equations \citep{marzouk2024distribution,marzouk2025minimal}. In the present setting, spatial discretization is replaced by sampling: the learned operator should remain consistent when the same input law is represented by a different empirical ensemble. 

\subsection{Contributions and Outline}
\label{subsec:contributions}

We develop a framework for operator learning on probability measures and establish the {ensemble consistency} of attention-based architectures. Our contributions are as follows.

\begin{enumerate}
    \item {(Measure Neural Mappings) We define an MNM as a measure-dependent neural pushforward satisfying ensemble consistency: for each fixed network parameter and admissible input law, the empirical output converges in probability to the specified population output.}
    \item {(Transformers as Measure Neural Mappings) We prove that fixed finite transformer stacks are MNMs when they are permutation equivariant, transport the input particles, and satisfy the stated assumptions. The results include residual connections, Lipschitz feed-forward maps, layer normalization, and any fixed finite number of attention heads. For pre-norm stacks with normalized query, key, and value branches, finite $p$th moments suffice. Post-norm stacks are covered under the stated strong sub-Gaussian and tail-integrability conditions.}
    \item {(Quantitative Limits on Noncompact Laws) We extend compact-support attention theory to noncompact laws under explicit strong sub-Gaussian and tail-integrability conditions. For self-attention, the expected output $W_p$ error is bounded by $C\bbE W_p(\mu^N,\mu)+CN^{-1/2}$. We prove corresponding cross-attention estimates, including an $N^{-1/2}$ bound when prescribed queries attend to sampled keys. The stack analysis propagates these bounds without assuming independence of the intermediate particles.}
    \item {(Grouped Evaluation and Measure-to-Vector Outputs) For grouped self-attention stacks with $G$ groups of $N/G$ particles, we bound the attention error by $C\sqrt{G/N}$ while retaining the input sampling error from all $N$ particles; convergence holds when $G=o(N)$. For Set Transformers with a self-attention encoder, fixed-seed pooling, and a Lipschitz readout, we prove an $N^{-1/2}$ expected feature-error bound under the stated hypotheses.}
    \item {(Numerical Convergence Rates) Experiments on seven operators in dimensions $d=1,2,3,5,10$ support the predicted Wasserstein sampling rates, using $W_2$ computations in one dimension and Sinkhorn approximations in higher dimensions. Sliced Wasserstein errors decay close to $N^{-1/2}$ across all tested dimensions; this rate is an empirical finding beyond the $W_p$ bounds proved here.}
\end{enumerate}

Section~\ref{sec:MNM} defines MNMs. Section~\ref{sec:transformer_mnm} develops the attention and transformer convergence results, including grouped evaluation, the multihead extension, and Set Transformers. Section~\ref{sec:experiments} tests the rates numerically, and Section~\ref{sec:conclusion} discusses limitations and future directions. The appendices provide supporting estimates and full experimental details.

\subsection{Notation}
\label{subsec:Notation}

We write $\bbR$ for the real numbers and $\bbN=\{1,2,\ldots\}$ for
the positive integers. For $N\in\bbN$, let $[N]=\{1,2,\ldots,N\}$.
We use $\langle\cdot,\cdot\rangle$ for the Euclidean inner product and
$\|\cdot\|_2$ for the Euclidean vector norm or the induced spectral norm
of a matrix. The closed ball of radius $R$ centered at the origin is
$B_R(0)$. Constants denoted by $C$ may change between estimates; their
permitted dependencies are stated where needed.

For a metric space $E$, let $\cP(E)$ denote the Borel probability measures
on $E$, equipped with the topology of weak convergence. We write
$X\sim\mu$ when $X$ has law $\mu$, and use $\bbP$ and $\bbE$ for
probability and expectation. For $p\geq1$, the space
$\cP_p(\bbR^d)$ consists of the laws $\mu$ with finite moment
$\int_{\bbR^d}\|x\|_2^p\mu(dx)$. On this space we use the
Wasserstein distance
\begin{equation}
    W_p(\mu,\nu)
    =\left(\inf_{\pi\in\Pi(\mu,\nu)}
    \int_{\bbR^d\times\bbR^d}\|x-y\|_2^p\pi(dx,dy)\right)^{1/p},
\end{equation}
where $\Pi(\mu,\nu)$ is the set of probability measures with marginals
$\mu$ and $\nu$.

For a measurable map $T$, the pushforward $T_\sharp\mu$ is the law of
$T(X)$ when $X\sim\mu$. We denote the Dirac mass at $x$ by $\delta_x$
and the Gaussian law with mean $m$ and covariance $\Sigma$ by
$\normal(m,\Sigma)$. Sans-serif symbols such as $\sG$, $\sA$, and
$\sB$ denote operators acting on probability measures; their outputs
may be measures or feature vectors. Fraktur symbols such as $\fG$
and $\fA$ denote the associated pointwise transport maps when a
pushforward representation is used.

{In a map written as $F:X\to Y$, $X$ denotes the ambient input space; the actual domain may be restricted to a subset of $X$ on which the defining expressions are well defined.}

The noncompact attention estimates use the following exponential-tail
and sub-Gaussian classes.

\begin{definition}[$\alpha$-Exponential Tail Decay]
\label{def:alpha_exp_tail_decay}
For $\alpha,\tau>0$, let $\cP_{\alpha,\tau}(\bbR^d)$ consist of the
laws $\mu\in\cP(\bbR^d)$ satisfying
\begin{equation}
    \mu\bigl(\{x:\|x\|_2\geq t\}\bigr)
    \leq 2\exp(-t^\alpha/\tau),\qquad t\geq0.
\end{equation}
\end{definition}

\begin{definition}[Moment and Cumulant Generating Functions]
\label{def:MGF_CGF}
For $\mu\in\cP(\bbR^d)$, its moment generating function and cumulant
generating function are
\begin{equation}
    \mathcal M_\mu(\lambda)
    =\int_{\bbR^d}e^{\langle\lambda,x\rangle}\mu(dx),\qquad
    \psi_\mu(\lambda)=\log\mathcal M_\mu(\lambda),
\end{equation}
wherever the integral is finite. For $X\sim\mu$, we also write
$\psi_X=\psi_\mu$.
\end{definition}

\begin{definition}[Sub-Gaussian]
\label{def:subgaussian}
For $a>0$, a law $\mu\in\cP_1(\bbR^d)$ belongs to
$\mathcal{SG}_a(\bbR^d)$ if, for $X\sim\mu$,
\begin{equation}
    \bbE\exp\bigl(\langle\lambda,X-\bbE X\rangle\bigr)
    \leq\exp\left(\frac{a^2\|\lambda\|_2^2}{2}\right),
    \qquad \lambda\in\bbR^d.
\end{equation}
We then call $X$ and its law sub-Gaussian with parameter $a$.
\end{definition}

\begin{definition}[Strong Sub-Gaussian]
\label{def:strong_sub_gaussian}
For $a>0$, a law $\mu\in\cP(\bbR^d)$ belongs to
$\mathcal{SSG}_a(\bbR^d)$ if its cumulant generating function is finite
and twice continuously differentiable on $\bbR^d$, with
\begin{equation}
    \|\nabla^2\psi_\mu(\lambda)\|_2\leq a^2,
    \qquad \lambda\in\bbR^d.
\end{equation}
We then call $\mu$ strong sub-Gaussian with parameter $a$.
\end{definition}

A random probability measure is a measurable random variable taking values
in a space of probability measures, with the Borel sigma-algebra of the
topology specified above. This includes the random input and output measures of
Section~\ref{sec:MNM} and the empirical measures
$\mu^N=N^{-1}\sum_{i=1}^N\delta_{X_i}$ used throughout the convergence
analysis. When $X_1,\ldots,X_N$ are i.i.d. with law
$\mu\in\cP_p(\bbR^d)$, the measure $\mu^N$ is random even though $\mu$
is fixed. An expression such as $\bbE W_p(\mu^N,\mu)$ averages over
these samples and measures the expected sampling error.

\section{Measure Neural Mappings}
\label{sec:MNM}

In this paper we consider the task of approximating operators of the form $\sG^\dagger:\cP(\bbR^{d_1})\to\cP(\bbR^{d_2})$. Emphasis is given to the general setting of measures supported on a possibly unbounded domain; this is of particular interest for practical settings like Bayesian inference. We consider the standard data scenario where we assume access to a set $\{\mu^{(j)}, \pi^{(j)} \}_{j=1}^J$ of input-output observation pairs such that $\bigl(\mu^{(j)},\pi^{(j)}\bigr)\sim\nu$ is an i.i.d. sequence distributed according to the measure $\nu$, supported on $\cP(\bbR^{d_1})\times \cP(\bbR^{d_2}) $. We let $\nu_\mu$ and $\nu_\pi$ denote the marginals of $\nu$ on $\cP(\bbR^{d_1})$ and $\cP(\bbR^{d_2})$ respectively, so that $\mu^{(j)}\sim \nu_\mu$ and $\pi^{(j)}\sim \nu_\pi$. In what follows, we consider
\begin{equation}
\label{eq:data_model}
\pi^{(j)} = \sG^\dagger\bigl(\mu^{(j)}\bigr).
\end{equation}
We
assume for simplicity that the data generation mechanism is compatible with \eqref{eq:data_model} so that $\nu_\pi = \sG^\dagger_\sharp \nu_\mu$, i.e., the data generation mechanism is not corrupted by noise. Given this setting, we may consider using supervised learning to construct an approximation of $\sG^\dagger$ picked from the parametric class
$$\sG:\cP(\bbR^{d_1})\times {\mathcal H} \to \cP(\bbR^{d_2}),$$ or
equivalently, $\sG_\theta:\cP(\bbR^{d_1})\to\cP(\bbR^{d_2}),$ for $\theta\in{\mathcal H}$. The set ${\mathcal H}$ is assumed to be a finite dimensional parameter space from which a $\theta^*\in{\mathcal H}$ is selected so that $\sG(\cdot,\theta^*) = \sG_{\theta^*}\approx \sG^\dagger$. In order to do so, we define a cost functional $c:\cP(\bbR^{d_2})\times\cP(\bbR^{d_2})\to \bbR$ and define $\theta^*$ by
\begin{equation}
\label{eq:data_assimilation_model}
\theta^*={\arg \min}_{\theta\in{\mathcal H}}\E_{\mu\sim\nu_\mu}\Bigl[ c\bigl(\sG(\mu,\theta), \sG^\dagger(\mu) \bigr)\Bigr].
\end{equation}
{In practice, we observe only samples $\{u_i^{(j)}\}_{i=1}^N$ and $\{z_i^{(j)}\}_{i=1}^N$ from each input-output pair. We therefore seek a measure-dependent pointwise network $\fG(\cdot;\mu,\theta)$ whose pushforward approximates the target operator:
\[
    \sG_\theta(\mu)=\fG(\cdot;\mu,\theta)_\sharp\mu
    \approx\sG^\dagger(\mu).
\]
This representation allows training and evaluation directly on particle ensembles.}
The learned parameters $\theta^*$ should be shared across particle representations of the same input law. {As an ensemble grows, its output measure should converge to the specified population output. We call this property \emph{ensemble consistency};} it is the sampling analogue of discretization invariance for neural operators.

{
\begin{definition}[Ensemble Consistency]
\label{def:ensemble_consistency}
Given a metric $\mathrm{d}$, for each fixed parameter $\theta$, an operator $\sG_\theta:\cP(\bbR^{d_1})\to\cP(\bbR^{d_2})$ is \emph{ensemble consistent} on an admissible class of input laws if, for each law $\mu$ in that class and each empirical measure $\mu^N = \frac1N\sum_{i=1}^N\delta_{u_i}$ formed from $u_i\sim\mu$ i.i.d., the output evaluated on the empirical input converges to the specified population output:
\begin{equation}\label{eq:ensemble_consistency}
    \mathrm d\bigl(\sG_\theta(\mu^N),\sG_\theta(\mu)\bigr)
    \xrightarrow{\mathbb P}0,
    \qquad N\to\infty.
\end{equation}
The operator is defined on both $\mu$ and its empirical representations.
\end{definition}
}
{For Wasserstein-$p$ convergence we work in $\cP_p$. {The estimates below establish convergence of the expected distance, which implies the convergence in probability required in Definition~\ref{def:ensemble_consistency}: by Markov's inequality, for every $\varepsilon>0$,}
{
\[
    \mathbb P\!\left(W_p\bigl(\sG_\theta(\mu^N),\sG_\theta(\mu)\bigr)>\varepsilon\right)
    \leq \frac{\bbE W_p\bigl(\sG_\theta(\mu^N),\sG_\theta(\mu)\bigr)}{\varepsilon}
    \longrightarrow 0.
\]
}}

{
\noindent\textbf{Remark (Pairwise ensemble invariance).}
Ensemble consistency implies that for any two empirical measures $\mu^N$ and $\widetilde\mu^M$ formed from i.i.d. samples of the same law $\mu$,
\begin{equation}
    \label{eq:ensemble_inv}
    \mathrm{d}\Bigl(\sG_\theta\bigl(\mu^N\bigr), \sG_\theta\bigl({\widetilde\mu^M}\bigr)\Bigr){\xrightarrow{\mathbb P}} 0,\quad \text{as }N,M\to\infty.
\end{equation}
This follows by the triangle inequality through $\sG_\theta(\mu)$. We call this consequence \emph{pairwise ensemble invariance}; alone, it does not identify the limit with $\sG_\theta(\mu)$.
}

A fixed Lipschitz neural network already gives an {ensemble-consistent} pushforward operator in Wasserstein-$p$ distance.

{\begin{proposition}[Lipschitz Neural Network]
\label{prop:lipschitz_nn}
Fix $p\geq1$ and $\theta\in\mathcal H$. If $\fG(\cdot;\theta):\bbR^{d_1}\to\bbR^{d_2}$ is globally Lipschitz, then
\[
    \sG_\theta(\mu)=\fG(\cdot;\theta)_\sharp\mu,
    \qquad \mu\in\cP_p(\bbR^{d_1}),
\]
is {ensemble consistent} under $W_p$.
\end{proposition}
\begin{proof}
Writing $C=\mathrm{Lip}(\fG(\cdot;\theta))$, the pushforward coupling gives
\[
    W_p(\sG_\theta(\mu^N),\sG_\theta(\mu))\leq C W_p(\mu^N,\mu).
\]
For $\mu\in\cP_p$, empirical measures converge to $\mu$ in $W_p$ almost surely and in expected distance. {The displayed bound proves the claim.}
\end{proof}}

Pushforward by a fixed map preserves convex mixtures of measures. To represent nonlinear operators, such as the normalized update in Bayesian inference, the transport map must therefore depend on the input law. This motivates the following definition.
\begin{definition}[Measure-dependent Neural Network]
\label{def:measure_dependence}
    A neural network $\fG$ is measure-dependent if it may be written as a map $\fG:\bbR^{d_1}\times \cP(\bbR^{d_1})\to \bbR^{d_2}$.
\end{definition}
This leads to the definition of a measure neural mapping.

{\begin{definition}[Measure Neural Mapping]
An operator of the form
\[
    \sG_\theta(\mu)=\fG(\cdot;\mu,\theta)_\sharp\mu,
\]
where $\fG$ is a measure-dependent neural network, is a \emph{measure neural mapping} on an admissible class of input laws {if it satisfies ensemble consistency in Definition~\ref{def:ensemble_consistency} on that class.}
\end{definition}
{Lipschitz estimates provide one route to ensemble consistency; they are not required by the definition. For unbounded input laws, the main difficulty is to obtain pointwise estimates that remain integrable in the query variable.}}

{In the following section we establish {ensemble consistency} for transformer maps on probability measures. Stacks that transport the input particles give measure neural mappings in the sense above; fixed-query blocks and feature pooling are also covered by convergence results, with their output spaces specified separately.}

\section{Transformer Measure Neural Mappings}
\label{sec:transformer_mnm}

This section establishes continuum limits and {ensemble consistency} for fixed
transformers with self-attention and cross-attention. Compact-support
stability of attention is already known \citep{vuckovic2020mathematical};
we extend the analysis to noncompact laws under explicit tail-decay and
strong sub-Gaussian assumptions.

The results cover both pre-norm and post-norm architectures. In pre-norm
attention, layer normalization bounds the query, key, and value inputs, so
the compact estimates already yield the continuum limit. Post-norm attention
can encounter a noncompact law before its first normalization; this is the
additional case addressed by our noncompact estimates and stack theorem.
Remark~\ref{rem:post_pre_norm} gives the corresponding pre-norm estimate,
including residual states with unbounded support.

\paragraph{Informal theorem (Transformer continuum limit).}
Let $\sT$ denote the measure map associated with a fixed finite
transformer, allowing both self-attention and cross-attention with prescribed
auxiliary measures. All parameters and auxiliary measures are independent of
the ensemble size. Fix $p\geq1$ and $\mu\in\cP_p$, and let $\mu^N$ be the empirical measure of
$N$ i.i.d. samples from $\mu$. For a pre-norm stack, assume that all query,
key, and value branches are normalized, the feed-forward maps are Lipschitz,
and the auxiliary measures belong to $\cP_p$, as in
Remark~\ref{rem:post_pre_norm}. For a post-norm stack, impose the architecture
and input-law hypotheses of Theorem~\ref{thm:transformer_stack_limit}. Then
\begin{equation}\label{eq:informal_transformer_limit}
    \bbE W_p\bigl(\sT(\mu^N),\sT(\mu)\bigr)
    \leq C\left(\bbE W_p(\mu^N,\mu)+N^{-1/2}\right)
    \longrightarrow0,
\end{equation}
where $C$ is independent of $N$. Thus evaluating the same network on growing
ensembles approximates a single deterministic output law $\sT(\mu)$.
{By Markov's inequality, this proves ensemble consistency in $W_p$. Stacks that transport the input particles also have the measure-dependent pushforward form $\sT(\mu)=\fT(\cdot;\mu)_\sharp\mu$, and are therefore MNMs.}
For pre-norm stacks the $N^{-1/2}$ term can be omitted. For post-norm stacks,
Theorem~\ref{thm:transformer_stack_limit} gives sharper bounds for the
different block types.

\paragraph{Two single-layer estimates.}
{Theorem~\ref{thm:self_convergence_iid} controls the attention error for
i.i.d. samples from a law with unbounded support under the tail-decay and strong
sub-Gaussian assumptions stated below. Theorem~\ref{thm:self_convergence}
bounds the output Wasserstein distance by the input distance on a fixed
compact set, without requiring independent particles. For a post-norm self-attention stack,
these estimates apply to the first block and the later blocks with compactly supported inputs,
respectively.}

\paragraph{From one block to a stack.}
Fix $p\geq1$. Let $\mu_0^N=N^{-1}\sum_{i=1}^N\delta_{X_i}$, where the $X_i$ are i.i.d.
with law $\mu_0$, and let $\sB_1,\ldots,\sB_L$ be fixed post-norm
self-attention blocks. Each block includes the residual connections,
feed-forward map, and layer normalizations defined in
Subsection~\ref{subsec:transformer_blocks}. Write
$\mu_\ell^N=\sB_\ell(\mu_{\ell-1}^N)$ and
$\mu_\ell=\sB_\ell(\mu_{\ell-1})$.
Figure~\ref{fig:postnorm_proof_strategy} summarizes the two steps.

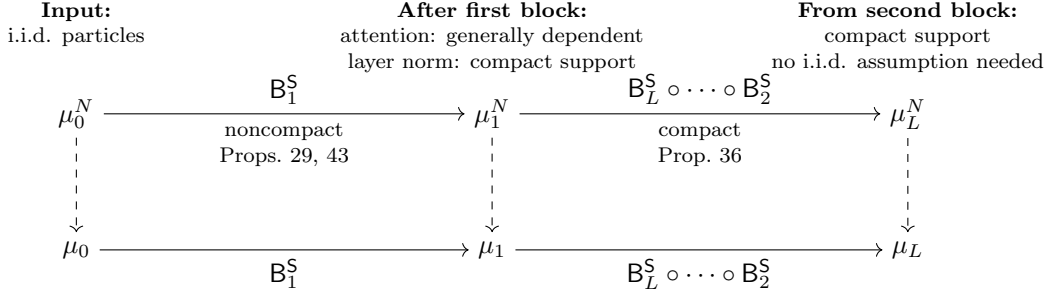
\begin{figure}[!htbp]
\centering
\normalsize
\begin{tikzpicture}
    \node (emp0) at (0,0) {$\mu_0^N$};
    \node (emp1) at (5.5,0) {$\mu_1^N$};
    \node (empL) at (11,0) {$\mu_L^N$};
    \node (pop0) at (0,-1.8) {$\mu_0$};
    \node (pop1) at (5.5,-1.8) {$\mu_1$};
    \node (popL) at (11,-1.8) {$\mu_L$};
    \node[anchor=north, align=center, font=\footnotesize] at (0,1.6)
      {\textbf{Input:}\\i.i.d. particles};
    \node[anchor=north, align=center, font=\footnotesize] at (5.5,1.6)
      {\textbf{After first block:}\\attention: generally dependent\\layer norm: compact support};
    \node[anchor=north, align=center, font=\footnotesize] at (11,1.6)
      {\textbf{From second block:}\\compact support\\no i.i.d. assumption needed};
    \draw[->] (emp0) -- node[above] {$\sB_1^{\mathsf S}$}
      node[below, align=center, font=\footnotesize]
      {noncompact\\Props.~\ref{prop:query_lipschitz}, \ref{prop:second_term}} (emp1);
    \draw[->] (emp1) -- node[above] {$\sB_L^{\mathsf S}\circ\cdots\circ\sB_2^{\mathsf S}$}
      node[below, align=center, font=\footnotesize]
      {compact\\Prop.~\ref{prop:block_compact_lipschitz}} (empL);
    \draw[->] (pop0) -- node[below] {$\sB_1^{\mathsf S}$} (pop1);
    \draw[->] (pop1) -- node[below] {$\sB_L^{\mathsf S}\circ\cdots\circ\sB_2^{\mathsf S}$} (popL);
    \draw[->, dashed] (emp0) -- (pop0);
    \draw[->, dashed] (emp1) -- (pop1);
    \draw[->, dashed] (empL) -- (popL);
\end{tikzpicture}
\caption{Post-norm self-attention stack ($L\geq2$).
Both rows use the same fixed blocks. The first step uses
Theorem~\ref{thm:self_convergence_iid} and its particle-coupling proof.
Layer normalization gives a common deterministic compact support
(Lemma~\ref{lem:layer_norm_bounded}), so later steps use compact stability
(Proposition~\ref{prop:block_compact_lipschitz}) without independence.
Dashed arrows denote convergence in expected $W_p$ as $N\to\infty$.}
\label{fig:postnorm_proof_strategy}
\end{figure}

Under the architecture and first-block assumptions of
Theorem~\ref{thm:transformer_stack_limit},
the noncompact attention coupling estimates extend through the residual and
Lipschitz post-attention maps to give
\begin{equation}\label{eq:proof_strategy_first_block}
    \bbE W_p(\mu_1^N,\mu_1)
    \leq C\left(\bbE W_p(\mu_0^N,\mu_0)
                  +N^{-1/2}\right).
\end{equation}

{Although the particles are generally dependent after attention, layer
normalization gives both $\mu_1^N$ and $\mu_1$ support in a common deterministic
ball, whose radius is independent of $N$ and the sample realization
(Lemma~\ref{lem:layer_norm_bounded}). Each subsequent block also ends in
layer normalization, so Proposition~\ref{prop:block_compact_lipschitz}
gives, for each realization,}
\begin{equation}\label{eq:proof_strategy_later_blocks}
    W_p(\mu_\ell^N,\mu_\ell)
    \leq C_\ell W_p(\mu_{\ell-1}^N,\mu_{\ell-1}),
    \qquad \ell=2,\ldots,L,
\end{equation}
without any independence assumption. For fixed $L$, taking expectations and
iterating multiplies the first-block error by a constant independent of $N$.
Since $\bbE W_p(\mu_0^N,\mu_0)\to0$, the first-block bound tends to zero,
and hence so does $\bbE W_p(\mu_L^N,\mu_L)$.

The following subsections
develop the self- and cross-attention estimates and combine them in ordinary
and grouped stacks. These proofs use single-head attention;
Subsection~\ref{subsec:multihead_extension} extends the results to multiple
heads before the application to Set Transformers in
Subsection~\ref{subsec:set_transformer_limit}.

\subsection{Attention on Measures}\label{subsec:attention_tail}

\begin{definition}[Attention on Measures]\label{def:attention_on_measures}
Let $Q\in\bbR^{d_k\times d_u}$, $K\in\bbR^{d_k\times d_v}$, and $V\in\bbR^{d_V\times d_v}$ be fixed. For $\mu\in\cP(\bbR^{d_u})$ and $\nu\in\cP(\bbR^{d_v})$ such that the integrals below are finite for every query $s$, define the attention law $\sA(\mu,\nu)\in\cP(\bbR^{d_V})$ by
\begin{subequations}
\begin{align}
    \sA(\mu,\nu) &= \fA(\cdot;\nu)_{\sharp}\mu,\\
    \fA(s;\nu) &= \bbE^{Y\sim \pp(\cdot;s,\nu)}[VY]
    = \frac{\int_{\bbR^{d_v}}\exp\left(\langle Qs,Ky\rangle\right)Vy\nu(dy)}{\int_{\bbR^{d_v}}\exp\left(\langle Qs,Ky\rangle\right)\nu(dy)},\label{eq:attention_pointwise_map}\\
    \pp(dy;s,\nu) &= \frac{\exp\left(\langle Qs,Ky\rangle\right)\nu(dy)}{\int_{\bbR^{d_v}}\exp\left(\langle Qs,Ky\rangle\right)\nu(dy)}.
\end{align}
\end{subequations}
\end{definition}

In Definition~\ref{def:attention_on_measures}, the first input $\mu$ is the query measure and the second input $\nu$ is the key measure (also called the value measure). Random variables and empirical samples drawn from $\mu$ are denoted by $X$ and $X_i$, while those drawn from $\nu$ are denoted by $Y$ and $Y_j$. Lowercase $x,y$ denote deterministic points or integration variables, and $s$ is reserved for a fixed pointwise query in $\fA(s;\nu)$.
When $d_u=d_v=d$ and the two inputs are identified, namely $\nu=\mu$, we call this \textbf{self-attention} and write
\begin{equation}
    \sA(\mu)=\sA(\mu,\mu).
\end{equation}
When the two inputs are kept distinct, with no relation between $\mu$ and $\nu$ required, we call $\sA(\mu,\nu)$ \textbf{cross-attention}.

For the single-attention estimates in Subsections~\ref{subsec:self_attention_limits} and~\ref{subsec:cross_attention_limits}, we set $V=I$ to simplify the notation. The general value map is recovered from the identity
\begin{equation}\label{eq:value_map_reduction}
    \fA_{Q,K,V}(s;\nu)
    =
    V\fA_{Q,K,I}(s;\nu).
\end{equation}
Consequently, every estimate involving a difference of attention outputs acquires at most the multiplicative factor $\|V\|_2$. For residual maps, the corresponding Lipschitz constants depend on $\|V\|_2$ through terms of the form $1+\|V\|_2C$. The integrability assumptions and all convergence rates remain unchanged. In Subsection~\ref{subsec:transformer_blocks}, we retain a possibly nonidentity value matrix $V_{\ell}$ in every block. We also write
\begin{equation}\label{eq:Theta_def}
    \Theta=Q^{\top}K.
\end{equation}
Since $\langle Qs,Ky\rangle=\langle s,\Theta y\rangle$, $\Theta\in\bbR^{d_u\times d_v}$ determines the attention weights. A fixed softmax scaling, such as $d_k^{-1/2}$, can be absorbed into $Q$ or $K$.

The domain restriction in Definition~\ref{def:attention_on_measures} is essential for unbounded key measures. If $\nu\in\cP_{1+\delta,\tau}(\bbR^{d_v})$ for some $\delta,\tau>0$, then its linear exponential moments, including those weighted by $\|y\|_2$, are finite. The denominator is therefore positive and finite and the numerator is absolutely integrable for every $s$. The pointwise map is continuous, so its pushforward is defined for any query law $\mu$. This alone does not guarantee that the output has a finite $p$th moment.

The noncompact convergence results below impose separate conditions on the
keys and queries: strong sub-Gaussianity of the key law controls query
stability, while query-tail bounds allow the empirical-key error to be
integrated. We state these conditions alongside the corresponding results.

\subsection{Self-Attention Limits}\label{subsec:self_attention_limits}

We first recall deterministic stability on a fixed compact set, the setting
of the existing compact-support theory \citep{vuckovic2020mathematical}.
We include a proof in our notation before presenting the noncompact extension.

\begin{theorem}[Compact Self-Attention]\label{thm:self_convergence}
Fix $p\geq 1$. Let $\mu,\nu\in\cP(B_R(0))$ on $\bbR^d$. Then there exists a constant $C=C(\Theta,R)>0$ such that
\begin{equation}
    W_p\left(\sA(\mu),\sA(\nu)\right)
    \leq
    C W_p(\mu,\nu).
\end{equation}
\end{theorem}

\begin{proof}
Using the pushforward definition, the triangle inequality gives
\begin{equation}\label{eq:triangle_thm2}
    W_p\left(\sA(\mu),\sA(\nu)\right)
    \leq
    W_p\left(\fA(\cdot;\mu)_{\sharp}\mu,\fA(\cdot;\mu)_{\sharp}\nu\right)
    +
    W_p\left(\fA(\cdot;\mu)_{\sharp}\nu,\fA(\cdot;\nu)_{\sharp}\nu\right).
\end{equation}
Since $\mu\in\cP(B_R(0))$, Lemma~\ref{lem:bounded_support_ssg} implies $\mu\in\mathcal{SSG}_R(\bbR^d)$. Proposition~\ref{prop:query_lipschitz} gives
\begin{equation}
    W_p\left(\fA(\cdot;\mu)_{\sharp}\mu,\fA(\cdot;\mu)_{\sharp}\nu\right)
    \leq
    R^2\|\Theta\|_2W_p(\mu,\nu).
\end{equation}
For the second term, Proposition~\ref{prop:second_term_noniid} gives a constant $C_1=C_1(\Theta,R)>0$ such that
\begin{equation}
    W_p\left(\fA(\cdot;\mu)_{\sharp}\nu,\fA(\cdot;\nu)_{\sharp}\nu\right)
    \leq
    C_1W_p(\mu,\nu).
\end{equation}
The claim follows with $C=R^2\|\Theta\|_2+C_1$.
\end{proof}

The noncompact case requires the following tail-decay conditions. For self-attention, the same law supplies both the queries and the keys. Thus the noncompact self-attention hypothesis imposes the key strong sub-Gaussian condition and the query-tail conditions on the same probability measure.

\begin{assumption}[Noncompact Self-Attention Input]\label{ass:self_noncompact_input}
Fix $p\geq 1$. For the self-attention $\sA(\mu)$, we assume:
\begin{enumerate}
    \item \textbf{Key/query regularity:} $\mu\in\mathcal{SSG}_{a_{\mu}}(\bbR^d)$ for some $a_{\mu}>0$.
    \item \textbf{Query tail:} At least one of the following query-tail conditions holds:
    \begin{enumerate}
        \item There exist $\delta,\tau>0$ such that $\mu\in\cP_{2+\delta,\tau}(\bbR^d)$.
        \item There exists $\tau>0$ such that $\mu\in\cP_{2,\tau}(\bbR^d)$ and $\Theta$ satisfies both inequalities in \eqref{eq:self_small_theta_sufficient}.
    \end{enumerate}
\end{enumerate}
\end{assumption}
The empirical-key estimate in Lemma~\ref{lem:urbain_lemma_revised} uses
the exponents
\begin{equation}\label{eq:empirical_key_exponent}
    r_p=\max\{p,2\},\qquad
    c(r)=\frac{r}{2(r(r+3)-1)}+\frac{r(r+1)(r+2)}{2(r+3)},\quad r\geq2.
\end{equation}
With the parameter $a_\mu$ from Assumption~\ref{ass:self_noncompact_input},
define the two quantities required by the self-attention proof:
\begin{subequations}\label{eq:self_integrability_quantities}
\begin{align}
    \mathcal I_{p,\Theta}(\mu,\mu)
    &=\int_{\bbR^d}\exp\left(\frac{p}{r_p}c(r_p)a_\mu^2
      \|\Theta^\top x\|_2^2\right)\mu(dx)<\infty,\label{eq:self_leave_one_out_integrability}\\
    \mathcal D_{p,\Theta}(\mu)&=\int_{\bbR^d}\exp\left(2p\langle x,\Theta x\rangle\right)\mathcal{M}_{\mu}\left(-2p\Theta^{\top}x\right)\mu(dx)<\infty.\label{eq:self_diagonal_integrability}
\end{align}
\end{subequations}
Their finiteness follows from the query-tail conditions, as detailed below.
The quantity $\mathcal I_{p,\Theta}(\mu,\mu)$ is used in Proposition~\ref{prop:second_term} to integrate the leave-one-out empirical key estimate obtained from Lemma~\ref{lem:urbain_lemma_revised}. The quantity $\mathcal D_{p,\Theta}(\mu)$ is used in Lemma~\ref{lem:remove_XX_term} to control the diagonal self-interaction. The strong sub-Gaussian condition is used both in Proposition~\ref{prop:query_lipschitz} and in the empirical key estimate.

Under Assumption~\ref{ass:self_noncompact_input} 2(a), both quantities in \eqref{eq:self_integrability_quantities} are finite. In the $2$-exponential case of Assumption~\ref{ass:self_noncompact_input} 2(b), concrete sufficient smallness conditions are

\begin{equation}\label{eq:self_small_theta_sufficient}
    \frac{p}{r_p}c(r_p)a_{\mu}^2\|\Theta\|_2^2<\frac{1}{\tau},\qquad
    2p\|\Theta\|_2+2p^2a_{\mu}^2\|\Theta\|_2^2<\frac{1}{\tau}.
\end{equation}
The strictness of the second inequality also absorbs the linear term arising
from the mean in the strong sub-Gaussian bound for
$\mathcal{M}_{\mu}(-2p\Theta^{\top}x)$.
This inequality ensures finiteness of $\mathcal D_{p,\Theta}(\mu)$, which
controls the diagonal self-interaction and is not needed for cross-attention.

\begin{remark}[Gaussian and faster-tailed examples]\label{rem:self_attention_examples}
A Gaussian measure $\mu=\normal(m,\Sigma)$ is strong sub-Gaussian with any $a_\mu>0$ such that $a_\mu^2\geq\|\Sigma\|_2$, and belongs to $\cP_{2,\tau}$ for a sufficiently large $\tau$. It therefore satisfies Assumption~\ref{ass:self_noncompact_input} when both inequalities in \eqref{eq:self_small_theta_sufficient} hold. These sufficient conditions can be enforced by bounding the spectral norms of $Q$ and $K$. For Gaussian query and key laws in cross-attention, the corresponding sufficient condition is \eqref{eq:cross_small_theta_sufficient}. A law in $\cP_{2+\delta,\tau}$, $\delta>0$, has the required quadratic exponential integrability for every fixed $\Theta$, but faster-than-Gaussian tail decay alone does not imply strong sub-Gaussianity.
\end{remark}

The next theorem separates the empirical-query error from the $N^{-1/2}$ attention error.

\begin{theorem}[Noncompact {i.i.d.} Self-Attention]\label{thm:self_convergence_iid}
Fix $p\geq 1$. Let $\mu$ satisfy Assumption~\ref{ass:self_noncompact_input}. Consider $\mu^N=\frac{1}{N}\sum_{j=1}^{N}\delta_{X_j}$, where $X_1,X_2,\ldots,X_N$ are i.i.d. samples from $\mu$. Then there exists a constant $C=C(p,\Theta,\mu)>0$ such that, for every $N\geq 2$,
\begin{equation}\label{eq:self_iid_main_rate}
    \bbE W_p\left(\sA(\mu),\sA(\mu^N)\right)
    \leq
    C\left(\bbE W_p(\mu,\mu^N)+N^{-1/2}\right).
\end{equation}
\end{theorem}

\begin{proof}
Using the pushforward definition, the triangle inequality gives
\begin{equation}\label{eq:triangle_thm1}
    \bbE W_p\left(\sA(\mu),\sA(\mu^N)\right)
    \leq
    \bbE W_p\left(\fA(\cdot;\mu)_{\sharp}\mu,\fA(\cdot;\mu)_{\sharp}\mu^N\right)
    +
    \bbE W_p\left(\fA(\cdot;\mu)_{\sharp}\mu^N,\fA(\cdot;\mu^N)_{\sharp}\mu^N\right).
\end{equation}
For the first term, Proposition~\ref{prop:query_lipschitz} gives
\begin{equation}
    \bbE W_p\left(\fA(\cdot;\mu)_{\sharp}\mu,\fA(\cdot;\mu)_{\sharp}\mu^N\right)
    \leq
    a_{\mu}^2\|\Theta\|_2\bbE W_p(\mu,\mu^N).
\end{equation}
For the second term, Proposition~\ref{prop:second_term} gives a constant $C_1=C_1(p,\Theta,\mu)>0$ such that
\begin{equation}
    \bbE W_p\left(\fA(\cdot;\mu)_{\sharp}\mu^N,\fA(\cdot;\mu^N)_{\sharp}\mu^N\right)
    \leq
    C_1N^{-1/2}.
\end{equation}
The claim follows with $C=a_{\mu}^2\|\Theta\|_2+C_1$.
\end{proof}

\subsection{Cross-Attention Limits}\label{subsec:cross_attention_limits}

Cross-attention separates the first input, which supplies the queries, from the second input, which supplies the keys and values. This separation lets us state the noncompact assumptions by role. {We assume strong sub-Gaussianity of the key law} to make $s\mapsto\fA(s;\nu)$ globally Lipschitz and to control empirical keys, while the query input only needs the integrability required to average the fixed-query empirical key estimate.

\begin{theorem}[Compact Cross-Attention]\label{thm:cross_compact_stability}
Fix $p\geq 1$. Let $\mu,\tilde{\mu}\in\cP(B_R(0))$ on $\bbR^{d_u}$ and let $\nu,\tilde{\nu}\in\cP(B_R(0))$ on $\bbR^{d_v}$. Then there exists a constant $C=C(\Theta,R)>0$ such that
\begin{equation}\label{eq:cross_compact_stability}
    W_p\left(\sA(\mu,\nu),\sA(\tilde{\mu},\tilde{\nu})\right)
    \leq
    C\left(W_p(\mu,\tilde{\mu})+W_p(\nu,\tilde{\nu})\right).
\end{equation}
\end{theorem}

\begin{proof}
By the triangle inequality,
\begin{equation}
    W_p\left(\sA(\mu,\nu),\sA(\tilde{\mu},\tilde{\nu})\right)
    \leq
    W_p\left(\fA(\cdot;\nu)_{\sharp}\mu,\fA(\cdot;\nu)_{\sharp}\tilde{\mu}\right)
    +
    W_p\left(\fA(\cdot;\nu)_{\sharp}\tilde{\mu},\fA(\cdot;\tilde{\nu})_{\sharp}\tilde{\mu}\right).
\end{equation}
Since $\nu$ is supported in $B_R(0)$, Lemma~\ref{lem:bounded_support_ssg} and Proposition~\ref{prop:query_lipschitz} imply
\begin{equation}
    W_p\left(\fA(\cdot;\nu)_{\sharp}\mu,\fA(\cdot;\nu)_{\sharp}\tilde{\mu}\right)
    \leq
    R^2\|\Theta\|_2W_p(\mu,\tilde{\mu}).
\end{equation}
For the second term, Lemma~\ref{lem:nonIID_W1_bound} gives a constant $C_1=C_1(\Theta,R)>0$ such that, uniformly for $s\in B_R(0)$,
\begin{equation}
    \|\fA(s;\nu)-\fA(s;\tilde{\nu})\|_2
    \leq
    C_1W_1(\nu,\tilde{\nu}).
\end{equation}
Using the coupling induced by the common query $s\sim\tilde{\mu}$ and $W_1(\nu,\tilde{\nu})\leq W_p(\nu,\tilde{\nu})$ gives
\begin{equation}
    W_p\left(\fA(\cdot;\nu)_{\sharp}\tilde{\mu},\fA(\cdot;\tilde{\nu})_{\sharp}\tilde{\mu}\right)
    \leq
    C_1W_p(\nu,\tilde{\nu}).
\end{equation}
The claim follows with $C=R^2\|\Theta\|_2+C_1$.
\end{proof}

\begin{assumption}[Noncompact Cross-Attention Input]\label{ass:cross_noncompact_input}
Fix $p\geq1$. Let $\mu$ be the query law and $\nu$ the key law. Assume:
\begin{enumerate}
    \item \textbf{Key regularity:} $\nu\in\mathcal{SSG}_{a_{\nu}}(\bbR^{d_v})$ for some $a_{\nu}>0$.
    \item \textbf{Query tail:} At least one of the following conditions holds:
    \begin{enumerate}
        \item There exist $\delta,\tau>0$ such that $\mu\in\cP_{2+\delta,\tau}(\bbR^{d_u})$.
        \item There exists $\tau>0$ such that $\mu\in\cP_{2,\tau}(\bbR^{d_u})$ and $\Theta$ satisfies \eqref{eq:cross_small_theta_sufficient}.
    \end{enumerate}
\end{enumerate}
\end{assumption}

Here the empirical-key estimate requires only one integrability quantity.
With $r_p=\max\{p,2\}$, $c(r)$ as in \eqref{eq:empirical_key_exponent},
and the key parameter $a_\nu$ from Assumption~\ref{ass:cross_noncompact_input},
define
\begin{equation}\label{eq:cross_integrability_quantity}
    \mathcal I_{p,\Theta}(\mu,\nu)
    =\int_{\bbR^{d_u}}\exp\left(\frac{p}{r_p}c(r_p)a_\nu^2
    \|\Theta^\top x\|_2^2\right)\mu(dx).
\end{equation}
This agrees with \eqref{eq:self_leave_one_out_integrability} when $\nu=\mu$;
the dependence on the chosen strong sub-Gaussian parameter is suppressed.
Under Assumption~\ref{ass:cross_noncompact_input} 2(a), the quantity in \eqref{eq:cross_integrability_quantity} is finite. In the $2$-exponential case of Assumption~\ref{ass:cross_noncompact_input} 2(b), a concrete sufficient smallness condition is
\begin{equation}\label{eq:cross_small_theta_sufficient}
    \frac{p}{r_p}c(r_p)a_{\nu}^2\|\Theta\|_2^2
    <
    \frac{1}{\tau}.
\end{equation}

The key regularity condition is used in Proposition~\ref{prop:query_lipschitz} for the query-empirical estimate and in Lemma~\ref{lem:urbain_lemma_revised} for the empirical key estimate. The quantity $\mathcal I_{p,\Theta}(\mu,\nu)$ is used in parts 2 and 3 of Theorem~\ref{thm:cross_noncompact}, where the fixed-query empirical key estimate is integrated against the query measure. Unlike the self-attention conditions in \eqref{eq:self_integrability_quantities}, no diagonal integrability condition is needed because the sampled key particles are independent of the query variable.

{The following theorem separates the query and key empirical errors.}

\begin{theorem}[Noncompact Cross-Attention]\label{thm:cross_noncompact}
Fix $p\geq 1$. Let $\mu\in\cP_p(\bbR^{d_u})$ and $\nu\in\cP_p(\bbR^{d_v})$. Let
\begin{equation}
\begin{aligned}
    \mu^N &= \frac{1}{N}\sum_{i=1}^{N}\delta_{X_i}, &
    \nu^M &= \frac{1}{M}\sum_{j=1}^{M}\delta_{Y_j},
\end{aligned}
\end{equation}
where $X_1,\ldots,X_N$ are i.i.d. samples from $\mu$, $Y_1,\ldots,Y_M$ are i.i.d. samples from $\nu$, and the two families are independent. Then:
\begin{enumerate}
    \item If $\nu\in\mathcal{SSG}_{a_{\nu}}(\bbR^{d_v})$, there exists $C_1=C_1(\Theta,\nu,p)>0$ such that
    \begin{equation}\label{eq:cross_query_empirical}
        \bbE W_p\left(\sA(\mu^N,\nu),\sA(\mu,\nu)\right)
        \leq
        C_1\bbE W_p(\mu^N,\mu).
    \end{equation}
    \item Under Assumption~\ref{ass:cross_noncompact_input}, there exists $C_2=C_2(\Theta,\mu,\nu,p)>0$ such that
    \begin{equation}\label{eq:cross_key_empirical}
        \bbE W_p\left(\sA(\mu,\nu^M),\sA(\mu,\nu)\right)
        \leq
        C_2M^{-1/2}.
    \end{equation}
    \item Under Assumption~\ref{ass:cross_noncompact_input}, with the constants $C_1$ and $C_2$ above,
    \begin{equation}\label{eq:cross_both_empirical}
        \bbE W_p\left(\sA(\mu^N,\nu^M),\sA(\mu,\nu)\right)
        \leq
        C_1\bbE W_p(\mu^N,\mu)+C_2M^{-1/2}.
    \end{equation}
\end{enumerate}
\end{theorem}

\begin{proof}
Since $\nu\in\mathcal{SSG}_{a_{\nu}}(\bbR^{d_v})$, Proposition~\ref{prop:query_lipschitz} gives
\begin{equation}
    W_p\left(\sA(\mu^N,\nu),\sA(\mu,\nu)\right)
    =
    W_p\left(\fA(\cdot;\nu)_{\sharp}\mu^N,\fA(\cdot;\nu)_{\sharp}\mu\right)
    \leq
    a_{\nu}^2\|\Theta\|_2W_p(\mu^N,\mu).
\end{equation}
Taking expectations proves \eqref{eq:cross_query_empirical} with $C_1=a_{\nu}^2\|\Theta\|_2$.

For \eqref{eq:cross_key_empirical}, assume $\mathcal I_{p,\Theta}(\mu,\nu)<\infty$ and use the coupling induced by the same query $s\sim\mu$, followed by Jensen's inequality:
\begin{equation}
\begin{split}
    \bbE W_p\left(\sA(\mu,\nu^M),\sA(\mu,\nu)\right)&\leq\left(\int_{\bbR^{d_u}}\bbE\|\fA(s;\nu^M)-\fA(s;\nu)\|_2^p\mu(ds)\right)^{1/p}.
\end{split}
\end{equation}
Set $r_p=\max\{p,2\}$. By Lemma~\ref{lem:urbain_lemma_revised}, applied with exponent $r_p$, there exists $C_3=C_3(r_p,\Theta,\nu)>0$ such that
\begin{equation}
    \bbE\|\fA(s;\nu^M)-\fA(s;\nu)\|_2^{r_p}\leq C_3M^{-r_p/2}\exp\left(c(r_p)a_{\nu}^2\|\Theta^{\top}s\|_2^2\right).
\end{equation}
Since $p\leq r_p$, Lyapunov's inequality gives
\begin{equation}\label{eq:cross_att_thm_lya}
    \bbE\|\fA(s;\nu^M)-\fA(s;\nu)\|_2^p\leq C_3^{p/r_p}M^{-p/2}\exp\left(\frac{p}{r_p}c(r_p)a_{\nu}^2\|\Theta^\top s\|_2^2\right).
\end{equation}
The integral is finite by \eqref{eq:cross_integrability_quantity}. This proves \eqref{eq:cross_key_empirical} with
\begin{equation}
    C_2=C_3^{1/r_p}\mathcal I_{p,\Theta}(\mu,\nu)^{1/p}.
\end{equation}

The same calculation also proves the stronger bound
\begin{equation}\label{eq:cross_key_particle_lp}
\begin{split}
    \left(\bbE W_p^p\left(\sA(\mu,\nu^M),\sA(\mu,\nu)\right)\right)^{1/p}
    &\leq\left(\int\bbE\|\fA(s;\nu^M)-\fA(s;\nu)\|_2^p\mu(ds)\right)^{1/p}\\
    &\leq C_2M^{-1/2},
\end{split}
\end{equation}
whenever $\nu\in\mathcal{SSG}_{a_\nu}$ and $\mathcal I_{p,\Theta}(\mu,\nu)<\infty$. We will use this form for residual maps and finite stacks.

Finally, for \eqref{eq:cross_both_empirical}, the triangle inequality gives
\begin{equation}
    \bbE W_p\left(\sA(\mu^N,\nu^M),\sA(\mu,\nu)\right)\leq\bbE W_p\left(\sA(\mu^N,\nu^M),\sA(\mu^N,\nu)\right)+\bbE W_p\left(\sA(\mu^N,\nu),\sA(\mu,\nu)\right).
\end{equation}
The second term is bounded by \eqref{eq:cross_query_empirical}. For the first term,
\begin{equation}\label{eq:cross_att_3_1}
    W_p^p\left(\sA(\mu^N,\nu^M),\sA(\mu^N,\nu)\right)\leq\frac{1}{N}\sum_{i=1}^N\left\|\fA(X_i;\nu^M)-\fA(X_i;\nu)\right\|_2^p.
\end{equation}
Conditionally on $X_i$, the key samples remain i.i.d. with law $\nu$. Applying \eqref{eq:cross_att_thm_lya} and then averaging over $X_i$ gives
\begin{equation}\label{eq:cross_att_3_2}
    \bbE\|\fA(X_i;\nu^M)-\fA(X_i;\nu)\|_2^p
    \leq C_3^{p/r_p}M^{-p/2}\mathcal I_{p,\Theta}(\mu,\nu).
\end{equation}
Taking expectations in \eqref{eq:cross_att_3_1} and applying Lyapunov's inequality yields
\begin{equation}
    \bbE W_p\left(\sA(\mu^N,\nu^M),\sA(\mu^N,\nu)\right)
    \leq C_3^{1/r_p}M^{-1/2}\mathcal I_{p,\Theta}(\mu,\nu)^{1/p}
    =C_2M^{-1/2}.
\end{equation}
Combining this estimate with \eqref{eq:cross_query_empirical} proves \eqref{eq:cross_both_empirical}.
\end{proof}

\subsection{Transformer Blocks}\label{subsec:transformer_blocks}

We first explain how the compact estimates cover pre-norm transformers,
then define the post-norm blocks used in
Figure~\ref{fig:postnorm_proof_strategy} and prove their stack limit. Both
self-attention and cross-attention are included. The post-norm proof follows
the first-block estimate and compact-stability argument above. For a
left-input cross-attention block, the prescribed key law may remain
unbounded: its strong sub-Gaussian assumption provides the global query
stability needed to propagate the error.

\begin{definition}[Layer Normalization]\label{def:layer_norm}
For $\gamma,\beta\in\bbR^d$ and $\epsilon>0$, define $f_{\gamma,\beta,\epsilon}^{\mathrm{LN}}:\bbR^d\to\bbR^d$ by
\begin{equation}
    \left(f_{\gamma,\beta,\epsilon}^{\mathrm{LN}}(x)\right)_k
    =
    \gamma_k\frac{x_k-\bar{x}}{\sqrt{\sigma_x^2+\epsilon}}+\beta_k
    \text{ for } k\in[d],
\end{equation}
where
\begin{equation}
\begin{aligned}
    \bar{x} &= \frac{1}{d}\sum_{i=1}^{d}x_i, &
    \sigma_x^2 &= \frac{1}{d}\sum_{i=1}^{d}(x_i-\bar{x})^2.
\end{aligned}
\end{equation}
\end{definition}

\begin{remark}[Post-norm and pre-norm attention]\label{rem:post_pre_norm}
The two conventions \citep{xiong2020layer} use the attention estimates
differently. When all attention branches are normalized,
pre-norm attention has bounded query, key, and value inputs, so compact
stability applies from the first layer. In the post-norm blocks below,
attention precedes normalization, and the first layer can therefore
require the noncompact estimates.

For pre-norm attention, suppose that every query, key, and value branch is formed from a layer-normalized token, with $\epsilon>0$. In the self-attention case, write $h=f^{\mathrm{LN}}$ and
\begin{equation}
    \fA_{\mathrm{pre}}(x;\rho)=\fA(h(x);h_\sharp\rho).
\end{equation}
Both normalized attention inputs have deterministic bounded support, by
Lemma~\ref{lem:layer_norm_bounded}. The compact pointwise estimates underlying
Theorems~\ref{thm:self_convergence} and~\ref{thm:cross_compact_stability},
together with Lipschitz continuity of $h$, give, for
$\rho,\tilde\rho\in\cP_p$,
\begin{equation}
\begin{split}
    \|\fA_{\mathrm{pre}}(x;\rho)-\fA_{\mathrm{pre}}(y;\tilde\rho)\|_2
    \leq C\bigl(\|x-y\|_2+W_p(\rho,\tilde\rho)\bigr).
\end{split}
\end{equation}
Coupling $x$ and $y$ and retaining the residual identity therefore proves a global $W_p$ stability bound for the pre-norm residual block. For Lipschitz $f^{\mathrm{NN}}$, the pre-norm feed-forward residual $x\mapsto x+f^{\mathrm{NN}}(h(x))$ is also globally Lipschitz. Thus a fixed finite stack of these pre-norm self-attention blocks satisfies
\begin{equation}
    \bbE W_p\bigl(\sT_{\mathrm{pre}}(\mu^N),\sT_{\mathrm{pre}}(\mu)\bigr)
    \leq C\bbE W_p(\mu^N,\mu),\qquad \mu\in\cP_p,
\end{equation}
without the additional noncompact attention-tail conditions. The same
argument uses compact joint-input stability for cross-attention when all its
query, key, and value branches are normalized and the prescribed auxiliary
laws have finite $p$th moments. Hence finite pre-norm stacks that transport
the input particles are also transformer MNMs. The bounded-support statement
concerns the normalized attention inputs: the unnormalized residual states
themselves need not have compact support.
\end{remark}

\begin{definition}[Post-Norm Transformer Blocks]\label{def:transformer_block}
Let $\sA_{\ell}$ be the attention operator in Definition~\ref{def:attention_on_measures} with matrices $Q_{\ell}$, $K_{\ell}$, and $V_{\ell}$, and let $\fA_{\ell}$ denote its pointwise attention map. Define
\begin{equation}
    \Theta_{\ell}=Q_{\ell}^{\top}K_{\ell}.
\end{equation}
Let $f_{\ell}^{\mathrm{NN}}$ be a feed-forward map, and let $f_{\ell,1}^{\mathrm{LN}}$ and $f_{\ell,2}^{\mathrm{LN}}$ be layer-normalization maps. For a cross-attention block, let $\nu_{\ell}$ denote a prescribed auxiliary probability measure. In the stack notation, $\nu_{\ell}$ is a generic auxiliary law: it is the key law in an $\mathsf{L}$ block and the query law in an $\mathsf{R}$ block. We define the residual attention maps
\begin{subequations}\label{eq:self_cross_residual_blocks}
\begin{align}
    \mathcal R_{\ell}^{\mathsf{S}}(\rho)
    &=
    \left(
    \mathrm{Id}(\cdot)
    +
    \fA_{\ell}(\cdot;\rho)
    \right)_{\sharp}\rho,\\
    \mathcal R_{\ell}^{\mathsf{L}}(\rho;\nu_{\ell})
    &=
    \left(
    \mathrm{Id}(\cdot)
    +
    \fA_{\ell}(\cdot;\nu_{\ell})
    \right)_{\sharp}\rho,\\
    \mathcal R_{\ell}^{\mathsf{R}}(\rho;\nu_{\ell})
    &=
    \left(
    \mathrm{Id}(\cdot)
    +
    \fA_{\ell}(\cdot;\rho)
    \right)_{\sharp}\nu_{\ell}.
\end{align}
\end{subequations}
Here $\mathsf{S}$ denotes self-attention, $\mathsf{L}$ denotes cross-attention in which the current state is the query input, and $\mathsf{R}$ denotes cross-attention in which the current state is the key input. The corresponding transformer blocks are, with $\kappa\in\{\mathsf L,\mathsf R\}$ in the second line,
\begin{subequations}\label{eq:self_cross_transformer_blocks}
\begin{align}
    \sB_{\ell}^{\mathsf{S}}(\rho)
    &=
    \left(f_{\ell,2}^{\mathrm{LN}}\right)_{\sharp}
    \left(\mathrm{Id}+f_{\ell}^{\mathrm{NN}}\right)_{\sharp}
    \left(f_{\ell,1}^{\mathrm{LN}}\right)_{\sharp}
    \mathcal R_{\ell}^{\mathsf{S}}(\rho),\\
    \sB_{\ell}^{\kappa}(\rho;\nu_{\ell})
    &=
    \left(f_{\ell,2}^{\mathrm{LN}}\right)_{\sharp}
    \left(\mathrm{Id}+f_{\ell}^{\mathrm{NN}}\right)_{\sharp}
    \left(f_{\ell,1}^{\mathrm{LN}}\right)_{\sharp}
    \mathcal R_{\ell}^{\kappa}(\rho;\nu_{\ell}).
\end{align}
\end{subequations}
All dimensions are assumed to be compatible with the displayed residual maps. In particular, the output of $\fA_{\ell}$ lies in the same space as the argument to which it is added.
\end{definition}

\begin{remark}[Prescribed auxiliary measures]\label{rem:transformer_architecture_scope}
Each auxiliary law $\nu_\ell$ is fixed independently of $N$ and shared by the empirical and continuum stacks. It therefore introduces no sampling error in Theorem~\ref{thm:transformer_stack_limit}. A fixed collection of auxiliary tokens is included by taking $\nu_\ell$ to be its empirical measure, which has bounded support and is strong sub-Gaussian by Lemma~\ref{lem:bounded_support_ssg}.
\end{remark}

\begin{assumption}[Finite Transformer Stack]\label{ass:finite_transformer_architecture}
Fix $L\geq 1$. For each $\ell\in[L]$, assume:
\begin{enumerate}
    \item The block type is $\kappa_{\ell}\in\{\mathsf{S},\mathsf{L},\mathsf{R}\}$.
    \item The matrices $Q_{\ell}$, $K_{\ell}$, and $V_{\ell}$ are fixed matrices of compatible dimensions.
    \item The map $f_{\ell}^{\mathrm{NN}}$ is Lipschitz.
    \item The maps $f_{\ell,1}^{\mathrm{LN}}$ and $f_{\ell,2}^{\mathrm{LN}}$ are of the form in Definition~\ref{def:layer_norm}, with positive $\epsilon$ parameters.
    \item If $\kappa_{\ell}=\mathsf{L}$, then an auxiliary key measure $\nu_{\ell}$ is specified and
    \begin{equation}
        \nu_{\ell}\in\mathcal{SSG}_{a_{\nu_{\ell}}}
    \end{equation}
    for some $a_{\nu_{\ell}}>0$.
    \item If $\kappa_{\ell}=\mathsf{R}$, then an auxiliary query measure $\nu_{\ell}$ is specified. For every $\ell\geq2$ with $\kappa_{\ell}=\mathsf{R}$, the measure $\nu_{\ell}$ is supported in a deterministic Euclidean ball.
\end{enumerate}
\end{assumption}

For compact notation define the one-step continuum map by
\begin{equation}\label{eq:block_operator_by_type}
    \sB_{\ell}(\rho)
    =
    \begin{cases}
    \sB_{\ell}^{\mathsf{S}}(\rho), & \kappa_{\ell}=\mathsf{S},\\
    \sB_{\ell}^{\mathsf{L}}(\rho;\nu_{\ell}), & \kappa_{\ell}=\mathsf{L},\\
    \sB_{\ell}^{\mathsf{R}}(\rho;\nu_{\ell}), & \kappa_{\ell}=\mathsf{R}.
    \end{cases}
\end{equation}

\begin{theorem}[Stack Limit]\label{thm:transformer_stack_limit}
Fix $p\geq 1$ and assume Assumption~\ref{ass:finite_transformer_architecture}. Let $\mu_0\in\cP_p$. For each $N\geq 2$, let
\begin{equation}
    \mu_0^N=\frac{1}{N}\sum_{j=1}^{N}\delta_{X_j},
\end{equation}
where $X_1,X_2,\ldots,X_N$ are i.i.d. samples from $\mu_0$. Assume that the first block satisfies one of the following conditions:
\begin{enumerate}
    \item If $\kappa_1=\mathsf{S}$, then $\mu_0$ satisfies Assumption~\ref{ass:self_noncompact_input} for $\Theta_1$. Denote its strong sub-Gaussian parameter by $a_0>0$.
    \item If $\kappa_1=\mathsf{L}$, then the fixed key measure $\nu_1$ satisfies $\nu_1\in\mathcal{SSG}_{a_{\nu_1}}$ for some $a_{\nu_1}>0$.
    \item If $\kappa_1=\mathsf{R}$, then the sampled key measure $\mu_0$ satisfies $\mu_0\in\mathcal{SSG}_{a_0}$ for some $a_0>0$, and the fixed query measure $\nu_1$ satisfies $\nu_1\in\cP_p$ and $\mathcal I_{p,\Theta_1}(\nu_1,\mu_0)<\infty$.
\end{enumerate}
Define
\begin{equation}
\begin{aligned}
    \mu_{\ell}&=\sB_{\ell}(\mu_{\ell-1}),&
    \mu_{\ell}^N&=\sB_{\ell}(\mu_{\ell-1}^N)
\end{aligned}
\end{equation}
for $\ell\in[L]$. The superscript $N$ records dependence on the initial $N$-sample empirical measure and does not necessarily indicate the number of atoms after an $\mathsf{R}$ block. Let
\begin{equation}\label{eq:first_block_transformer_rate}
    \Delta_N
    =
    \begin{cases}
    \bbE W_p(\mu_0^N,\mu_0)+N^{-1/2}, & \kappa_1=\mathsf{S},\\
    \bbE W_p(\mu_0^N,\mu_0), & \kappa_1=\mathsf{L},\\
    N^{-1/2}, & \kappa_1=\mathsf{R}.
    \end{cases}
\end{equation}
Since $\mu_0\in\cP_p$, one has $\bbE W_p(\mu_0^N,\mu_0)\to0$ and hence $\Delta_N\to0$. Therefore, for every fixed $\ell\in[L]$,
\begin{equation}
    \bbE W_p(\mu_{\ell}^N,\mu_{\ell})\to 0.
\end{equation}
Moreover, for every fixed $L$ there exists $C>0$, depending only on $p$, $L$, the architecture parameters, $\mu_0$, and the auxiliary probability measures, such that
\begin{equation}
    \bbE W_p(\mu_{L}^N,\mu_{L})\leq C\Delta_N.
\end{equation}

For an architecture-sensitive refinement, define the first right-input cross-attention layer by
\begin{equation}\label{eq:first_right_input_layer}
    \ell_{\mathsf{R}}=\min\left(\{\ell\in[L]:\kappa_\ell=\mathsf R\}\cup\{L+1\}\right).
\end{equation}
Then, possibly after enlarging $C$,

\begin{equation}\label{eq:transformer_stack_architecture_refinement}
    \bbE W_p(\mu_L^N,\mu_L)\leq C\Bigg[\mathbf1_{\{\ell_{\mathsf R}=L+1\}}\bbE W_p(\mu_0^N,\mu_0)
    +\left(\sum_{\ell=1}^{\ell_{\mathsf R}-1}\mathbf1_{\{\kappa_\ell=\mathsf S\}}+\mathbf1_{\{\ell_{\mathsf R}\leq L\}}\right)N^{-1/2}\Bigg].
\end{equation}

Consequently, the smaller of the right-hand sides in the two displayed bounds may be used for any fixed architecture.

\end{theorem}

\begin{proof}
Let $\Phi_{\ell}=f_{\ell,2}^{\mathrm{LN}}\circ\left(\mathrm{Id}+f_{\ell}^{\mathrm{NN}}\right)\circ f_{\ell,1}^{\mathrm{LN}}.$
By Lemma~\ref{lem:layer_norm_bounded} and the Lipschitz assumption on $f_{\ell}^{\mathrm{NN}}$, each $(\Phi_{\ell})_{\sharp}$ is Lipschitz in $W_p$. For the first block, denote the Lipschitz constant of $(\Phi_1)_{\sharp}$ by
\begin{equation}
    C_1=C_1(f_1^{\mathrm{NN}},f_{1,1}^{\mathrm{LN}},f_{1,2}^{\mathrm{LN}})>0.
\end{equation}

We distinguish the three possible first-block types. If $\kappa_1=\mathsf{S}$, 
Proposition~\ref{prop:query_lipschitz} and the identity $\fA_{Q,K,V}=V\fA_{Q,K,I}$ show that $\mathrm{Id}+\fA_1(\cdot;\mu_0)$ is Lipschitz with constant $1+a_0^2\|V_1\|_2\|\Theta_1\|_2.$
For the remaining term, adding the same identity contribution to both images leaves the common-particle coupling cost unchanged, while the value map contributes a factor bounded by $\|V_1\|_2$. Thus the estimate established in the proof of Proposition~\ref{prop:second_term}, specifically \eqref{eq:first_bound_EWp_prop2_revised}, gives a constant
\begin{equation}
    C_2=C_2(p,\Theta_1,V_1,\mu_0)>0
\end{equation}
such that
\begin{equation}
    \bbE W_p\left(\mathcal R_1^{\mathsf{S}}(\mu_0^N),\mathcal R_1^{\mathsf{S}}(\mu_0)\right)
    \leq
    C_2\left(\bbE W_p(\mu_0^N,\mu_0)+N^{-1/2}\right).
\end{equation}

If $\kappa_1=\mathsf{L}$, Proposition~\ref{prop:query_lipschitz} applied to the fixed key probability measure $\nu_1$, together with the identity part of the residual map and the identity $\fA_{Q,K,V}=V\fA_{Q,K,I}$, gives a constant
\begin{equation}
    C_3=C_3(\Theta_1,V_1,\nu_1)>0
\end{equation}
such that
\begin{equation}
    \bbE W_p\left(\mathcal R_1^{\mathsf{L}}(\mu_0^N;\nu_1),\mathcal R_1^{\mathsf{L}}(\mu_0;\nu_1)\right)
    \leq
    C_3\bbE W_p(\mu_0^N,\mu_0).
\end{equation}

If $\kappa_1=\mathsf{R}$, use the common fixed query $s\sim\nu_1$ in the two residual maps. The identity contributions then cancel in the coupling cost. The fixed-query estimate used to prove \eqref{eq:cross_key_empirical}, applied to the pair $(\nu_1,\mu_0)$, together with the identity $\fA_{Q,K,V}=V\fA_{Q,K,I}$, gives a constant
\begin{equation}
    C_4=C_4(p,\Theta_1,V_1,\nu_1,\mu_0)>0
\end{equation}
such that
\begin{equation}
    \bbE W_p\left(\mathcal R_1^{\mathsf{R}}(\mu_0^N;\nu_1),\mathcal R_1^{\mathsf{R}}(\mu_0;\nu_1)\right)
    \leq
    C_4N^{-1/2}.
\end{equation}

Applying $(\Phi_1)_{\sharp}$ and absorbing its Lipschitz constant into the relevant one of $C_2,C_3,C_4$ gives a constant $C_5>0$, depending only on the relevant first-block data, such that
\begin{equation}
    \bbE W_p(\mu_1^N,\mu_1)\leq C_5\Delta_N.
\end{equation}
By Lemma~\ref{lem:layer_norm_bounded}, both $\mu_1^N$ and $\mu_1$ are supported in a deterministic ball depending only on the first-block parameters.

For $\ell\geq 2$, the empirical particles need not be independent. This is harmless because the evolving measures $\mu_{\ell-1}^N$ and $\mu_{\ell-1}$ are compactly supported by the preceding layer normalization.

If $\kappa_{\ell}=\mathsf{S}$, Proposition~\ref{prop:block_compact_lipschitz}, together with the identity $\fA_{Q,K,V}=V\fA_{Q,K,I}$, gives a constant $C_{\ell}>0$ such that
\begin{equation}
    W_p(\mu_{\ell}^N,\mu_{\ell})
    \leq
    C_{\ell}W_p(\mu_{\ell-1}^N,\mu_{\ell-1}).
\end{equation}

If $\kappa_{\ell}=\mathsf{L}$, then $\nu_{\ell}\in\mathcal{SSG}_{a_{\nu_{\ell}}}$ by Assumption~\ref{ass:finite_transformer_architecture}. Proposition~\ref{prop:query_lipschitz} and the identity $\fA_{Q,K,V}=V\fA_{Q,K,I}$ show that the pointwise residual map
\begin{equation}
    x\mapsto x+\fA_{\ell}(x;\nu_{\ell})
\end{equation}
is globally Lipschitz with constant
\begin{equation}
    1+a_{\nu_{\ell}}^2\|V_{\ell}\|_2\|\Theta_{\ell}\|_2.
\end{equation}
Combining this estimate with the Lipschitz continuity of $(\Phi_{\ell})_{\sharp}$ gives a constant $C_{\ell}>0$ such that
\begin{equation}
    W_p(\mu_{\ell}^N,\mu_{\ell})
    \leq
    C_{\ell}W_p(\mu_{\ell-1}^N,\mu_{\ell-1}).
\end{equation}

If $\kappa_{\ell}=\mathsf{R}$, then the active key measures $\mu_{\ell-1}^N$ and $\mu_{\ell-1}$ and the fixed query measure $\nu_{\ell}$ are compactly supported. Proposition~\ref{prop:block_compact_lipschitz}, together with the identity $\fA_{Q,K,V}=V\fA_{Q,K,I}$, gives a constant $C_{\ell}>0$ such that
\begin{equation}
    W_p(\mu_{\ell}^N,\mu_{\ell})
    \leq
    C_{\ell}W_p(\mu_{\ell-1}^N,\mu_{\ell-1}).
\end{equation}

Taking expectations and iterating over finitely many layers proves the result with
\begin{equation}
    C=C_5\prod_{\ell=2}^{L}C_{\ell},
\end{equation}
where the empty product is interpreted as $1$ when $L=1$.

It remains to prove the architecture-sensitive refinement in \eqref{eq:transformer_stack_architecture_refinement}. Construct an oracle stack by setting $\widetilde{\mu}_0^N=\mu_0^N$ and
\begin{equation}\label{eq:ordinary_oracle_recursion}
    \widetilde{\mu}_{\ell}^N=\begin{cases}\left(\Phi_{\ell}\circ\left(\mathrm{Id}+\fA_{\ell}(\cdot;\mu_{\ell-1})\right)\right)_{\sharp}\widetilde{\mu}_{\ell-1}^N,&\ell<\ell_{\mathsf{R}},\ \kappa_{\ell}=\mathsf{S},\\
    \sB_{\ell}^{\mathsf{L}}(\widetilde{\mu}_{\ell-1}^N;\nu_{\ell}),&\ell<\ell_{\mathsf{R}},\ \kappa_{\ell}=\mathsf{L},\\
    \mu_{\ell},&\ell\geq\ell_{\mathsf{R}}.
    \end{cases}
\end{equation}
For every $\ell<\ell_{\mathsf{R}}$, the oracle transition is the pushforward by a deterministic pointwise map. Hence $\widetilde{\mu}_{\ell}^N$ is the empirical measure of $N$ i.i.d. samples from $\mu_{\ell}$. If $\ell_{\mathsf{R}}=L+1$, the oracle stack is the image of $\mu_0^N$ under a composition of population pointwise maps. Proposition~\ref{prop:query_lipschitz}, \eqref{eq:value_map_reduction}, Lemma~\ref{lem:layer_norm_bounded}, and the Lipschitz assumptions on the feed-forward maps imply that this composition is Lipschitz. If $\ell_{\mathsf{R}}\leq L$, the oracle and population stacks agree from layer $\ell_{\mathsf{R}}$ onward. Therefore,
\begin{equation}\label{eq:ordinary_oracle_global_error}
    \bbE W_p(\widetilde{\mu}_L^N,\mu_L)\leq C\mathbf{1}_{\{\ell_{\mathsf{R}}=L+1\}}\bbE W_p(\mu_0^N,\mu_0).
\end{equation}

Define
\begin{equation}
    D_{\ell}=\left(\bbE W_p^p(\mu_{\ell}^N,\widetilde{\mu}_{\ell}^N)\right)^{1/p},\qquad D_0=0,
\end{equation}
and set

\begin{equation}
    \zeta_{\ell,N}=N^{-1/2}
    \left(\mathbf1_{\{\ell<\ell_{\mathsf R},\ \kappa_\ell=\mathsf S\}}
          +\mathbf1_{\{\ell=\ell_{\mathsf R}\leq L\}}\right).
\end{equation}

There exist constants $C_{\ell}>0$, independent of $N$, such that
\begin{equation}\label{eq:ordinary_oracle_error_recursion}
    D_{\ell}\leq C_{\ell}D_{\ell-1}+C_{\ell}\zeta_{\ell,N},\qquad \ell\in[L].
\end{equation}
Indeed, for an $\mathsf{S}$ block before $\ell_{\mathsf{R}}$, add and subtract the output obtained by applying the empirical block to $\widetilde{\mu}_{\ell-1}^N$. Compact stability controls the propagated discrepancy when $\ell\geq2$, while the particle-coupling estimate \eqref{eq:self_key_particle_bound} controls the empirical-key discrepancy. For an $\mathsf{L}$ block, the key law $\nu_{\ell}$ is fixed, so Proposition~\ref{prop:query_lipschitz} gives only the propagated term. At $\ell=\ell_{\mathsf{R}}\leq L$, add and subtract $\sB_{\ell_{\mathsf{R}}}^{\mathsf{R}}(\widetilde{\mu}_{\ell_{\mathsf{R}}-1}^N;\nu_{\ell_{\mathsf{R}}})$. The propagated discrepancy is controlled by compact stability if $\ell_{\mathsf R}\geq2$ and vanishes if $\ell_{\mathsf R}=1$, because the initial empirical and oracle measures coincide. The remaining term is controlled by \eqref{eq:cross_key_particle_lp}, with $N$ i.i.d. oracle keys. After layer $\ell_{\mathsf{R}}$, the oracle input is the population measure, so compact stability propagates the existing discrepancy without introducing a new empirical-key term. These estimates prove \eqref{eq:ordinary_oracle_error_recursion}.

Iterating over the fixed number of layers gives

\begin{equation}\label{eq:ordinary_oracle_error_iteration}
    D_L\leq C\left(\sum_{\ell=1}^{\ell_{\mathsf R}-1}\mathbf1_{\{\kappa_\ell=\mathsf S\}}+\mathbf1_{\{\ell_{\mathsf R}\leq L\}}\right)N^{-1/2}.
\end{equation}

Combining \eqref{eq:ordinary_oracle_global_error} and \eqref{eq:ordinary_oracle_error_iteration} with the triangle and Lyapunov inequalities proves \eqref{eq:transformer_stack_architecture_refinement}.

\end{proof}

For two empirical inputs $\mu_0^N$ and $\mu_0^M$ from the same law, the triangle inequality through $\mu_L$ shows that the expected distance between their stack outputs tends to zero. In particular, the stack is {pairwise ensemble invariant} in probability; independence between these two ensembles is not required for this implication.

Stacks containing only $\mathsf S$ and $\mathsf L$ blocks transport the evolving input law and hence fit the measure neural mapping definition. An $\mathsf R$ block instead transports the auxiliary query law. Its {ensemble consistency} follows from the same estimates, but it need not be representable as a deterministic pushforward of the original input law: a single input atom can produce several distinct auxiliary-query outputs.

\subsection{Grouped Evaluation of Large Particle Sets}
\label{subsec:grouped_particle_evaluation}

Evaluating attention jointly on a large particle ensemble may be prohibitive because the attention matrix has quadratic size. We therefore consider a grouped implementation: partition the $N$ particles into $G$ disjoint groups of equal size $N_G$, apply the same attention map or transformer independently to each group, and concatenate the outputs. At the measure level, concatenation is the equally weighted mixture of the groupwise output measures. We write
\begin{equation}
    N=GN_G.
\end{equation}
All learned parameters and prescribed auxiliary measures are shared across groups.

The next result separates the role of grouping for the three block types. In self-attention, the concatenated output retains all $N$ query particles, while each query is evaluated against only the $N_G=N/G$ keys in its group. For an $\mathsf{L}$ block, the fixed pointwise map commutes exactly with finite mixtures. For an $\mathsf{R}$ block, each group approximates the evolving key/value law with $N_G$ particles while using the same prescribed query law.

\begin{theorem}[Grouped Single-Attention Limits]
\label{thm:grouped_single_attention}
Fix $p\geq1$, let $G,N_G\in\mathbb{N}$ with $N_G\geq2$, and set $N=GN_G$. Let $\{X_{g,i}:g\in[G],\ i\in[N_G]\}$ be i.i.d. samples from $\rho$, and define
\begin{equation}
    \rho_g^{N_G}=\frac{1}{N_G}\sum_{i=1}^{N_G}\delta_{X_{g,i}},\qquad \rho^N=\frac{1}{G}\sum_{g=1}^G\rho_g^{N_G}.
\end{equation}
The constants below are independent of $N$, $G$, and $N_G$; their dependence on the value matrix follows from \eqref{eq:value_map_reduction}.
\begin{enumerate}
    \item Suppose that $\rho$ satisfies Assumption~\ref{ass:self_noncompact_input}. Define
    \begin{equation}
        \widehat{\eta}_{\mathsf{S}}^{N,G}=\frac{1}{G}\sum_{g=1}^G\sA(\rho_g^{N_G}).
    \end{equation}
    Then there exists $C_{\mathsf{S}}=C_{\mathsf{S}}(p,\Theta,V,\rho)>0$ such that
    \begin{equation}\label{eq:grouped_self_attention_bound}
        \bbE W_p\left(\widehat{\eta}_{\mathsf{S}}^{N,G},\sA(\rho)\right)\leq C_{\mathsf{S}}\left(\bbE W_p(\rho^N,\rho)+N_G^{-1/2}\right).
    \end{equation}

    \item Let $\nu\in\mathcal{SSG}_{a_\nu}(\bbR^{d_v})$ be a fixed key measure and let $\rho\in\cP_p(\bbR^{d_u})$. Define
    \begin{equation}
        \widehat{\eta}_{\mathsf{L}}^{N,G}=\frac{1}{G}\sum_{g=1}^G\sA(\rho_g^{N_G},\nu).
    \end{equation}
    Then
    \begin{equation}\label{eq:grouped_left_cross_identity}
        \widehat{\eta}_{\mathsf{L}}^{N,G}=\sA(\rho^N,\nu),
    \end{equation}
    and there exists $C_{\mathsf{L}}=C_{\mathsf{L}}(\Theta,V,\nu)>0$ such that
    \begin{equation}\label{eq:grouped_left_cross_bound}
        \bbE W_p\left(\widehat{\eta}_{\mathsf{L}}^{N,G},\sA(\rho,\nu)\right)\leq C_{\mathsf{L}}\bbE W_p(\rho^N,\rho).
    \end{equation}

    \item Let $\nu\in\cP_p(\bbR^{d_u})$ be a fixed query measure, let $\rho\in\mathcal{SSG}_{a_{\rho}}(\bbR^{d_v})$, and assume $\mathcal I_{p,\Theta}(\nu,\rho)<\infty$. Define
    \begin{equation}
        \widehat{\eta}_{\mathsf{R}}^{N,G}=\frac{1}{G}\sum_{g=1}^G\sA(\nu,\rho_g^{N_G}).
    \end{equation}
    Then there exists $C_{\mathsf{R}}=C_{\mathsf{R}}(p,\Theta,V,\nu,\rho)>0$ such that
    \begin{equation}\label{eq:grouped_right_cross_bound}
        \bbE W_p\left(\widehat{\eta}_{\mathsf{R}}^{N,G},\sA(\nu,\rho)\right)\leq C_{\mathsf{R}}N_G^{-1/2}.
    \end{equation}
\end{enumerate}
\end{theorem}

\begin{proof}
For self-attention, introduce the population-key empirical output
\begin{equation}
    \widetilde{\eta}_{\mathsf{S}}^N=\fA(\cdot;\rho)_{\sharp}\rho^N.
\end{equation}
Pairing the two outputs generated from each common query particle $X_{g,i}$ gives
\begin{equation}\label{eq:grouped_self_common_particle_coupling}
    W_p^p\left(\widehat{\eta}_{\mathsf{S}}^{N,G},\widetilde{\eta}_{\mathsf{S}}^N\right)\leq\frac{1}{N}\sum_{g=1}^G\sum_{i=1}^{N_G}\left\|\fA(X_{g,i};\rho_g^{N_G})-\fA(X_{g,i};\rho)\right\|_2^p.
\end{equation}
The leave-one-out and diagonal estimates in the proof of Proposition~\ref{prop:second_term}, applied separately within each group with sample size $N_G$, imply
\begin{equation}\label{eq:grouped_self_local_key_error}
    \left(\bbE W_p^p\left(\widehat{\eta}_{\mathsf{S}}^{N,G},\widetilde{\eta}_{\mathsf{S}}^N\right)\right)^{1/p}\leq CN_G^{-1/2}.
\end{equation}
By Proposition~\ref{prop:query_lipschitz} and \eqref{eq:value_map_reduction},
\begin{equation}
    W_p\left(\widetilde{\eta}_{\mathsf{S}}^N,\sA(\rho)\right)\leq a_{\rho}^2\|V\|_2\|\Theta\|_2W_p(\rho^N,\rho),
\end{equation}
where $a_{\rho}$ is the strong sub-Gaussian parameter of $\rho$. The triangle and Lyapunov inequalities prove \eqref{eq:grouped_self_attention_bound}.

For left-input cross-attention, the pointwise map $s\mapsto\fA(s;\nu)$ is fixed, and pushforward by a fixed map commutes with finite mixtures. Hence
\begin{equation}
    \widehat{\eta}_{\mathsf{L}}^{N,G}=\frac{1}{G}\sum_{g=1}^G\fA(\cdot;\nu)_{\sharp}\rho_g^{N_G}=\fA(\cdot;\nu)_{\sharp}\rho^N=\sA(\rho^N,\nu).
\end{equation}
Equation~\eqref{eq:grouped_left_cross_bound} follows from Proposition~\ref{prop:query_lipschitz} and \eqref{eq:value_map_reduction}.

For right-input cross-attention, couple each $\sA(\nu,\rho_g^{N_G})$ with $\sA(\nu,\rho)$ using the same query $s\sim\nu$, and average the resulting couplings over $g$. This gives
\begin{equation}
    W_p^p\left(\widehat{\eta}_{\mathsf{R}}^{N,G},\sA(\nu,\rho)\right)\leq\frac{1}{G}\sum_{g=1}^G\int\left\|\fA(s;\rho_g^{N_G})-\fA(s;\rho)\right\|_2^p\nu(ds).
\end{equation}
The $L^p$ empirical-key estimate used in the proof of \eqref{eq:cross_key_empirical}, together with \eqref{eq:value_map_reduction}, yields
\begin{equation}\label{eq:grouped_right_local_key_error}
    \left(\bbE W_p^p\left(\widehat{\eta}_{\mathsf{R}}^{N,G},\sA(\nu,\rho)\right)\right)^{1/p}\leq C_{\mathsf{R}}N_G^{-1/2}.
\end{equation}
Lyapunov's inequality proves \eqref{eq:grouped_right_cross_bound}.
\end{proof}

We now propagate these groupwise errors through a finite transformer stack, using the block types $\kappa_\ell$ and the first right-input layer $\ell_{\mathsf R}$ in \eqref{eq:first_right_input_layer}.

\begin{theorem}[Grouped Transformer Limit]
\label{thm:grouped_transformer_limit}
Fix $p\geq1$, let $\mu_0\in\cP_p$, and assume Assumption~\ref{ass:finite_transformer_architecture} together with the first-block condition in Theorem~\ref{thm:transformer_stack_limit}. Let $G,N_G\in\mathbb{N}$ with $N_G\geq2$, set $N=GN_G$, and let $\{X_{g,i}:g\in[G],\ i\in[N_G]\}$ be i.i.d. samples from $\mu_0$. Define
\begin{equation}
    \widehat{\mu}_{0,g}=\frac{1}{N_G}\sum_{i=1}^{N_G}\delta_{X_{g,i}},\qquad \mu_0^N=\frac{1}{G}\sum_{g=1}^G\widehat{\mu}_{0,g}.
\end{equation}
Evaluate the stack independently within each group by
\begin{equation}
    \widehat{\mu}_{\ell,g}=\sB_{\ell}(\widehat{\mu}_{\ell-1,g}),\qquad \widehat{\mu}_{\ell}^{N,G}=\frac{1}{G}\sum_{g=1}^G\widehat{\mu}_{\ell,g},\qquad \ell\in[L].
\end{equation}
The superscripts $N$ and $(N,G)$ record dependence on the initial sample size and grouping parameters and do not necessarily indicate the number of atoms after an $\mathsf{R}$ block. Let the population stack be
\begin{equation}
    \mu_{\ell}=\sB_{\ell}(\mu_{\ell-1}),\qquad \ell\in[L].
\end{equation}
Then there exists $C>0$, depending only on $p$, $L$, the architecture parameters, $\mu_0$, and the prescribed auxiliary measures, such that

\begin{equation}\label{eq:grouped_transformer_bound}
    \bbE W_p\left(\widehat\mu_L^{N,G},\mu_L\right)\leq C\Bigg[\mathbf1_{\{\ell_{\mathsf R}=L+1\}}\bbE W_p(\mu_0^N,\mu_0)
    +\left(\sum_{\ell=1}^{\ell_{\mathsf R}-1}\mathbf1_{\{\kappa_\ell=\mathsf S\}}+\mathbf1_{\{\ell_{\mathsf R}\leq L\}}\right)N_G^{-1/2}\Bigg].
\end{equation}

The attention-error factor can also be written in terms of $N$ and $G$:
\begin{equation}\label{eq:grouped_transformer_bound_NG}
    N_G^{-1/2}=\sqrt{\frac{G}{N}}.
\end{equation}

Since $\mu_0\in\cP_p$, one has $\bbE W_p(\mu_0^N,\mu_0)\to0$. If at least one block has type $\mathsf S$ or $\mathsf R$, the additional groupwise terms vanish whenever $N_G\to\infty$, equivalently $G=o(N)$. For fixed $G$, the attention error is $O(N^{-1/2})$. If $\kappa_\ell=\mathsf L$ for every $\ell$, grouping is exact:
\begin{equation}\label{eq:grouped_all_left_exact}
    \widehat{\mu}_L^{N,G}=\mu_L^N.
\end{equation}
For architectures with self-attention or right-input cross-attention, keeping $N_G$ fixed while increasing $N$ does not make the attention error terms in this bound vanish.
\end{theorem}

\begin{proof}
Define the pointwise post-attention map
\begin{equation}
    \Phi_{\ell}=f_{\ell,2}^{\mathrm{LN}}\circ(\mathrm{Id}+f_{\ell}^{\mathrm{NN}})\circ f_{\ell,1}^{\mathrm{LN}}.
\end{equation}
By Lemma~\ref{lem:layer_norm_bounded} and the Lipschitz assumption on $f_{\ell}^{\mathrm{NN}}$, each $\Phi_{\ell}$ is globally Lipschitz and has bounded range. Construct oracle group measures by setting $\widetilde{\mu}_{0,g}=\widehat{\mu}_{0,g}$ and
\begin{equation}\label{eq:grouped_oracle_recursion}
    \widetilde{\mu}_{\ell,g}=\begin{cases}\left(\Phi_{\ell}\circ\left(\mathrm{Id}+\fA_{\ell}(\cdot;\mu_{\ell-1})\right)\right)_{\sharp}\widetilde{\mu}_{\ell-1,g},&\ell<\ell_{\mathsf{R}},\ \kappa_{\ell}=\mathsf{S},\\
    \sB_{\ell}^{\mathsf{L}}(\widetilde{\mu}_{\ell-1,g};\nu_{\ell}),&\ell<\ell_{\mathsf{R}},\ \kappa_{\ell}=\mathsf{L},\\
    \mu_{\ell},&\ell\geq\ell_{\mathsf{R}}.
    \end{cases}
\end{equation}
For every $\ell<\ell_{\mathsf{R}}$, the oracle transition is the pushforward by a deterministic pointwise map. Hence $\widetilde{\mu}_{\ell,g}$ is the empirical measure of $N_G$ i.i.d. samples from $\mu_{\ell}$, and the $GN_G=N$ oracle particles are i.i.d. with law $\mu_{\ell}$. Set
\begin{equation}
    \widetilde{\mu}_{\ell}^{N,G}=\frac{1}{G}\sum_{g=1}^G\widetilde{\mu}_{\ell,g}.
\end{equation}
For $\ell<\ell_{\mathsf{R}}$, the concatenated oracle measure is the image of $\mu_0^N$ under the corresponding composition of population pointwise maps. Proposition~\ref{prop:query_lipschitz}, \eqref{eq:value_map_reduction}, Lemma~\ref{lem:layer_norm_bounded}, and the Lipschitz assumption on each $f_{\ell}^{\mathrm{NN}}$ imply that this composition is Lipschitz. If $\ell_{\mathsf{R}}\leq L$, the oracle measures equal the population measures from layer $\ell_{\mathsf{R}}$ onward. Therefore,
\begin{equation}\label{eq:grouped_oracle_global_error}
    \bbE W_p\left(\widetilde{\mu}_L^{N,G},\mu_L\right)\leq C\mathbf{1}_{\{\ell_{\mathsf{R}}=L+1\}}\bbE W_p(\mu_0^N,\mu_0).
\end{equation}

It remains to compare the independently evaluated groups with their oracle counterparts. Define
\begin{equation}
    D_{\ell}=\left(\frac{1}{G}\sum_{g=1}^G\bbE W_p^p\left(\widehat{\mu}_{\ell,g},\widetilde{\mu}_{\ell,g}\right)\right)^{1/p},\qquad D_0=0,
\end{equation}
and set

\begin{equation}
    \zeta_{\ell,N_G}=N_G^{-1/2}
    \left(\mathbf1_{\{\ell<\ell_{\mathsf R},\ \kappa_\ell=\mathsf S\}}
          +\mathbf1_{\{\ell=\ell_{\mathsf R}\leq L\}}\right).
\end{equation}

We claim that there exist constants $C_{\ell}>0$, independent of $N$, $G$, and $N_G$, such that
\begin{equation}\label{eq:grouped_error_recursion}
    D_{\ell}\leq C_{\ell}D_{\ell-1}+C_{\ell}\zeta_{\ell,N_G},\qquad \ell\in[L].
\end{equation}

Consider first an $\mathsf{S}$ block with $\ell<\ell_{\mathsf{R}}$. Add and subtract the output obtained by applying the empirical block to $\widetilde{\mu}_{\ell-1,g}$. If $\ell\geq2$, Proposition~\ref{prop:block_compact_lipschitz}, together with \eqref{eq:value_map_reduction}, controls the propagated discrepancy; for $\ell=1$, that discrepancy vanishes because $D_0=0$. The remaining discrepancy keeps the same oracle query particles and replaces their empirical key measure by $\mu_{\ell-1}$. The $L^p$ estimates in the proof of Proposition~\ref{prop:second_term}, followed by the residual and post-attention Lipschitz maps, give the self-attention contribution to $\zeta_{\ell,N_G}$. At $\ell=1$, the required assumptions are precisely the first-block assumptions. For $\ell\geq2$, the population measure $\mu_{\ell-1}$ has bounded support by Lemma~\ref{lem:layer_norm_bounded}; it is therefore strong sub-Gaussian by Lemma~\ref{lem:bounded_support_ssg} and belongs to $\cP_{2+\delta,\tau}$ for suitable $\delta,\tau>0$, so Assumption~\ref{ass:self_noncompact_input} holds.

For an $\mathsf{L}$ block, the key measure $\nu_{\ell}$ is fixed. Proposition~\ref{prop:query_lipschitz}, \eqref{eq:value_map_reduction}, the residual identity, and the Lipschitz post-attention maps give $D_{\ell}\leq C_{\ell}D_{\ell-1}$; no compact-support assumption on $\nu_{\ell}$ is needed.

At $\ell=\ell_{\mathsf{R}}\leq L$, add and subtract $\sB_{\ell_{\mathsf{R}}}^{\mathsf{R}}(\widetilde{\mu}_{\ell_{\mathsf{R}}-1,g};\nu_{\ell_{\mathsf{R}}})$. If $\ell_{\mathsf{R}}\geq2$, Proposition~\ref{prop:block_compact_lipschitz}, together with \eqref{eq:value_map_reduction}, controls the propagated discrepancy; if $\ell_{\mathsf{R}}=1$, that discrepancy vanishes. The remaining discrepancy is an empirical-key error with $N_G$ i.i.d. keys and fixed query measure $\nu_{\ell_{\mathsf{R}}}$. The $L^p$ empirical-key estimate used in the proof of \eqref{eq:cross_key_empirical}, followed by the post-attention Lipschitz maps, gives the right-input contribution to $\zeta_{\ell,N_G}$. The required integrability is part of the first-block assumptions when $\ell_{\mathsf{R}}=1$ and follows from compact support when $\ell_{\mathsf{R}}\geq2$.

For $\ell>\ell_{\mathsf{R}}$, the oracle input is already the population measure $\mu_{\ell-1}$, so no fresh empirical-key error occurs. Self-attention and $\mathsf{R}$ blocks are controlled by Proposition~\ref{prop:block_compact_lipschitz} and \eqref{eq:value_map_reduction}; an $\mathsf{L}$ block is controlled as above by Proposition~\ref{prop:query_lipschitz}. Taking the $L^p$ norm over the product of the group index and the underlying probability space and applying Minkowski's inequality proves \eqref{eq:grouped_error_recursion}. Iterating over the fixed number of layers gives

\begin{equation}\label{eq:grouped_error_iteration}
    D_L\leq C\left(\sum_{\ell=1}^{\ell_{\mathsf R}-1}\mathbf1_{\{\kappa_\ell=\mathsf S\}}+\mathbf1_{\{\ell_{\mathsf R}\leq L\}}\right)N_G^{-1/2}.
\end{equation}

Finally, mixing couplings for the $G$ pairs of group measures gives
\begin{equation}
    W_p^p\left(\widehat{\mu}_L^{N,G},\widetilde{\mu}_L^{N,G}\right)\leq\frac{1}{G}\sum_{g=1}^GW_p^p\left(\widehat{\mu}_{L,g},\widetilde{\mu}_{L,g}\right).
\end{equation}
Taking expectations and applying Lyapunov's inequality gives
\begin{equation}
    \bbE W_p\left(\widehat{\mu}_L^{N,G},\widetilde{\mu}_L^{N,G}\right)\leq D_L.
\end{equation}
Combining this estimate with \eqref{eq:grouped_oracle_global_error},
\eqref{eq:grouped_error_iteration}, and the triangle inequality proves
\eqref{eq:grouped_transformer_bound}. Identity~\eqref{eq:grouped_transformer_bound_NG}
follows from $N=GN_G$.

If $\kappa_\ell=\mathsf L$ for every $\ell$, then each block is the pushforward by a fixed pointwise map and therefore commutes with finite mixtures. Iterating this identity proves \eqref{eq:grouped_all_left_exact}.
\end{proof}

\subsection{Extension to Multihead Attention}
\label{subsec:multihead_extension}

The preceding results use one attention head. We now extend them to a
fixed finite number of heads, all evaluated on the same query and key
inputs. The extension uses the pointwise estimates underlying the
single-head proofs, so it preserves the dependence among head outputs.

Let $H\geq1$ be fixed. For head $h\in[H]$, let
$\fA_h(s;\nu)$ be the attention map in
Definition~\ref{def:attention_on_measures}, with matrices $Q_h,K_h,V_h$
and $\Theta_h=Q_h^\top K_h$. Given a fixed output matrix $O$, define
\begin{equation}\label{eq:multihead_pointwise_map}
\begin{aligned}
    \fA^{\mathrm{MH}}(s;\nu)
    &=O\operatorname{concat}\bigl(\fA_1(s;\nu),\ldots,
                                  \fA_H(s;\nu)\bigr),\\
    \sA^{\mathrm{MH}}(\mu,\nu)
    &=\fA^{\mathrm{MH}}(\cdot;\nu)_\sharp\mu,
\end{aligned}
\end{equation}
where $\operatorname{concat}$ stacks the head outputs in one vector.
The dimensions of $O$ are chosen to make the residual connections
compatible. A block is defined as in
Definition~\ref{def:transformer_block}, replacing its attention map by
\eqref{eq:multihead_pointwise_map}. Different layers may have different
fixed numbers of heads and output matrices.

\begin{corollary}[Multihead Extension]\label{cor:multihead_extension}
Fix $p\geq1$. The compact and noncompact self- and cross-attention
estimates remain valid for $\sA^{\mathrm{MH}}$, provided that the
corresponding hypotheses hold for every head. The constants may depend
on $H$, $O$, and all head parameters, but not on the sample sizes.
In particular, under Assumption~\ref{ass:self_noncompact_input} for
each $\Theta_h$,
\begin{equation}\label{eq:multihead_self_rate}
    \bbE W_p\bigl(\sA^{\mathrm{MH}}(\mu^N,\mu^N),
                   \sA^{\mathrm{MH}}(\mu,\mu)\bigr)
    \leq C\bigl(\bbE W_p(\mu^N,\mu)+N^{-1/2}\bigr).
\end{equation}
The conclusions of Theorems~\ref{thm:transformer_stack_limit}
and~\ref{thm:grouped_transformer_limit}, including their
architecture-sensitive bounds, also hold for multihead blocks.
The residual, feed-forward, normalization, and auxiliary-law assumptions
are unchanged; the attention and first-block hypotheses hold for every
head. All heads in a block share its type and auxiliary law, and
residual dimensions are matched after the output projection.
All head counts and output matrices are fixed independently of
$N$ and $G$. The pre-norm result in
Remark~\ref{rem:post_pre_norm} extends in the same way.
\end{corollary}

\begin{proof}
For any queries $s,t$ and key laws $\nu,\widetilde\nu$ for which the
maps are defined, concatenation and the output projection give
\begin{equation}\label{eq:multihead_coupling_bound}
    \bigl\|\fA^{\mathrm{MH}}(s;\nu)
            -\fA^{\mathrm{MH}}(t;\widetilde\nu)\bigr\|_2
    \leq \|O\|_2\sum_{h=1}^H
       \bigl\|\fA_h(s;\nu)-\fA_h(t;\widetilde\nu)\bigr\|_2.
\end{equation}
On compact sets, apply the pointwise query and key stability estimates
to each summand and integrate over one coupling of the query laws.
Minkowski's inequality gives the compact multihead estimate. For
noncompact keys, Proposition~\ref{prop:query_lipschitz}, with each
value matrix retained, bounds the Lipschitz constant of the population
query map by
\[
    \|O\|_2\sum_{h=1}^H
       a_\nu^2\|V_h\|_2\|\Theta_h\|_2.
\]
This controls transport of the query measure. For the empirical-key
error, use \eqref{eq:multihead_coupling_bound} in $L^p$ under the
common-particle coupling of \eqref{eq:self_key_particle_bound} for
self-attention, or the common-query coupling of
\eqref{eq:cross_key_particle_lp} for cross-attention. Each head
contributes an $O(N^{-1/2})$ or $O(M^{-1/2})$ term. The finite sum
preserves the order, proving the single-layer claims, including
\eqref{eq:multihead_self_rate}. Independence between heads is not
required: they use the same particles and the same coupling.

For stacks, replace the pointwise attention map in the oracle
recursion~\eqref{eq:ordinary_oracle_recursion} by
$\fA^{\mathrm{MH}}$. Before the first prescribed-query block, the
oracle particles are still i.i.d. images of the initial samples under
a deterministic population map. The headwise estimates just proved
give the same local error orders in
\eqref{eq:ordinary_oracle_error_recursion}. Residual connections,
feed-forward maps, and layer normalization propagate these errors
as in the single-head proof. Thus the stack bounds retain their rates.
Applying the same argument to each group gives the
$N_G^{-1/2}$ attention term in
\eqref{eq:grouped_transformer_bound}; the oracle input contribution
still uses all $N$ particles. Finally, normalized pre-norm branches
have bounded range for every head, so their deterministic compact
estimate follows from the first part of the argument.
\end{proof}

\subsection{Example: Set Transformers as {Measure-to-Vector Maps}}
\label{subsec:set_transformer_limit}

The Set Transformer of \citet{lee2019set} combines an equivariant encoder
with attention pooling at learned seed vectors. We consider a self-attention
encoder with parameters, depth, head counts, and seeds fixed independently
of the input sample size. Following the operator formulation of
\citet[Sections~3.1 and~4.1]{bach2026learning}, its {measure-to-vector} map is
\begin{equation}\label{eq:set_transformer_feature_map}
    \mu
    \xrightarrow[\text{encoder}]{\sB_L^{\mathsf S}\circ\cdots\circ\sB_1^{\mathsf S}}
    \mu_L
    \xrightarrow{g_\sharp}
    g_\sharp\mu_L
    \xrightarrow[\eqref{eq:set_transformer_pooling}]{\text{fixed-seed pooling}}
    (z_j(\mu))_{j=1}^k
    \xrightarrow{f_{\mathrm{out}}}
    \sFST(\mu).
\end{equation}
The blocks are those of Definition~\ref{def:transformer_block}, with the
multihead attention map in \eqref{eq:multihead_pointwise_map}.
For an input probability measure
$\mu$ on $\bbR^{d_0}$, set $\mu_0=\mu$ and
\begin{equation}
    \mu_{\ell}=\sB_{\ell}^{\mathsf{S}}(\mu_{\ell-1}),
    \qquad \ell\in[L],
\end{equation}
where $L\geq1$. Let
$g:\bbR^{d_L}\to\bbR^{d}$ be a fixed Lipschitz map applied separately
to each encoded token. This is the row-wise feed-forward map used
before pooling. Fix an ordered tuple of seed vectors
$(s_1,\ldots,s_k)\in(\bbR^{d})^k$. Pooling is an
$\mathsf R$ block with these seeds as queries and $g_\sharp\mu_L$
as the key law. Define its ordered outputs by
\begin{equation}\label{eq:set_transformer_pooling}
    \delta_{z_j(\mu)}
    =\sB_{\mathrm{pool}}^{\mathsf R}(g_\sharp\mu_L;\delta_{s_j}),
    \qquad j\in[k].
\end{equation}
The block sends each Dirac query law to a Dirac output law, so this
defines $z_j(\mu)$ uniquely. Its attention output has dimension $d$,
its feed-forward map is Lipschitz, and both layer normalizations have
positive $\epsilon$. For a fixed map
$f_{\mathrm{out}}:(\bbR^{d})^k\to\bbR^q$, the {vector output} is
$\sFST(\mu)=f_{\mathrm{out}}(z_1(\mu),\ldots,z_k(\mu))$ on the
input laws for which the encoder attention integrals exist.
Assume that $f_{\mathrm{out}}$ is Lipschitz on the bounded set reached
by the pooled tokens. This includes a finite decoder of self-attention
blocks acting on the $k$ seed outputs, followed by a Lipschitz
feed-forward readout: on this bounded set, the compact pointwise
estimates control token differences with the seed labels preserved.

The map is invariant to input permutations because it depends on the
input ensemble only through its empirical measure. Its output retains
the seed labels, so the convergence argument below compares corresponding
pooled tokens.

\begin{theorem}[Fixed-Parameter Set Transformer Limit]
\label{thm:set_transformer_limit}
Fix $p\geq1$ and the architecture just described. Assume that the
encoder satisfies Assumption~\ref{ass:finite_transformer_architecture}
head by head, as in Corollary~\ref{cor:multihead_extension}, with
$\kappa_{\ell}=\mathsf{S}$ for every $\ell$. Assume that $\mu$
satisfies Assumption~\ref{ass:self_noncompact_input} for every first-layer
encoder attention matrix $\Theta_{1,h}$ and this exponent $p$. Let
$\mu^N=N^{-1}\sum_{i=1}^N\delta_{X_i}$, where the $X_i$ are i.i.d.
with law $\mu$. Then the {measure-to-vector} map in
\eqref{eq:set_transformer_feature_map} is well defined at $\mu$ and
$\mu^N$, and there is a constant $C$ independent of $N$ such that
\begin{equation}\label{eq:set_transformer_feature_rate}
    \bbE\left\|\sFST(\mu^N)-\sFST(\mu)\right\|_2
    \leq C N^{-1/2},
    \qquad N\geq2.
\end{equation}
The constant may depend on the input law, the dimensions, and all
fixed architecture parameters. Consequently, the output converges in
$L^1$ and in probability. Outputs from any two i.i.d. ensembles with
the same input law converge to one another in these senses as both
sample sizes tend to infinity.
\end{theorem}

\begin{proof}
Let $\mu_L^N$ be the empirical encoder output and let
$\widetilde\mu_L^N$ be the oracle output defined in
\eqref{eq:ordinary_oracle_recursion}. Since the encoder contains only
self-attention blocks, $\widetilde\mu_L^N$ is an empirical measure of
$N$ i.i.d. samples with law $\mu_L$. Moreover, the encoder estimate
\eqref{eq:ordinary_oracle_error_iteration}, extended by
Corollary~\ref{cor:multihead_extension}, gives
\begin{equation}\label{eq:set_transformer_encoder_oracle}
    \left(\bbE W_p^p(\mu_L^N,\widetilde\mu_L^N)\right)^{1/p}
    \leq C N^{-1/2}.
\end{equation}
All three measures $\mu_L^N$, $\widetilde\mu_L^N$, and $\mu_L$ have
support in one deterministic ball, by the final encoder layer
normalization. Their pushforwards by $g$ also have support in a common
deterministic ball.

For each seed, insert the intermediate attention output
$\fA_{\mathrm{pool}}(s_j;g_{\sharp}\widetilde\mu_L^N)$.
The compact key estimate in Lemma~\ref{lem:nonIID_W1_bound}, applied
head by head and combined using \eqref{eq:multihead_coupling_bound}, gives
\begin{equation}
    \bigl\|\fA_{\mathrm{pool}}(s_j;g_{\sharp}\mu_L^N)
        -\fA_{\mathrm{pool}}(s_j;g_{\sharp}\widetilde\mu_L^N)
      \bigr\|_2\leq C W_p(\mu_L^N,\widetilde\mu_L^N).
\end{equation}
Choose the ball large enough to contain every fixed seed as well.
The keys of $g_{\sharp}\widetilde\mu_L^N$ are i.i.d. with bounded
support. The fixed-query empirical-key estimate in
Lemma~\ref{lem:urbain_lemma_revised}, used with exponent $2$ for each
head and combined using \eqref{eq:multihead_coupling_bound}, therefore
implies
\begin{equation}
    \bbE\left\|
       \fA_{\mathrm{pool}}(s_j;g_{\sharp}\widetilde\mu_L^N)
       -\fA_{\mathrm{pool}}(s_j;g_{\sharp}\mu_L)
    \right\|_2
    \leq C N^{-1/2}.
\end{equation}
The residual seed cancels when comparing outputs with the same label
$j$. Apply the Lipschitz post-attention operations in the pooling
block, sum over the fixed number $k$ of labels, and apply the
Lipschitz readout. Equation~\eqref{eq:set_transformer_encoder_oracle}
and Lyapunov's inequality then prove
\eqref{eq:set_transformer_feature_rate}.

Finally, Markov's inequality gives convergence in probability, and
the triangle inequality through $\sFST(\mu)$ proves the assertion
for two ensembles, whether or not those ensembles are independent
of each other.
\end{proof}

\begin{remark}[Induced Attention and Compact Inputs]
The inducing vectors in an induced self-attention block are fixed
parameters, but their updated states depend on the input law.
Writing $\eta_I$ for the empirical measure of the fixed inducing
vectors, the measure form of this block is
\begin{equation}
    \rho\longmapsto
    \sB_2^{\mathsf L}\bigl(\rho;\sB_1^{\mathsf R}(\rho;\eta_I)\bigr).
\end{equation}
The first block updates the inducing queries using the input keys;
the second updates the input queries using these induced keys.
The intermediate key law
$\sB_1^{\mathsf R}(\rho;\eta_I)$ depends on $\rho$ and is not a
prescribed auxiliary law in
Assumption~\ref{ass:finite_transformer_architecture}. Thus the
noncompact theorem above is stated for a self-attention-block
encoder. For inputs supported in a fixed ball, however, the two-input
compact block estimate in Proposition~\ref{prop:block_compact_lipschitz}
applies successively to the two displayed blocks, while retaining
$\rho$ as the outer query law. Finite encoders built from either
self-attention or induced self-attention blocks, followed by the
ordered pooling and readout above, consequently satisfy
\begin{equation}
    \bbE\left\|\sFST(\mu^N)-\sFST(\mu)\right\|_2
    \leq C\bbE W_p(\mu^N,\mu)\longrightarrow0
\end{equation}
for compactly supported $\mu$. Here all inducing vectors, seeds, and
network parameters are again fixed independently of $N$.
\end{remark}

\section{Numerical Experiments}
\label{sec:experiments}

We examine how the output of a {fixed trained self-attention transformer} converges
as the input ensemble grows, using compactly supported, Gaussian, and
faster-than-Gaussian input laws. The experiments compare the observed decay
with the rates suggested by Theorem~\ref{thm:transformer_stack_limit}.
We also measure approximation to the true operator to distinguish
particle-count convergence from the accuracy of the learned map.

\subsection{Experimental setting}
\label{subsec:experimental_protocol}

For each problem, we seek to approximate an operator
$\sG^\dagger:\mathcal X_d\to\mathcal Y_d$, where $\mathcal X_d$ and
$\mathcal Y_d$ are classes of probability measures. Given independent
samples $X_1,\ldots,X_N$ from $\mu\in\mathcal X_d$, write
\[
    \mu^N=\frac1N\sum_{i=1}^N\delta_{X_i},\qquad
    \widehat\pi_\theta^N=\sG_\theta(\mu^N),\qquad
    \pi_\theta=\sG_\theta(\mu).
\]
The network maps the input ensemble to an output ensemble of the same
size. It consists of a pointwise input map, a stack of multihead
self-attention and feed-forward blocks with post-norm layer normalization,
and a pointwise output map. There are no positional encodings.
At each optimization step, each item in the batch receives a freshly
drawn random input measure $\mu$. We sample input particles from $\mu$
and determine its true output law $\sG^\dagger(\mu)$ analytically from
the selected measure parameters. The input ensemble passes through the
transformer, and we compare the predicted particles with an independent
ensemble sampled from this true output law. Training minimizes their
empirical energy distance \citep{szekely2013energy}, averaged over the
batch; on the torus, this distance uses sine--cosine embeddings.
The input measures vary during training, while the target operator
remains fixed. At evaluation, the trained parameters are frozen and
attention acts on the complete ensemble, including the reference.
Appendices~\ref{app:experimental_generators} and
\ref{app:experimental_training} specify the random measure families used
for training and testing, target sampling, and detailed training settings.

We use $p=2$ throughout. Under the hypotheses of
Theorem~\ref{thm:transformer_stack_limit} and the multihead extension in
Corollary~\ref{cor:multihead_extension}, a fixed self-attention stack
satisfies
\begin{equation}\label{eq:experimental_transformer_rate}
    \bbE W_2(\widehat\pi_\theta^N,\pi_\theta)
    \leq C\left(\bbE W_2(\mu^N,\mu)+N^{-1/2}\right).
\end{equation}
The same estimate holds with Lipschitz input and output maps when the
lifted input law satisfies the first-layer assumptions. In particular, for
regular compact input laws, the empirical Wasserstein estimates of
\citet{borda2023empirical,divol2021short} give the comparison rates
\begin{equation}\label{eq:experimental_dimension_rates}
    r_d(N)=
    \begin{cases}
        N^{-1/2}, & d=1,\\[2pt]
        (\log N/N)^{1/2}, & d=2,\\[2pt]
        N^{-1/d}, & d\geq3.
    \end{cases}
\end{equation}
Here $d$ is the dimension of the input distribution. These rates apply,
for example, to densities bounded above and away from zero on the flat
torus. The attention term
in \eqref{eq:experimental_transformer_rate} does not worsen their order.
We use the same curves as benchmarks for the noncompact experiments;
their application there requires additional distributional assumptions
and may involve logarithmic corrections.

The population output $\pi_\theta$ is unavailable numerically. We replace
it by an independently generated learned output with
$M=500{,}000$ particles and measure
\begin{equation}\label{eq:experimental_reference_error}
    e_D(N)=D\bigl(\widehat\pi_\theta^N,
                         \widehat\pi_\theta^M\bigr).
\end{equation}
For $D=W_2$, the triangle inequality gives
\begin{equation}\label{eq:experimental_finite_reference}
    \left|W_2(\widehat\pi_\theta^N,\widehat\pi_\theta^M)
          -W_2(\widehat\pi_\theta^N,\pi_\theta)\right|
    \leq W_2(\widehat\pi_\theta^M,\pi_\theta).
\end{equation}
When the rate in \eqref{eq:experimental_dimension_rates} applies, the
expected reference contribution is $O(r_d(M))$. The benchmark $r_d(M)$
is smaller than $r_d(N)$ when $N\ll M$; near the largest test sizes,
reference error can still affect the fitted slope. We interpret the results as
comparisons with a large finite reference. Appendix~\ref{app:experimental_rate_fits}
quantifies this distinction.

In one dimension, $W_2$ can be computed using quantiles. In dimensions
$d\geq2$, exact optimal transport against the reference ensemble is too
costly for repeated evaluations. We use a Sinkhorn approximation on the
$W_2$ scale and the sliced Wasserstein distance $\mathrm{SW}_2$
\citep{feydy2019sinkhorn,bonneel2015sliced}. The former is the square root
of a suitably normalized, debiased entropic transport cost; at fixed
regularization it is a surrogate for $W_2$. Sliced Wasserstein averages
one-dimensional transport costs over projections. Its empirical rate is
$N^{-1/2}$ under an integrated quantile regularity condition
\citep{manole2022minimax}, with possible logarithmic corrections for
unbounded laws \citep{nietert2022statistical}. We thus use $N^{-1/2}$ as
the dimension-independent sliced benchmark. Although $\mathrm{SW}_2\leq W_2$
already transfers convergence in Euclidean space, our theorem does not
establish this sharper sliced rate in higher dimensions. The sliced
experiments examine whether such a refinement is supported numerically.
The metric definitions, periodic geometry, and numerical parameters are
specified in Appendix~\ref{app:experimental_metrics}.

We evaluate the seven operators in dimensions $d=1,2,3,5,10$ at
$N=2^8,2^9,\ldots,2^{16}$, using one trained model per operator and
dimension, for a total of $35$ models. The $d=5,10$ cases extend the tests
to higher dimensions with the same particle-count grid, reference size,
and random-stream construction. Their larger architectures and H100 training
configurations are specified in Appendix~\ref{app:experimental_training}.
All models use training seed $7$. For each model, input-measure
seeds $0,1,2$ select three test laws. For each law, we average $R_d$
independent particle-sampling realizations, with $R_1=100$, $R_2=25$, $R_3=10$,
and $R_5=R_{10}=5$. Each law has one independent
$M=500{,}000$-particle learned reference, shared across particle counts
and sampling replicates.

These averages estimate the expected distance to the fixed learned
reference. In $d=1$, they approximate
$\bbE W_2(\widehat\pi_\theta^N,\pi_\theta)$ up to the finite-reference
error in \eqref{eq:experimental_finite_reference}; in higher dimensions,
they estimate the expected Sinkhorn and sliced errors. The repeat counts
balance sampling variability and evaluation cost.
Appendix~\ref{app:experimental_sampling_repeats} compares different
sampling counts for selected problems in $d=1,2,3$.

Figures show the mean and one sampling standard error for input-measure
seed $0$, conditional on its reference. Tables summarize log--log decay
exponents fitted separately to each sampling replicate over all nine
particle counts: the $R_d$ exponents are averaged within each law, and the
resulting three rates are reported as mean $\pm$ SE. These summaries are
conditional on the trained model. Only the multiplicative constant of
each theoretical guide is fitted. Appendix~\ref{app:experimental_rate_fits}
gives the random-stream construction and fitting formulas.

\subsection{Operators and input distributions}
\label{subsec:experimental_operators}

We consider seven operators, organized into four families. All have target
laws determined analytically from the input-measure parameters, from which
independent samples can be drawn. Parameter ranges and the seeds defining
the fixed tilting potentials are specified in
Appendix~\ref{app:experimental_generators}.
We distinguish dependence on the input law from nonlinearity in the measure.
Since probability measures form a convex set rather than a vector space,
{measure-linearity} here means preservation of convex mixtures, wherever these
mixtures belong to the stated domain.

\paragraph{Doubling on the torus.}
Let $\mathbb T^d=[0,1)^d$ with the periodic distance, and define
\[
    \sG^\dagger_{\mathrm{dbl}}:\cP(\mathbb T^d)\to\cP(\mathbb T^d),
    \qquad
    \sG^\dagger_{\mathrm{dbl}}(\mu)=(\Psi_d)_\sharp\mu,
    \qquad \Psi_d(x)=2x\pmod1.
\]
Inputs are product Gaussian laws truncated to $[0,1)^d$. Their densities
are bounded above and away from zero, so this experiment provides the
compact setting for the rate comparison. The operator {satisfies measure-linearity};
it supplies a simple test before introducing distribution-dependent
transformations. The network uses sine--cosine input coordinates and an
angular output map to respect periodicity.

\paragraph{Gaussian transformations.}
Let $\mathcal G_d=\{\normal(m,\Sigma):m\in\bbR^d,
\ \Sigma\succ0\}$. We consider the two maps from $\mathcal G_d$ to itself
\begin{align*}
    \sG^\dagger_{\mathrm{scale}}\bigl(\normal(m,\Sigma)\bigr)
        &=\normal(m/2,\Sigma/2),\\
    \sG^\dagger_{\mathrm{inv}}\bigl(\normal(m,\Sigma)\bigr)
        &=\normal(m,\Sigma^{-1}).
\end{align*}
For example, rescaling can be realized by the measure-dependent pointwise
map $x\mapsto m/2+(x-m)/\sqrt2$. Unlike torus doubling, these tasks
require information about the input law. Since $\mathcal G_d$ is not
closed under convex mixtures, measure-linearity is addressed through the
componentwise extensions below, rather than through the parameter formulas.

\paragraph{Gaussian-mixture transformations.}
Let $\mathcal M_d$ be the class of mixtures of at most three nondegenerate
Gaussian components. For
$\mu=\sum_{k=1}^K w_k\normal(m_k,\Sigma_k)\in\mathcal M_d$, set
\begin{align*}
    \sG^\dagger_{\mathrm{mix,scale}}(\mu)
        &=\sum_{k=1}^K w_k\normal(m_k/2,\Sigma_k/2),\\
    \sG^\dagger_{\mathrm{mix,inv}}(\mu)
        &=\sum_{k=1}^K w_k\normal(m_k,\Sigma_k^{-1}).
\end{align*}
Both operators map $\mathcal M_d$ to itself. Component labels do not
affect the output; identical components can be merged without changing the
result. Both maps preserve convex mixtures whenever the combined law lies
in $\mathcal M_d$:
\[
    \sG^\dagger_{\mathrm{mix},\bullet}(t\mu+(1-t)\nu)
    =t\sG^\dagger_{\mathrm{mix},\bullet}(\mu)
     +(1-t)\sG^\dagger_{\mathrm{mix},\bullet}(\nu),
    \qquad 0\leq t\leq1.
\]
Their natural extensions to all finite Gaussian mixtures {therefore
satisfy measure-linearity}, despite the need to infer component parameters
from samples. These problems retain Gaussian tails while introducing
multimodality and the need to infer component structure from the ensemble.
We treat them as a test of the extent to which the observed convergence
persists beyond the single-Gaussian setting.

\begin{remark}[Gaussian tails]
\label{rem:experimental_gaussian_tails}
A nondegenerate Gaussian law and a finite Gaussian mixture both have
quadratic exponential tails: they belong to $\cP_{2,\tau}$ for a
sufficiently large $\tau$, but not to $\cP_{2+\delta,\tau}$ for
$\delta>0$. Thus both families require the quadratic-tail branch of
Assumption~\ref{ass:self_noncompact_input}, including the first-layer
smallness conditions in \eqref{eq:self_small_theta_sufficient}.
As explained in Remark~\ref{rem:self_attention_examples}, a Gaussian law
also satisfies the separate strong sub-Gaussian condition. Gaussian tail
decay alone does not establish this stronger condition for a mixture
with unequal component covariances. For example, under a linear
exponential tilt by $s$, component $k$ has mean $m_k+\Sigma_k s$;
the between-component covariance need not stay bounded uniformly in $s$.
The strong sub-Gaussian condition must therefore be verified separately
for the mixture laws. The same considerations apply after the affine
input lift, and to the Gaussian source laws used in tilting below.
We do not enforce the smallness conditions during training, so these
experiments examine convergence without certifying all noncompact
hypotheses for the trained models.
\end{remark}

\paragraph{Exponential tilting.}
For a fixed potential $u:\bbR^d\to\bbR$, define
\begin{equation}\label{eq:experimental_tilting}
    \sG^\dagger_u(\mu)(dx)
       =\frac{e^{-u(x)}}{Z_u(\mu)}\mu(dx),
    \qquad Z_u(\mu)=\int e^{-u(x)}\mu(dx).
\end{equation}
Its domain is
$\mathcal X_u=\{\mu\in\cP_2(\bbR^d):0<Z_u(\mu)<\infty,
\ \int\|x\|_2^2e^{-u(x)}\mu(dx)<\infty\}$,
and its codomain is $\cP_2(\bbR^d)$. The normalization makes this a
generally nonlinear operator on measures: mixing two source laws changes
their output weights in proportion to their respective normalizing
constants. We choose source and target densities
proportional to $e^{-V_s}$ and $e^{-V_t}$, and set $u=V_t-V_s$.
For Gaussian tilting, $V_s$ and $V_t$ are quadratic potentials. For cubic
tilting, they have the form
\[
    V(x)=\|C^{-1}(x-m)\|_3^3,\qquad C\text{ invertible}.
\]
The latter laws satisfy both the faster-than-Gaussian tail condition and
strong sub-Gaussianity; Appendix~\ref{app:experimental_cubic_ssg} verifies
these properties for the specific family used here. In each tilting
experiment, $u$ is fixed while input and target laws vary together:
we add a common random quadratic--linear potential $h$ to both $V_s$
and $V_t$. The input density is proportional to $e^{-V_s-h}$ and its
true output to $e^{-V_t-h}$, so their potential difference remains $u$.
Appendix~\ref{app:experimental_random_tilting} specifies this construction.

\subsection{Convergence measured by optimal transport}
\label{subsec:experimental_sinkhorn}

Figure~\ref{fig:reference_sinkhorn} compares the learned outputs with the
large learned reference. Table~\ref{tab:experimental_sinkhorn_rates}
summarizes rates across the three test measures. The $d=2$ comparison
exponent $0.4381$ is the log--log fit of $(\log N/N)^{1/2}$ over the
same nine-point grid; its asymptotic algebraic exponent is $1/2$.

\begin{figure}[t]
\centering
\includegraphics[width=0.97\textwidth]{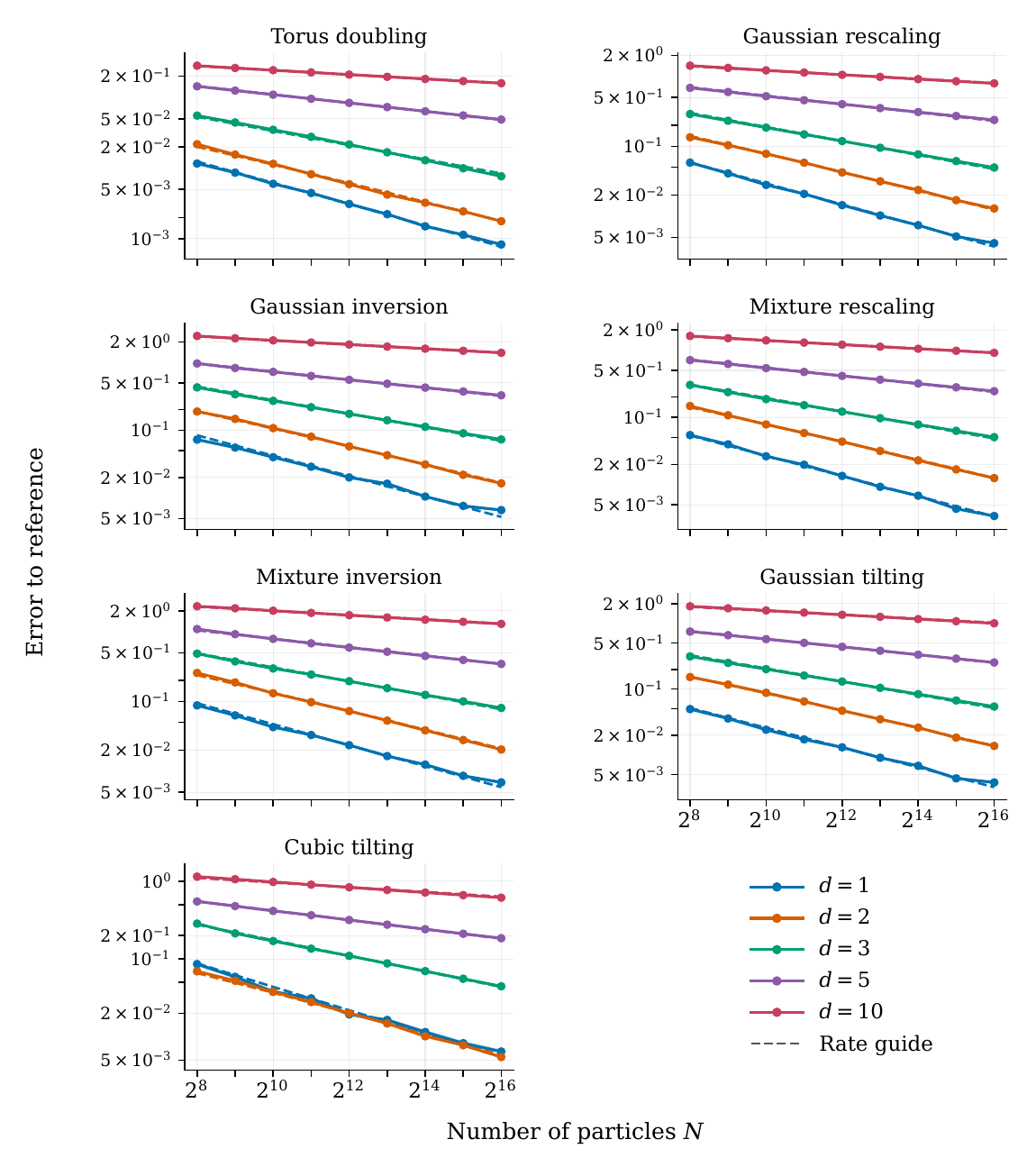}
\caption{Convergence to the learned reference for all seven operators.
Each panel shows $W_2$ in $d=1$ and the Sinkhorn approximation in $d=2,3,5,10$;
torus doubling uses periodic ground distance. Curves use input-measure
seed $0$ and average $R_d$ sampling replicates as specified in
Section~\ref{subsec:experimental_protocol}, with one sampling standard
error conditional on the shared reference shown by shading. Dashed guides have shapes $r_d(N)$ with fitted
multiplicative constants. The horizontal axis is shared; vertical limits
and ticks are chosen separately for each panel.}
\label{fig:reference_sinkhorn}
\end{figure}

\begin{table}[t]
\centering
\caption{Numerical decay exponents for $W_2$ ($d=1$) and the Sinkhorn
approximation ($d=2,3,5,10$). Each measure-level rate averages $R_d$
separately fitted sampling exponents; entries report mean $\pm$ SE
across the three input measures. The last row is the $W_2$ comparison on the finite grid,
not a theorem for fixed-regularization Sinkhorn divergence.}
\label{tab:experimental_sinkhorn_rates}
\footnotesize
\setlength{\tabcolsep}{3pt}
\begin{tabular}{lccccc}
\toprule
Operator & $d=1$ & $d=2$ & $d=3$ & $d=5$ & $d=10$ \\
\midrule
Torus doubling & $0.4806 \pm 0.0013$ & $0.4502 \pm 0.0017$ & $0.3543 \pm 0.0018$ & $0.1941 \pm 0.0010$ & $0.1035 \pm 0.0002$ \\
Gaussian rescaling & $0.4904 \pm 0.0080$ & $0.4325 \pm 0.0023$ & $0.3167 \pm 0.0008$ & $0.1904 \pm 0.0004$ & $0.1049 \pm 0.0001$ \\
Gaussian inversion & $0.4609 \pm 0.0171$ & $0.4465 \pm 0.0019$ & $0.3216 \pm 0.0016$ & $0.1948 \pm 0.0007$ & $0.1046 \pm 0.0012$ \\
Mixture rescaling & $0.4950 \pm 0.0057$ & $0.4421 \pm 0.0015$ & $0.3202 \pm 0.0005$ & $0.1914 \pm 0.0004$ & $0.1051 \pm 0.0010$ \\
Mixture inversion & $0.4630 \pm 0.0183$ & $0.4560 \pm 0.0083$ & $0.3228 \pm 0.0012$ & $0.2038 \pm 0.0012$ & $0.1056 \pm 0.0002$ \\
Gaussian tilting & $0.4911 \pm 0.0089$ & $0.4410 \pm 0.0009$ & $0.3158 \pm 0.0002$ & $0.1966 \pm 0.0003$ & $0.1084 \pm 0.0004$ \\
Cubic tilting & $0.4750 \pm 0.0032$ & $0.4638 \pm 0.0019$ & $0.3300 \pm 0.0018$ & $0.1977 \pm 0.0011$ & $0.1135 \pm 0.0007$ \\
\midrule
$W_2$ comparison & $0.5000$ & $0.4381$ & $0.3333$ & $0.2000$ & $0.1000$ \\
\bottomrule
\end{tabular}
\end{table}

Across the seven operators, the mean exponents range over $0.461$--$0.495$ in $d=1$, $0.432$--$0.464$ in $d=2$, $0.316$--$0.354$ in $d=3$, $0.190$--$0.204$ in $d=5$, $0.103$--$0.113$ in $d=10$.
{The main trend is a slower decay as dimension increases: the fitted
exponents are near $1/2$ in one dimension, near the logarithmically
corrected guide in two dimensions, and near $1/d$ in dimensions
$d=3,5,10$.}
This trend appears for compact torus inputs, Gaussian and mixture inputs,
and cubic-tail inputs. In particular, the mixture transformations and
normalized tilting operators have rates comparable to the simpler
single-Gaussian transformations at the same dimension. Thus the observed
particle-count dependence persists when the target task requires
recovering mixture structure or applying a nonlinear normalization.
These comparisons concern the convergence of each trained map to its own
population output; they do not imply equal accuracy on the target tasks.

There are also differences within each dimension. For example, Gaussian
inversion in $d=1$ has a mean exponent $0.4609$, and torus
doubling in $d=3$ gives $0.3543$. Such variation is compatible
with an upper bound and a fit over a finite range of particle counts.
Reference error, distribution-dependent corrections, and entropic
regularization can also affect the slopes. The higher-dimensional results
compare a Sinkhorn surrogate with Wasserstein benchmarks; the
fixed-regularization surrogate is outside the theorem.
Appendices~\ref{app:experimental_theory_scope} and
\ref{app:experimental_sinkhorn_ablation} discuss the input assumptions
and sensitivity to regularization.

\subsection{Convergence measured by sliced Wasserstein}
\label{subsec:experimental_sliced}

Figure~\ref{fig:reference_sliced} and
Table~\ref{tab:experimental_sliced_rates} compare sliced Wasserstein
errors with the dimension-independent benchmark $N^{-1/2}$. This tests
whether the one-dimensional projections used by the sliced metric yield
faster observed convergence than the full transport comparison in
$d=2,3,5,10$.

\begin{figure}[t]
\centering
\includegraphics[width=0.97\textwidth]{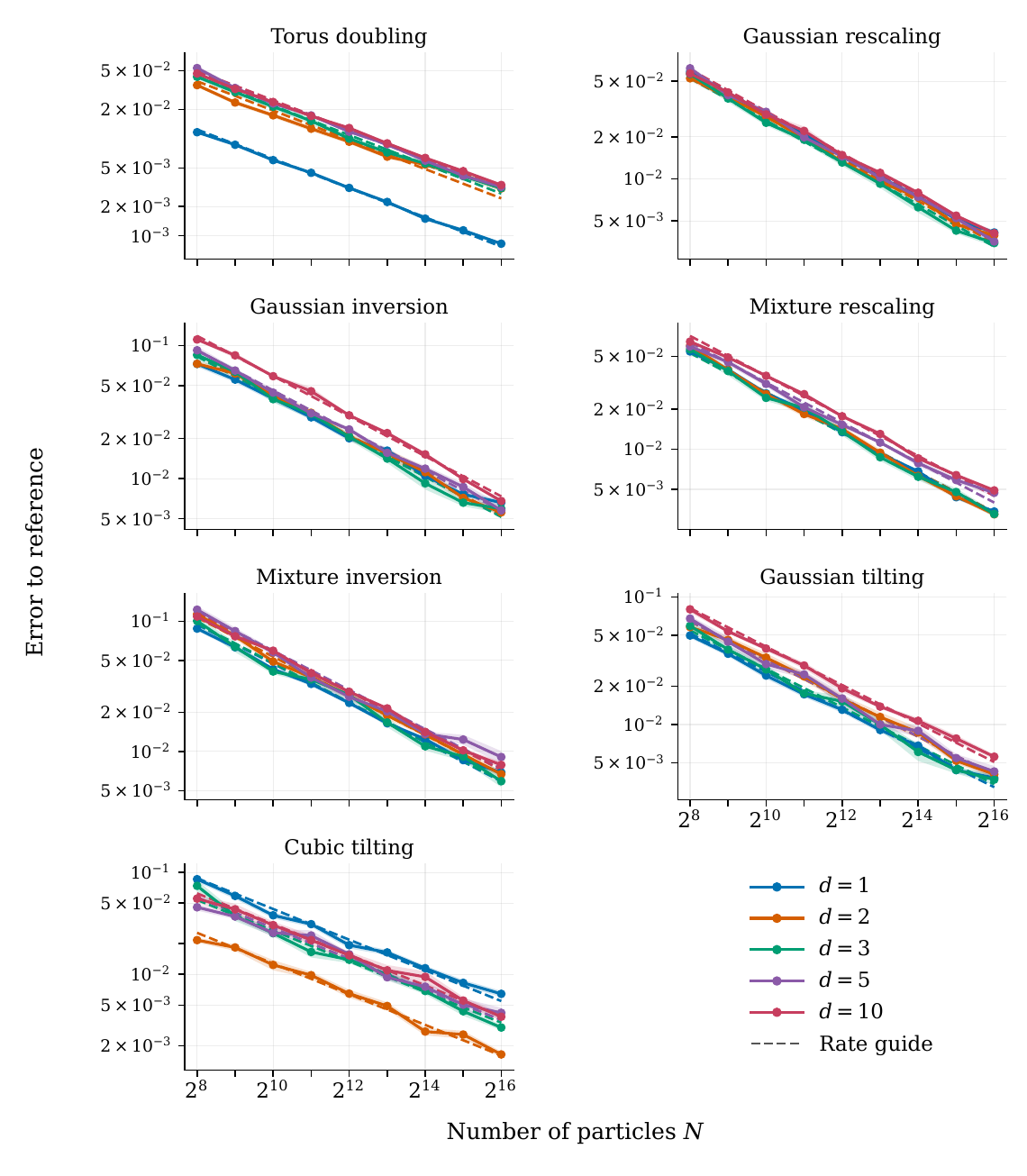}
\caption{Sliced comparison with the learned reference in dimensions
$d=1,2,3,5,10$. The $d=1$ curves
retain $W_2$ (periodic $W_2$ for torus doubling). Higher-dimensional torus
curves use the sine--cosine embedding; the others use Euclidean projections.
Curves show input-measure seed $0$, averaged over $R_d$ sampling
replicates, with one sampling standard error conditional on the shared
reference. Dashed guides have slope $-1/2$.
Each panel has its own vertical limits and ticks.}
\label{fig:reference_sliced}
\end{figure}

\begin{table}[t]
\centering
\caption{Sliced-comparison decay exponents, reported as mean $\pm$ SE
across three measure-level rates, each averaging $R_d$ per-sampling fits.
The $d=1$ entries use $W_2$ and coincide with
Table~\ref{tab:experimental_sinkhorn_rates}.}
\label{tab:experimental_sliced_rates}
\footnotesize
\setlength{\tabcolsep}{3pt}
\begin{tabular}{lccccc}
\toprule
Operator & $d=1$ & $d=2$ & $d=3$ & $d=5$ & $d=10$ \\
\midrule
Torus doubling & $0.4806 \pm 0.0013$ & $0.4806 \pm 0.0233$ & $0.4964 \pm 0.0082$ & $0.4998 \pm 0.0025$ & $0.4766 \pm 0.0039$ \\
Gaussian rescaling & $0.4904 \pm 0.0080$ & $0.4916 \pm 0.0051$ & $0.5084 \pm 0.0012$ & $0.5137 \pm 0.0044$ & $0.4874 \pm 0.0047$ \\
Gaussian inversion & $0.4609 \pm 0.0171$ & $0.4885 \pm 0.0034$ & $0.4985 \pm 0.0057$ & $0.4996 \pm 0.0074$ & $0.4962 \pm 0.0047$ \\
Mixture rescaling & $0.4950 \pm 0.0057$ & $0.5068 \pm 0.0079$ & $0.5029 \pm 0.0130$ & $0.4824 \pm 0.0069$ & $0.4866 \pm 0.0035$ \\
Mixture inversion & $0.4630 \pm 0.0183$ & $0.4925 \pm 0.0128$ & $0.4999 \pm 0.0061$ & $0.4728 \pm 0.0129$ & $0.4986 \pm 0.0120$ \\
Gaussian tilting & $0.4911 \pm 0.0089$ & $0.4936 \pm 0.0056$ & $0.5012 \pm 0.0045$ & $0.4913 \pm 0.0061$ & $0.4901 \pm 0.0081$ \\
Cubic tilting & $0.4750 \pm 0.0032$ & $0.4971 \pm 0.0096$ & $0.5115 \pm 0.0115$ & $0.4821 \pm 0.0117$ & $0.5078 \pm 0.0164$ \\
\midrule
Sliced comparison & $0.5000$ & $0.5000$ & $0.5000$ & $0.5000$ & $0.5000$ \\
\bottomrule
\end{tabular}
\end{table}

Across the seven operators, the mean exponents range over $0.461$--$0.495$ in $d=1$, $0.481$--$0.507$ in $d=2$, $0.496$--$0.512$ in $d=3$, $0.473$--$0.514$ in $d=5$, $0.477$--$0.508$ in $d=10$.
The sliced results show much less dependence on dimension than the
transport results. In $d=3$, all seven mean exponents lie between $0.496$
and $0.512$, compared with $0.316$--$0.354$ for the Sinkhorn comparison.
{The same contrast holds in $d=5,10$: sliced exponents remain
near $1/2$, while Sinkhorn exponents decrease towards $1/5$ and $1/10$.}
The same behavior occurs for Gaussian mixtures and both tilting families,
so it is not confined to a single input family or {a target
operator satisfying measure-linearity}.

{The observed $N^{-1/2}$ sliced rate is not established
by our theorem:} attention couples the output particles, so standard
one-dimensional empirical-measure estimates do not apply directly to
their projections.

\subsection{Approximation to the true operator}
\label{subsec:experimental_approximation}

Figures~\ref{fig:approximation_sinkhorn}
and~\ref{fig:approximation_sliced} compare the learned outputs with
independent samples from the true output law $\sG^\dagger(\mu)$.
For many tasks, the error initially decreases with $N$ and then levels
off. In one-dimensional mixture inversion, for example, the mean $W_2$
error falls from about $0.217$ at $N=256$ to $0.203$ at $N=1024$, but
remains near $0.196$ at the largest particle count. For two-dimensional
cubic tilting, the mean Sinkhorn error is about $0.078$ at $N=256$,
$0.035$ at $N=16{,}384$, and $0.032$ at $N=65{,}536$. The onset of the
plateau varies: one-dimensional cubic tilting remains between about
$0.078$ and $0.103$ across the grid, while several three-dimensional Sinkhorn curves
continue to decrease at the largest particle counts.
The $d=5,10$ Sinkhorn curves also continue to decrease over this grid,
whereas their sliced errors reveal flattening more clearly. For Gaussian
rescaling, the sliced error at $N=65{,}536$ is about $0.016$ in both
dimensions, changing by less than $2\%$ when $N$ increases from
$32{,}768$ to $65{,}536$. For mixture inversion, the corresponding
endpoint errors are about $0.043$ in $d=5$ and $0.034$ in $d=10$,
{again changing by less than $2\%$ over the final doubling.}

This behavior is consistent with a bottleneck in the approximation of
the true operator by the fixed learned map. To see the distinction, write
$a_\theta(\mu)=W_2(\sG_\theta(\mu),\sG^\dagger(\mu))$. The reverse
triangle inequality gives
\[
 \left|W_2(\sG_\theta(\mu^N),\sG^\dagger(\mu))-a_\theta(\mu)\right|
 \leq W_2(\sG_\theta(\mu^N),\sG_\theta(\mu)).
\]
{As the particle-count error vanishes, the error relative to the true
operator approaches $a_\theta(\mu)$, which is independent of $N$.}

The sliced comparisons show the same leveling more clearly in several
three-dimensional tasks. For Gaussian rescaling, the sliced error falls
from about $0.065$ at $N=256$ to $0.027$ at $N=4096$ and is between $0.024$ and $0.025$
at the final three particle counts.
Our measurements also include
finite-target sampling error and, in higher dimensions, errors from
transport regularization or a finite number of projections. The plateau
therefore does not isolate model capacity from optimization error or
these numerical contributions. It does show why particle-count
convergence alone is insufficient to assess accuracy on the target
operator.

\begin{figure}[t]
\centering
\includegraphics[width=0.97\textwidth]{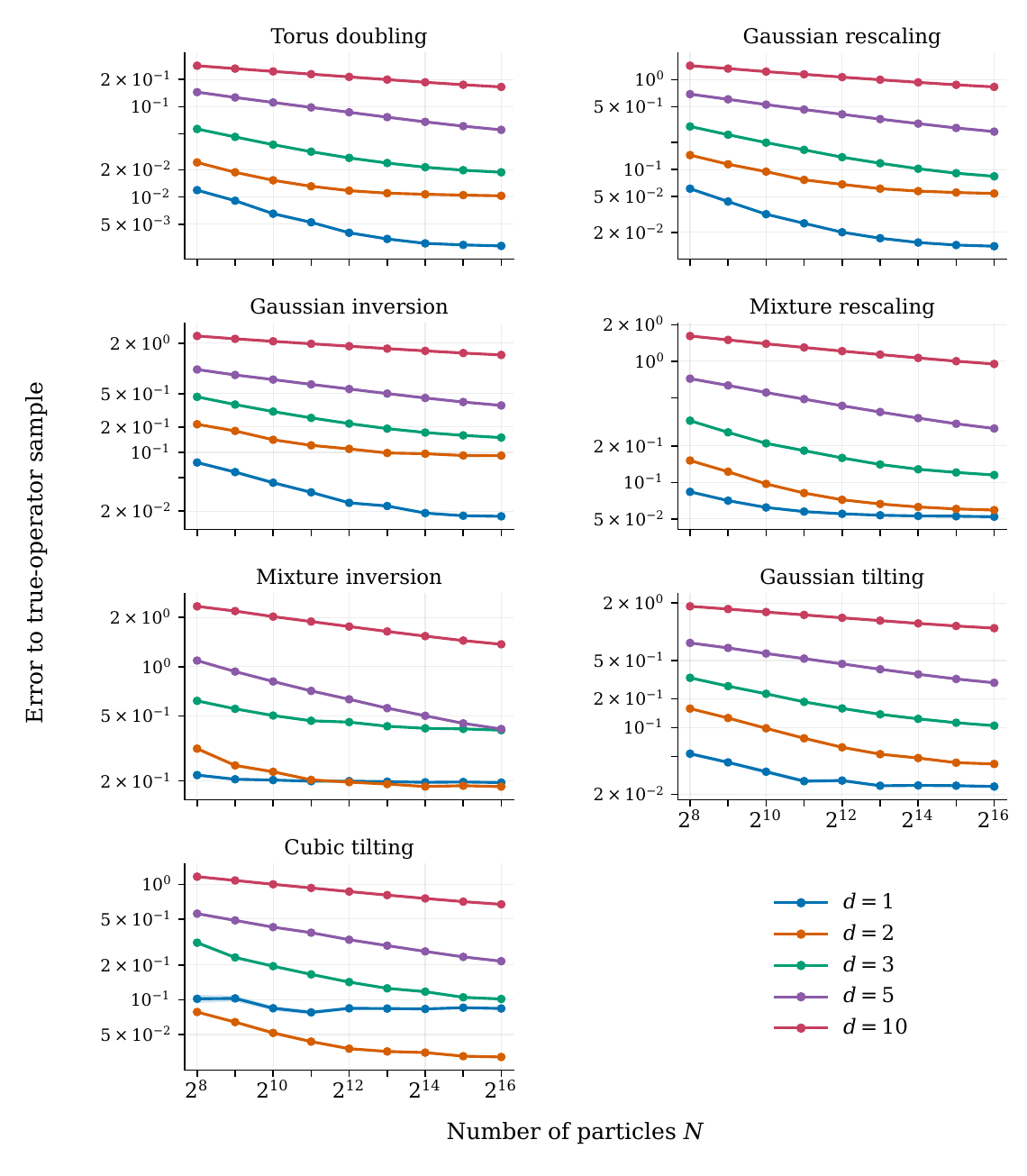}
\caption{Approximation to the true operator using
$W_2$ in $d=1$ and Sinkhorn in $d=2,3,5,10$. All curves use input-measure seed
$0$ and the mean over $R_d$ sampling replicates, with one sampling
standard error. The true-output law is fixed; independent target particles are
drawn in each sampling replicate. Panels have individual vertical scales.}
\label{fig:approximation_sinkhorn}
\end{figure}

\begin{figure}[t]
\centering
\includegraphics[width=0.97\textwidth]{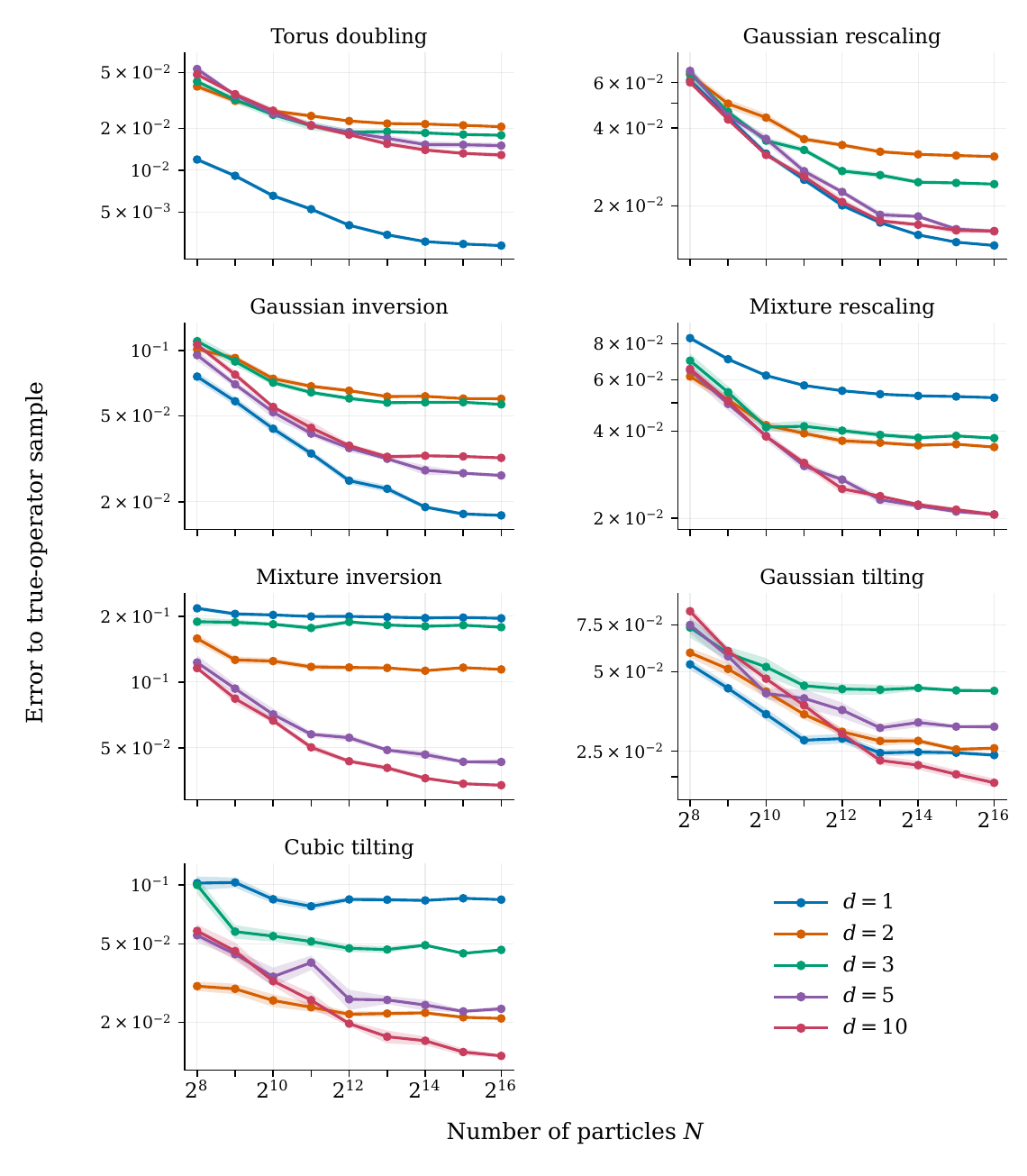}
\caption{Approximation to the true operator using sliced Wasserstein
in $d=2,3,5,10$, retaining $W_2$ in $d=1$. Seed selection, averaging and bands
match Figure~\ref{fig:approximation_sinkhorn}.}
\label{fig:approximation_sliced}
\end{figure}

\section{Conclusions}
\label{sec:conclusion}

We introduced and defined measure neural mappings through measure-dependent
transport and {ensemble consistency}, and established quantitative continuum
limits for attention-based architectures. For a fixed self-attention
transformer $\sT$ and an empirical law $\mu^N$ formed from $N$ independent
samples of $\mu$, the noncompact analysis gives
\begin{equation}\label{eq:conclusion_sampling_rate}
    \bbE W_p\bigl(\sT(\mu^N),\sT(\mu)\bigr)
    \leq C\bbE W_p(\mu^N,\mu)+CN^{-1/2},
\end{equation}
under the architecture and input-law assumptions of
Theorem~\ref{thm:transformer_stack_limit} and its multihead extension in
Corollary~\ref{cor:multihead_extension}, with $C$ independent of $N$.
This is a Wasserstein Lipschitz-type estimate for empirical approximation:
the first term propagates the input sampling error, while the additional
$N^{-1/2}$ term controls the error from replacing population attention
integrals by empirical averages. It does not assert a global Lipschitz
bound between arbitrary noncompact input laws. If the input empirical
error has rate $r(N)$, the output error is therefore
$O(r(N)+N^{-1/2})$. For regular compact input laws in the $W_2$ setting
of Section~\ref{subsec:experimental_protocol}, this retains the rates
$N^{-1/2}$ in dimension one, $(\log N/N)^{1/2}$ in dimension two, and
$N^{-1/d}$ in dimensions $d\geq3$.

Pre-norm attention is covered by compact stability after normalization,
without the additional attention sampling term. The noncompact estimates
extend the result to post-norm stacks, whose first attention layer may act
on an unbounded law. The cross-attention results distinguish sampling of
queries from sampling of keys; prescribed-query blocks admit an
$N^{-1/2}$ bound without the empirical input Wasserstein term.
For grouped self-attention stacks with $G$ groups of size $N_G=N/G$,
the additional term becomes $N_G^{-1/2}=\sqrt{G/N}$, while the empirical
input term still uses all $N$ particles. Thus $G=o(N)$ suffices for
convergence. For Set Transformers with self-attention encoders,
fixed-seed pooling, and Lipschitz readouts, the expected {output error} is
$O(N^{-1/2})$.

The experiments on seven operators in dimensions $d=1,2,3,5,10$ support
the predicted dependence on ensemble size. In the higher-dimensional
tests, Sinkhorn decay exponents remain close to the $W_2$ benchmarks
$1/5$ and $1/10$, while sliced Wasserstein decay stays close to the
dimension-independent rate $N^{-1/2}$. These are finite-reference
comparisons: the fixed-regularization Sinkhorn surrogate is outside the
theorem, and the sharper sliced rate remains an empirical observation.
The experiments also distinguish convergence to the learned operator
from approximation of the true operator: increasing the ensemble size
can reduce sampling error until the error of the fixed learned map
becomes the dominant contribution.

\subsection{Limitations}

The results concern a fixed network as the ensemble size increases. They do
not bound how well the network approximates a target operator, or how that
error depends on model capacity, training data, and optimization.
For attention acting on unbounded inputs before normalization,
the noncompact estimates require strong sub-Gaussianity and additional
tail or integrability conditions, which exclude general heavy-tailed laws.

\subsection{Future Work}

One direction is to replace the exponential in softmax attention by a
positive function $\varphi:\bbR\to(0,\infty)$. Slower growth of $\varphi$
may permit weaker tail assumptions and extend the theory to heavy-tailed
distributions with suitable finite moments \citep{hargenrader2026attention}. The analysis must control
both the normalizing denominator and the stability of the attention map.

A second direction is to study continuum limits for other
permutation-equivariant architectures, including graph neural networks.
For graphs constructed from input particles, the construction of edges,
the normalization of message aggregation, and the scaling of neighborhoods
with ensemble size affect consistency. We aim to identify when these
updates admit a measure-dependent transport representation and a continuum
limit.

\appendix

\section{Supporting Estimates for Attention and Transformers}
\label{app:theoretical_details}

This appendix provides the estimates used to prove the continuum limits in
Section~\ref{sec:transformer_mnm}. In the attention proofs, we set $V=I$
unless a value matrix is explicitly retained in a residual or block
estimate; general value matrices are recovered through
\eqref{eq:value_map_reduction}. The usage notes identify where each
auxiliary result enters the argument.

Subsection~\ref{app:strong_subgaussian_estimates} relates strong
sub-Gaussianity to the Lipschitz dependence of attention on its query.
Subsection~\ref{app:compact_estimates} develops deterministic stability
bounds on compact sets, including the bounds needed after layer
normalization. Subsection~\ref{app:noncompact_key_estimates} controls
empirical key averages for a fixed query on an unbounded domain.
Subsection~\ref{app:noncompact_self_estimates} then treats self-attention,
where the query also appears among the sampled keys, using a leave-one-out
argument and a bound on the diagonal interaction.

\subsection{Strong Sub-Gaussianity and Lipschitz Attention}
\label{app:strong_subgaussian_estimates}

\begin{lemma}[Covariance Bound]\label{lem:sg_cov_bound}
Let $X\sim\mu\in\mathcal{SG}_{a_{\mu}}(\bbR^d)$. Then
\begin{equation}
    \|\mathrm{Cov}(X)\|_2\leq a_{\mu}^2.
\end{equation}
\end{lemma}

\lemusage{Used inside Lemma~\ref{lem:ssg_consequences} to derive the tilted covariance bound from the SSG assumption.}

\begin{proof}
For any unit vector $u$, put $Z=\langle u,X-\bbE X\rangle$. The sub-Gaussian bound gives
\begin{equation}
    \log\bbE e^{tZ}=\frac{t^2}{2}\bbE Z^2+o(t^2)\leq\frac{a_{\mu}^2t^2}{2}.
\end{equation}
Dividing by $t^2/2$ and letting $t\to0$ yields $u^{\top}\mathrm{Cov}(X)u\leq a_{\mu}^2$. Taking the supremum over unit vectors proves the claim.
\end{proof}

\begin{lemma}[SSG Consequences]\label{lem:ssg_consequences}
Let $\mu\in\mathcal{SSG}_{a_{\mu}}(\bbR^d)$. Then:
\begin{enumerate}
    \item $\mu\in\mathcal{SG}_{a_{\mu}}(\bbR^d)$.
    \item For any $s\in\bbR^d$, the tilted measure
    \begin{equation}
        d\mu_s(x)
        =
        \frac{\exp\left(\langle s,x\rangle\right)}{Z_s}d\mu(x),
        Z_s=\int_{\bbR^d}\exp\left(\langle s,y\rangle\right)d\mu(y)
    \end{equation}
    satisfies $\mu_s\in\mathcal{SSG}_{a_{\mu}}(\bbR^d)$.
    \item For $X_s\sim\mu_s$,
    \begin{equation}
        \|\mathrm{Cov}(X_s)\|_2\leq a_{\mu}^2.
    \end{equation}
    \item All polynomial moments of $\mu$ are finite.
\end{enumerate}
\end{lemma}

\lemusage{Used in Lemma~\ref{lem:attention_lipschitz}, Proposition~\ref{prop:query_lipschitz}, Lemma~\ref{lem:bound_G}, and Lemma~\ref{lem:urbain_lemma_revised}.}

\begin{proof}
Let $X\sim\mu$. By definition,
\begin{equation}
    \psi_{\mu}(\lambda)
    =
    \log\bbE\exp\left(\langle\lambda,X\rangle\right),
    \psi_{\mu}(0)=0,
    \nabla\psi_{\mu}(0)=\bbE X.
\end{equation}
Let $X^{\circ}:=X-\bbE X$ and denote its law by $\mu^{\circ}:=\mathcal L(X^{\circ})$. Then
\begin{equation}
    \psi_{\mu^{\circ}}(\lambda)
    =
    \psi_{\mu}(\lambda)-\langle\lambda,\bbE X\rangle.
\end{equation}
Hence $\psi_{\mu^{\circ}}\in C^2(\bbR^d)$,
\begin{equation}
    \psi_{\mu^{\circ}}(0)=0,
    \nabla\psi_{\mu^{\circ}}(0)=0,
    \nabla^2\psi_{\mu^{\circ}}(\lambda)=\nabla^2\psi_{\mu}(\lambda).
\end{equation}
Taylor's formula gives
\begin{equation}
\begin{split}
    \psi_{\mu^{\circ}}(\lambda)
    &=
    \psi_{\mu^{\circ}}(0)+\langle\nabla\psi_{\mu^{\circ}}(0),\lambda\rangle
    +
    \int_0^1(1-t)\langle\lambda,\nabla^2\psi_{\mu^{\circ}}(t\lambda)\lambda\rangle dt\\
    &=
    \int_0^1(1-t)\langle\lambda,\nabla^2\psi_{\mu}(t\lambda)\lambda\rangle dt\\
    &\leq
    \int_0^1(1-t)a_{\mu}^2\|\lambda\|_2^2dt
    =
    \frac{a_{\mu}^2\|\lambda\|_2^2}{2}.
\end{split}
\end{equation}
This is exactly $\mu\in\mathcal{SG}_{a_{\mu}}(\bbR^d)$.

For the tilted measure, let $X_s\sim\mu_s$. Its CGF is $\psi_{X_s}(\lambda)
=
\log\bbE^{\mu_s}\exp\left(\langle\lambda,X_s\rangle\right).$
Since
\begin{equation}
\begin{split}
    \bbE^{\mu_s}\exp\left(\langle\lambda,X_s\rangle\right)
    =
    \int_{\bbR^d}\exp\left(\langle\lambda,y\rangle\right)
    \frac{\exp\left(\langle s,y\rangle\right)}{Z_s}d\mu(y)
    =
    \frac{\bbE^\mu\exp\left(\langle\lambda+s,X\rangle\right)}
    {\bbE^\mu\exp\left(\langle s,X\rangle\right)},
\end{split}
\end{equation}
we get
\begin{equation}
    \psi_{X_s}(\lambda)=\psi_{\mu}(\lambda+s)-\psi_{\mu}(s).
\end{equation}
Therefore,
\begin{equation}
    \|\nabla^2\psi_{X_s}(\lambda)\|_2
    =
    \|\nabla^2\psi_{\mu}(\lambda+s)\|_2
    \leq
    a_{\mu}^2.
\end{equation}
Thus $\mu_s\in\mathcal{SSG}_{a_{\mu}}(\bbR^d)$. Applying the first claim to $\mu_s$ and then Lemma~\ref{lem:sg_cov_bound} gives
\begin{equation}
    \|\mathrm{Cov}(X_s)\|_2\leq a_{\mu}^2.
\end{equation}
The finite polynomial moments follow from $\mu\in\mathcal{SG}_{a_{\mu}}(\bbR^d)$.
\end{proof}

\begin{lemma}[Lipschitz Attention]\label{lem:attention_lipschitz}
Let $\fA(s;\nu)$ be the attention map with $V=I$, $s\in\bbR^{d_u}$ and $\nu\in\cP(\bbR^{d_v})$. 
The following implications hold:
\begin{enumerate}
    \item If $\nu\in\mathcal{SSG}_{a_{\nu}}(\bbR^{d_v})$, then $s\mapsto\fA(s;\nu)$ is globally Lipschitz with constant $a_{\nu}^2\|\Theta\|_2$.
    \item If $\mu\in\cP(\bbR^d)$ and the pointwise attention map with $Q=K=V=I$ is well defined on all of $\bbR^d$ and globally Lipschitz with a constant $C>0$, then $\mu\in\mathcal{SSG}_{\sqrt{C}}(\bbR^d)$. A constant map also satisfies this statement, since any positive Lipschitz bound may be used.
\end{enumerate}

\lemusage{The first implication is used in Proposition~\ref{prop:query_lipschitz}. The second implication is not used in the main proofs but explains why SSG is the natural condition for global Lipschitz attention.}
\end{lemma}

\begin{proof}
For the first implication, consider $\lambda=\Theta^{\top}s$, then $\fA(s;\nu)
=\nabla_{\lambda}\psi_{\nu}(\lambda).$
The Hessian of the cumulant generating function is the covariance under the corresponding exponential tilt, so
\begin{equation}
    \nabla_s\fA(s;\nu)
    =
    \mathrm{Cov}^{Y\sim \pp(\cdot;s,\nu)}(Y)\Theta^{\top}.
\end{equation}
The distribution $\pp(\cdot;s,\nu)$ is an exponential tilt of $\nu$. By Lemma~\ref{lem:ssg_consequences},
\begin{equation}
    \|\mathrm{Cov}^{Y\sim \pp(\cdot;s,\nu)}(Y)\|_2\leq a_{\nu}^2.
\end{equation}
Therefore,
\begin{equation}
    \|\nabla_s\fA(s;\nu)\|_2
    \leq
    a_{\nu}^2\|\Theta\|_2.
\end{equation}
The mean-value theorem gives the claimed Lipschitz estimate.

For the second implication, global well-definedness implies that $\mathcal M_{\mu}(x)$ is finite for every $x\in\bbR^d$. It is therefore smooth, as is its logarithm. With $Q=K=V=I$, we have $\Theta=I$ and $\fA(x;\mu)=\nabla\psi_{\mu}(x)$. If $\fA(\cdot;\mu)$ is $C$-Lipschitz, then
\begin{equation}
    \|\nabla^2\psi_{\mu}(x)\|_2
    =
    \|\nabla_x\fA(x;\mu)\|_2
    \leq
    C
\end{equation}
for all $x\in\bbR^d$. Hence $\mu\in\mathcal{SSG}_{\sqrt{C}}(\bbR^d)$. This converse uses $Q=K=V=I$. For a fixed degenerate $\Theta$, the Jacobian only controls $\nabla^2\psi_{\mu}(\Theta^{\top}s)\Theta^{\top}$, evaluated on the range of $\Theta^{\top}$, and need not imply SSG.
\end{proof}

\begin{proposition}[Stability in the Query Measure]\label{prop:query_lipschitz}
Fix $p\geq1$ and let $\nu\in\mathcal{SSG}_{a_\nu}(\bbR^{d_v})$. With $V=I$, the map $s\mapsto\fA(s;\nu)$ is globally Lipschitz with constant $a_\nu^2\|\Theta\|_2$. For any $\mu,\tilde\mu\in\cP_p(\bbR^{d_u})$,
\begin{equation}
    W_p\left(\fA(\cdot;\nu)_\sharp\mu,\fA(\cdot;\nu)_\sharp\tilde\mu\right)
    \leq a_\nu^2\|\Theta\|_2W_p(\mu,\tilde\mu).
\end{equation}
\end{proposition}
\begin{proof}
The pointwise estimate is Lemma~\ref{lem:attention_lipschitz}. If $\pi$ is an optimal coupling of $\mu$ and $\tilde\mu$, its image under $(x,y)\mapsto(\fA(x;\nu),\fA(y;\nu))$ gives
\begin{equation}
\begin{split}
    W_p^p\left(\fA(\cdot;\nu)_\sharp\mu,\fA(\cdot;\nu)_\sharp\tilde\mu\right)
    &\leq\int\|\fA(x;\nu)-\fA(y;\nu)\|_2^p\pi(dx,dy)\\
    &\leq a_\nu^{2p}\|\Theta\|_2^p W_p^p(\mu,\tilde\mu).
\end{split}
\end{equation}
Taking the $p$th root proves the claim. This single estimate applies to both self-attention and cross-attention.
\end{proof}

\subsection{Compact Estimates}
\label{app:compact_estimates}

\begin{lemma}[Bounded Support]\label{lem:bounded_support_ssg}
Let $\mu\in\cP(B_R(0))$, i.e. $\operatorname{supp}(\mu)\subseteq\{x\in\bbR^d:\|x\|_2\leq R\}.$
Then $\mu\in\mathcal{SSG}_R(\bbR^d)$.
\end{lemma}

\lemusage{Used in Theorem~\ref{thm:self_convergence}, Theorem~\ref{thm:cross_compact_stability}, and the transformer-stack argument through Lemma~\ref{lem:layer_norm_bounded}.}

\begin{proof}
Since $\mu$ is supported in the ball of radius $R$, for any $\lambda\in\bbR^d$ the moment generating function
\begin{equation}
    \mathcal{M}_{\mu}(\lambda)=\bbE^{X\sim\mu}\exp\left(\langle\lambda,X\rangle\right)
\end{equation}
is finite and smooth in $\lambda$. Hence $\psi_{\mu}(\lambda)=\log \mathcal{M}_{\mu}(\lambda)$ is $C^2(\bbR^d)$. Moreover, by the standard identity for cumulant generating functions,
\begin{equation}
    \nabla^2\psi_{\mu}(\lambda)
    =
    \mathrm{Cov}^{X\sim\mu_{\lambda}}(X),
\end{equation}
where
\begin{equation}
    \mu_{\lambda}(dx)
    =
    \frac{\exp\left(\langle\lambda,x\rangle\right)\mu(dx)}
    {\int_{\bbR^d}\exp\left(\langle\lambda,y\rangle\right)\mu(dy)}.
\end{equation}
The tilted measure $\mu_{\lambda}$ has the same support as $\mu$. Therefore, for every unit vector $u\in\bbR^d$,
\begin{equation}
    \mathrm{Var}^{X\sim\mu_{\lambda}}(\langle u,X\rangle)
    \leq
    \bbE^{X\sim\mu_{\lambda}}\langle u,X\rangle^2
    \leq
    R^2
    \Longrightarrow
    \|\nabla^2\psi_{\mu}(\lambda)\|_2
    \leq 
    R^2.
\end{equation}
Thus $\mu\in\mathcal{SSG}_R(\bbR^d)$.
\end{proof}

\begin{lemma}[Compact Ratio Bound]\label{lem:nonIID_W1_bound}
Let $\nu,\tilde{\nu}\in\cP(B_R(0))$ on $\bbR^{d_v}$. 
Then there exists $C=C(\Theta,R)=e^{2\|\Theta\|_2R^2}(1+\|\Theta\|_2R^2)+e^{4\|\Theta\|_2R^2}\|\Theta\|_2R^2>0$ such that, for any $s\in B_R(0)\subset\bbR^{d_u}$,
\begin{equation}
    \|\fA(s;\nu)-\fA(s;\tilde{\nu})\|_2
    \leq
    C W_1(\nu,\tilde{\nu}).
\end{equation}
\end{lemma}

\lemusage{Used in Proposition~\ref{prop:second_term_noniid}, Theorem~\ref{thm:cross_compact_stability}, Proposition~\ref{prop:compact_residual_lipschitz}, and Proposition~\ref{prop:compact_cross_residual_lipschitz}.}

\begin{proof}
Define
\begin{equation}
\begin{aligned}
    g_0(x) &= \exp\left(\langle s,\Theta x\rangle\right), &
    g_1(x) &= x\exp\left(\langle s,\Theta x\rangle\right).
\end{aligned}
\end{equation}
For $x\in B_R(0)\subset\bbR^{d_v}$ and $s\in B_R(0)\subset\bbR^{d_u}$, $|\langle s,\Theta x\rangle|\leq\|\Theta\|_2R^2.$

Set $b_R=\|\Theta\|_2R^2$. Then
\begin{subequations}
\begin{align}
    &e^{-b_R} \leq g_0(x)\leq e^{b_R}, \quad \|g_1(x)\|_2 \leq Re^{b_R},\\
    &\mathrm{Lip}(g_0,x) \leq e^{b_R}\|\Theta\|_2R, \quad \mathrm{Lip}(g_1,x) \leq e^{b_R}(1+b_R).
\end{align}
\end{subequations}

For any $\rho\in\cP(B_R(0))$, define
\begin{equation}
\begin{aligned}
    N_s(\rho) &= \int g_1(x)\rho(dx), &
    D_s(\rho) &= \int g_0(x)\rho(dx).
\end{aligned}
\end{equation}
Then $\fA(s;\rho)=N_s(\rho)/D_s(\rho)$, $D_s(\rho)\geq e^{-b_R}$, and $\|N_s(\rho)\|_2\leq Re^{b_R}$. Thus
\begin{equation}
\begin{split}
    \|\fA(s;\nu)-\fA(s;\tilde{\nu})\|_2
    &\leq
    \frac{\|N_s(\nu)-N_s(\tilde{\nu})\|_2}{D_s(\nu)}
    +
    \frac{\|N_s(\tilde{\nu})\|_2|D_s(\nu)-D_s(\tilde{\nu})|}{D_s(\nu)D_s(\tilde{\nu})}.
\end{split}
\end{equation}
By Kantorovich-Rubinstein duality,
\begin{subequations}
\begin{align}
    \|N_s(\nu)-N_s(\tilde{\nu})\|_2
    &\leq
    \mathrm{Lip}(g_1,x)W_1(\nu,\tilde{\nu}),\\
    |D_s(\nu)-D_s(\tilde{\nu})|
    &\leq
    \mathrm{Lip}(g_0,x)W_1(\nu,\tilde{\nu}).
\end{align}
\end{subequations}
Combining these estimates gives
\begin{equation}
    \|\fA(s;\nu)-\fA(s;\tilde{\nu})\|_2
    \leq
    C(\Theta,R)W_1(\nu,\tilde{\nu}).
\end{equation}
\end{proof}

\begin{proposition}[Compact Second Term]\label{prop:second_term_noniid}
{Fix $p\geq1$.} Let $\mu,\nu\in\cP(B_R(0))$ on $\bbR^d$. Then there exists $C=C(\Theta,R)>0$ such that
\begin{equation}
    W_p\left(\fA(\cdot;\mu)_{\sharp}\nu,\fA(\cdot;\nu)_{\sharp}\nu\right)
    \leq
    C W_p(\mu,\nu).
\end{equation}
\end{proposition}

\begin{proof}
Lemma~\ref{lem:nonIID_W1_bound} gives $C=C(\Theta,R)>0$ such that, for any $s\in B_R(0)$,
\begin{equation}
    \|\fA(s;\mu)-\fA(s;\nu)\|_2
    \leq
    CW_1(\mu,\nu).
\end{equation}
Using the coupling induced by the common variable $s\sim\nu$,
\begin{equation}
\begin{split}
    W_p\left(\fA(\cdot;\mu)_{\sharp}\nu,\fA(\cdot;\nu)_{\sharp}\nu\right)
    &\leq
    \left(
    \int_{B_R(0)}
    \|\fA(s;\mu)-\fA(s;\nu)\|_2^p
    \nu(ds)
    \right)^{1/p}\\
    &\leq
    CW_1(\mu,\nu)
    \leq
    CW_p(\mu,\nu).
\end{split}
\end{equation}
\end{proof}

\begin{lemma}[Layer Normalization]\label{lem:layer_norm_bounded}
Let $f_{\gamma,\beta,\epsilon}^{\mathrm{LN}}$ be as in Definition~\ref{def:layer_norm}, where $\epsilon>0$. Then for every $x\in\bbR^d$,
\begin{equation}
    \|f_{\gamma,\beta,\epsilon}^{\mathrm{LN}}(x)\|_2
    \leq
    \sqrt{d}\|\gamma\|_{\infty}+\|\beta\|_2.
\end{equation}
Consequently, for every $\rho\in\cP(\bbR^d)$,
\begin{equation}
\operatorname{supp}\left((f_{\gamma,\beta,\epsilon}^{\mathrm{LN}})_{\sharp}\rho\right)\subseteq B_{R_0}(0), \qquad R_0=\sqrt{d}\|\gamma\|_{\infty}+\|\beta\|_2.
\end{equation}
Moreover, $f_{\gamma,\beta,\epsilon}^{\mathrm{LN}}$ is globally Lipschitz. More precisely, for every $x,y\in\bbR^d$,
\begin{equation}
    \|f_{\gamma,\beta,\epsilon}^{\mathrm{LN}}(x)-f_{\gamma,\beta,\epsilon}^{\mathrm{LN}}(y)\|_2
    \leq
    \frac{\|\gamma\|_{\infty}}{\sqrt{\epsilon}}\|x-y\|_2.
\end{equation}
\end{lemma}

\lemusage{Used in Theorem~\ref{thm:transformer_stack_limit} to show that after layer normalization the empirical and continuum measures have deterministic compact support, also used in Proposition~\ref{prop:block_compact_lipschitz}.}

\begin{proof}
Here $\odot$ denotes componentwise multiplication, so that $(a\odot b)_k=a_kb_k$ for $a,b\in\bbR^d$. Let
\begin{equation}
    z_k=\frac{x_k-\bar{x}}{\sqrt{\sigma_x^2+\epsilon}}.
\end{equation}
Then
\begin{equation}
    \|z\|_2^2
    =
    \frac{\sum_{k=1}^{d}(x_k-\bar{x})^2}{\sigma_x^2+\epsilon}
    =
    \frac{d\sigma_x^2}{\sigma_x^2+\epsilon}
    \leq
    d.
\end{equation}
Therefore,
\begin{equation}
    \|f_{\gamma,\beta,\epsilon}^{\mathrm{LN}}(x)\|_2
    =
    \|\gamma\odot z+\beta\|_2
    \leq
    \|\gamma\|_{\infty}\|z\|_2+\|\beta\|_2
    \leq
    \sqrt{d}\|\gamma\|_{\infty}+\|\beta\|_2.
\end{equation}
This proves the boundedness claim, and the support inclusion follows immediately.

It remains to prove the Lipschitz estimate. Let $\mathbf e_d=(1,\ldots,1)^{\top}$ and $P=I-d^{-1}\mathbf e_d\mathbf e_d^{\top}$. Then $Px=x-\bar{x}\mathbf e_d$ and $d^{-1}\|Px\|_2^2=\sigma_x^2$, so
\begin{equation}
    f_{\gamma,\beta,\epsilon}^{\mathrm{LN}}(x)
    =
    \gamma\odot\frac{Px}{\sqrt{d^{-1}\|Px\|_2^2+\epsilon}}+\beta.
\end{equation}
For the map $u\mapsto u/\sqrt{d^{-1}\|u\|_2^2+\epsilon}$, its Jacobian at $u$ is $\frac{1}{h(u)}I-\frac{1}{d h(u)^3}uu^{\top},$
where $h(u)=\sqrt{d^{-1}\|u\|_2^2+\epsilon}.$

The eigenvalues of this matrix are $1/h(u)$ on directions orthogonal to $u$ and $\epsilon/h(u)^3$ in the direction of $u$. Hence its operator norm is at most $1/h(u)\leq 1/\sqrt{\epsilon}$. Since $\|P\|_2\leq 1$ and the operator norm of $v\mapsto\gamma\odot v$ is $\|\gamma\|_{\infty}$, we have
\begin{equation}
    \|Df_{\gamma,\beta,\epsilon}^{\mathrm{LN}}(x)\|_2
    \leq
    \frac{\|\gamma\|_{\infty}}{\sqrt{\epsilon}}.
\end{equation}

The claimed Lipschitz estimate follows from the mean value theorem.
\end{proof}

\begin{proposition}[Compact Self Residual]\label{prop:compact_residual_lipschitz}
Fix $p\geq1$. Let $\rho,\tilde{\rho}\in\cP(B_R(0))$ on $\bbR^d$, and let $\fA=\fA_{Q,K,V}$ with $V:\bbR^d\to\bbR^d$. Define
\begin{equation}
    \mathcal R^{\mathsf{S}}(\rho)
    =
    \left(\mathrm{Id}(\cdot)+\fA(\cdot;\rho)\right)_{\sharp}\rho.
\end{equation}
Then there exists $C=C(\Theta,V,R)>0$ such that
\begin{equation}
    W_p\left(\mathcal R^{\mathsf{S}}(\rho),\mathcal R^{\mathsf{S}}(\tilde{\rho})\right)
    \leq
    C W_p(\rho,\tilde{\rho}).
\end{equation}
\end{proposition}

\begin{proof}
Let $\pi$ be an optimal coupling of $\rho$ and $\tilde{\rho}$. Then
\begin{equation}
\begin{split}
    W_p\left(\mathcal R^{\mathsf{S}}(\rho),\mathcal R^{\mathsf{S}}(\tilde{\rho})\right)
    &\leq
    W_p(\rho,\tilde{\rho})
    +
    \left(
    \int
    \|\fA(x;\rho)-\fA(y;\tilde{\rho})\|_2^p
    \pi(dx,dy)
    \right)^{1/p}.
\end{split}
\end{equation}
After adding and subtracting $\fA(y;\rho)$, apply Proposition~\ref{prop:query_lipschitz}, Lemma~\ref{lem:bounded_support_ssg}, Lemma~\ref{lem:nonIID_W1_bound}, and \eqref{eq:value_map_reduction} to obtain a constant $C_1=C_1(\Theta,R)>0$ such that
\begin{equation}
    \|\fA(x;\rho)-\fA(y;\tilde{\rho})\|_2
    \leq
    \|V\|_2\left(R^2\|\Theta\|_2\|x-y\|_2+C_1W_p(\rho,\tilde{\rho})\right).
\end{equation}
Thus the claim holds with $C=1+\|V\|_2\left(R^2\|\Theta\|_2+C_1\right)$.
\end{proof}

\begin{proposition}[Compact Cross Residuals]\label{prop:compact_cross_residual_lipschitz}
Fix $p\geq1$. Let $\rho,\tilde{\rho}\in\cP(B_R(0))$ on the active state space and let $\nu,\tilde{\nu}\in\cP(B_R(0))$ on the auxiliary space. For each of the following block types, take $\fA=\fA_{Q,K,V}$ with matrices compatible with its query and key spaces and with value output dimension equal to query dimension. Define
\begin{subequations}
\begin{align}
    \mathcal R^{\mathsf{L}}(\rho;\nu)
    &=
    \left(\mathrm{Id}(\cdot)+\fA(\cdot;\nu)\right)_{\sharp}\rho,\\
    \mathcal R^{\mathsf{R}}(\rho;\nu)
    &=
    \left(\mathrm{Id}(\cdot)+\fA(\cdot;\rho)\right)_{\sharp}\nu.
\end{align}
\end{subequations}
Then there exists $C=C(\Theta,V,R)>0$ such that, for $\kappa\in\{\mathsf{L},\mathsf{R}\}$,
\begin{equation}
    W_p\left(\mathcal R^{\kappa}(\rho;\nu),\mathcal R^{\kappa}(\tilde{\rho};\tilde{\nu})\right)
    \leq
    C\left(W_p(\rho,\tilde{\rho})+W_p(\nu,\tilde{\nu})\right).
\end{equation}
\end{proposition}

\begin{proof}
Lemma~\ref{lem:nonIID_W1_bound} gives a constant $C_1=C_1(\Theta,R)>0$. For the $\mathsf{L}$ block, where the active measure is the left input, couple $x\sim\rho$ and $y\sim\tilde{\rho}$ optimally, then add and subtract $\fA(y;\nu)$. Proposition~\ref{prop:query_lipschitz}, Lemma~\ref{lem:bounded_support_ssg}, Lemma~\ref{lem:nonIID_W1_bound}, and \eqref{eq:value_map_reduction} give
\begin{equation}
    \|x-y+\fA(x;\nu)-\fA(y;\tilde{\nu})\|_2
    \leq
    (1+\|V\|_2R^2\|\Theta\|_2)\|x-y\|_2+\|V\|_2C_1W_p(\nu,\tilde{\nu}).
\end{equation}
Taking the $L^p$ norm gives the desired $\mathsf{L}$ estimate. For the $\mathsf{R}$ block, where the active measure is the right input, couple $x\sim\nu$ and $y\sim\tilde{\nu}$ optimally and add and subtract $\fA(y;\rho)$; the same bounds give
\begin{equation}
    \|x-y+\fA(x;\rho)-\fA(y;\tilde{\rho})\|_2
    \leq
    (1+\|V\|_2R^2\|\Theta\|_2)\|x-y\|_2+\|V\|_2C_1W_p(\rho,\tilde{\rho}).
\end{equation}
Thus both estimates hold with $C=1+\|V\|_2\left(R^2\|\Theta\|_2+C_1\right)$.
\end{proof}

\begin{proposition}[Compact Block]\label{prop:block_compact_lipschitz}
Fix $p\geq 1$ and let a block indexed by $\ell\in[L]$ satisfy Assumption~\ref{ass:finite_transformer_architecture}. Assume that all active and auxiliary probability measures entering the block are supported in $B_R(0)$. Then there exists $C>0$, depending only on $p$, $\Theta_{\ell}$, $V_{\ell}$, $R$, and the Lipschitz parameters of the feed-forward and layer-normalization maps, such that the corresponding block map is Lipschitz in $W_p$. More precisely, for a self-attention block,
\begin{equation}
    W_p\left(\sB_{\ell}^{\mathsf{S}}(\rho),\sB_{\ell}^{\mathsf{S}}(\tilde{\rho})\right)
    \leq
    C W_p(\rho,\tilde{\rho}),
\end{equation}
and for cross-attention blocks, with $\kappa\in\{\mathsf{L},\mathsf{R}\}$,
\begin{equation}
    W_p\left(\sB_{\ell}^{\kappa}(\rho;\nu),\sB_{\ell}^{\kappa}(\tilde{\rho};\tilde{\nu})\right)
    \leq
    C\left(W_p(\rho,\tilde{\rho})+W_p(\nu,\tilde{\nu})\right).
\end{equation}
\end{proposition}

\begin{proof}
The residual self-attention part is controlled by Proposition~\ref{prop:compact_residual_lipschitz}, and the residual cross-attention parts are controlled by Proposition~\ref{prop:compact_cross_residual_lipschitz}. These give constants $C_1=C_1(\Theta_{\ell},V_{\ell},R)>0$ for the residual part. The two layer-normalization maps are Lipschitz by Lemma~\ref{lem:layer_norm_bounded}, and $\mathrm{Id}+f_{\ell}^{\mathrm{NN}}$ is Lipschitz because $f_{\ell}^{\mathrm{NN}}$ is Lipschitz. These post-attention maps give a constant $C_2=C_2(f_{\ell}^{\mathrm{NN}},f_{\ell,1}^{\mathrm{LN}},f_{\ell,2}^{\mathrm{LN}})>0$. Applying the pushforward Lipschitz inequality to each component of the block proves the proposition with $C=C_1C_2$.
\end{proof}

\subsection{Noncompact Empirical Key Estimates}
\label{app:noncompact_key_estimates}

This subsection treats empirical approximation of the key/value measure for a fixed query $s$. It is used directly for cross-attention and, after the leave-one-out reduction in Subsection~\ref{app:noncompact_self_estimates}, for noncompact self-attention.

Let $Y_1,\ldots,Y_M$ be i.i.d. samples from $\nu\in\cP(\bbR^{d_v})$, and let
\begin{equation}
    \nu^M=\frac{1}{M}\sum_{j=1}^{M}\delta_{Y_j}
\end{equation}
be the corresponding empirical key/value measure. For a fixed query $s\in\bbR^{d_u}$, {recall from Definition~\ref{def:attention_on_measures} that} the empirical and population attention maps are
\begin{subequations}\label{eq:weighted_attention_notation}
\begin{align}
    \fA(s;\nu^M)
    &=
    \frac{\sum_{j=1}^{M}Y_j\exp\left(\langle s,\Theta Y_j\rangle\right)}
    {\sum_{j=1}^{M}\exp\left(\langle s,\Theta Y_j\rangle\right)},\\
    \fA(s;\nu)
    &=
    \frac{\int_{\bbR^{d_v}}x\exp\left(\langle s,\Theta x\rangle\right)\nu(dx)}
    {\int_{\bbR^{d_v}}\exp\left(\langle s,\Theta x\rangle\right)\nu(dx)}.
\end{align}
\end{subequations}
Recall the MGF (Definition~\ref{def:MGF_CGF}):
\begin{equation}
    \int_{\bbR^{d_v}}\exp\left(\langle s,\Theta x\rangle\right)\nu(dx)
    =
    \mathcal{M}_{\nu}(\Theta^{\top}s).
\end{equation}
We assume throughout this subsection that the numerator and denominator in the attention maps \eqref{eq:weighted_attention_notation} are finite. Define the lower-tail event
\begin{equation}
    \Omega_s
    =
    \left\{
    \frac{1}{M}\sum_{j=1}^{M}\exp\left(\langle s,\Theta Y_j\rangle\right)
    \leq
    \frac{\mathcal{M}_{\nu}(\Theta^{\top}s)}{2}
    \right\}.
\end{equation}

\begin{lemma}[Bound $\Omega_s$]\label{lem:untruncated_concentration}
Assume $q\geq 2$ and
\begin{equation}
    \mathcal{M}_{\nu}(q\Theta^{\top}s)<\infty.
\end{equation}
Then there exists $C=C(q)>0$ such that
\begin{equation}
    \bbP(\Omega_s)
    \leq
    \frac{2^{2q-1}C}{M^{q/2}}
    \left(
    1+
    \frac{\mathcal{M}_{\nu}(q\Theta^{\top}s)}
    {\mathcal{M}_{\nu}(\Theta^{\top}s)^q}
    \right).
\end{equation}

\lemusage{Used in Lemma~\ref{lem:stronger_Urbain_lemma}. Through Lemma~\ref{lem:stronger_Urbain_lemma}, it is used in Lemma~\ref{lem:urbain_lemma_revised}, Theorem~\ref{thm:cross_noncompact}, and Proposition~\ref{prop:second_term}.}
\end{lemma}

\begin{proof}
Let
\begin{equation}
    Z_j=\mathcal{M}_{\nu}(\Theta^{\top}s)-\exp\left(\langle s,\Theta Y_j\rangle\right).
\end{equation}
Then $\bbE Z_j=0$. Moreover,
\begin{equation}
    \Omega_s
    \subseteq
    \left\{
    \left|\frac{1}{M}\sum_{j=1}^{M}Z_j\right|
    \geq
    \frac{\mathcal{M}_{\nu}(\Theta^{\top}s)}{2}
    \right\}.
\end{equation}
By Markov's inequality and the Marcinkiewicz-Zygmund inequality, there exists $C_1=C_1(q)>0$ such that
\begin{equation}
\begin{split}
    \bbP(\Omega_s)
    &\leq
    \left(
    \frac{2}{\mathcal{M}_{\nu}(\Theta^{\top}s)}
    \right)^q
    \bbE\left|\frac{1}{M}\sum_{j=1}^{M}Z_j\right|^q \quad \text{(Markov)}\\
    &\leq
    \left(
    \frac{2}{\mathcal{M}_{\nu}(\Theta^{\top}s)}
    \right)^q
    C_1M^{-q/2}\bbE|Z_1|^q. \quad \text{(Marcinkiewicz-Zygmund)}
\end{split}
\end{equation}
Moreover,
\begin{equation}
\begin{split}
    \bbE|Z_1|^q
    &=
    \bbE\left|
    \mathcal{M}_{\nu}(\Theta^{\top}s)
    -
    \exp\left(\langle s,\Theta Y_1\rangle\right)
    \right|^q\\
    &\leq
    2^{q-1}
    \left(
    \mathcal{M}_{\nu}(\Theta^{\top}s)^q
    +
    \bbE\exp\left(q\langle s,\Theta Y_1\rangle\right)
    \right)\\
    &=
    2^{q-1}
    \left(
    \mathcal{M}_{\nu}(\Theta^{\top}s)^q
    +
    \mathcal{M}_{\nu}(q\Theta^{\top}s)
    \right).
\end{split}
\end{equation}
Combining the estimates yields
\begin{equation}
\begin{split}
    \bbP(\Omega_s)
    &\leq
    \frac{2^{2q-1}C_1}{M^{q/2}}
    \left(
    1+
    \frac{\mathcal{M}_{\nu}(q\Theta^{\top}s)}
    {\mathcal{M}_{\nu}(\Theta^{\top}s)^q}
    \right).
\end{split}
\end{equation}
The claim follows with $C=C_1$.
\end{proof}

\begin{lemma}[{Auxiliary Empirical Key Estimate}]\label{lem:stronger_Urbain_lemma}
Let $p\geq 2$, $\epsilon>0$, and
\begin{equation}
    q=\frac{p^2+2p+p\epsilon}{\epsilon}.
\end{equation}
Assume
\begin{subequations}
\begin{align}
    \bbE\left\|
    \left(Y_1-\fA(s;\nu)\right)
    \exp\left(\langle s,\Theta Y_1\rangle\right)
    \right\|_2^p
    &<\infty,\\
    \bbE\|Y_1\|_2^{p+\epsilon}
    &<\infty,\\
    \mathcal{M}_{\nu}(q\Theta^{\top}s)
    &<\infty.
\end{align}
\end{subequations}
Then there exists $C=C(p,\epsilon)>0$ such that
\begin{equation}
    \bbE\|\fA(s;\nu^M)-\fA(s;\nu)\|_2^p
    \leq
    C M^{-p/2}\mathcal G_{p,\epsilon}(\nu,s),
\end{equation}
where
\begin{equation}\label{eq:G_attention_specific}
\begin{split}
    \mathcal G_{p,\epsilon}(\nu,s)
    &=
    \frac{
    \bbE\left\|
    \left(Y_1-\fA(s;\nu)\right)
    \exp\left(\langle s,\Theta Y_1\rangle\right)
    \right\|_2^p
    }
    {\mathcal{M}_{\nu}(\Theta^{\top}s)^p}\\
    &+
    \left(
    1+
    \frac{\mathcal{M}_{\nu}(q\Theta^{\top}s)}
    {\mathcal{M}_{\nu}(\Theta^{\top}s)^q}
    \right)^{\epsilon/(p+\epsilon)}
    \left(
    \bbE\|Y_1\|_2^{p+\epsilon}
    +
    \|\fA(s;\nu)\|_2^{p+\epsilon}
    \right)^{p/(p+\epsilon)}.
\end{split}
\end{equation}

\lemusage{Used in Lemma~\ref{lem:urbain_lemma_revised}.}
\end{lemma}

\begin{proof}
First write
\begin{equation}
    \fA(s;\nu^M)-\fA(s;\nu)
    =
    \frac{
    \sum_{j=1}^{M}
    \left(Y_j-\fA(s;\nu)\right)
    \exp\left(\langle s,\Theta Y_j\rangle\right)
    }
    {
    \sum_{j=1}^{M}
    \exp\left(\langle s,\Theta Y_j\rangle\right)
    }.
\end{equation}
We split the expectation into the good event $\Omega\setminus\Omega_s$ and the bad event $\Omega_s$:
\begin{equation}
\begin{split}
    \bbE\|\fA(s;\nu^M)-\fA(s;\nu)\|_2^p
    &=
    \bbE\left[
    1_{\Omega\setminus\Omega_s}
    \left\|
    \frac{
    \sum_{j=1}^{M}
    \left(Y_j-\fA(s;\nu)\right)
    \exp\left(\langle s,\Theta Y_j\rangle\right)
    }
    {
    \sum_{j=1}^{M}
    \exp\left(\langle s,\Theta Y_j\rangle\right)
    }
    \right\|_2^p
    \right]\\
    &+
    \bbE\left[
    1_{\Omega_s}
    \left\|
    \frac{
    \sum_{j=1}^{M}
    \left(Y_j-\fA(s;\nu)\right)
    \exp\left(\langle s,\Theta Y_j\rangle\right)
    }
    {
    \sum_{j=1}^{M}
    \exp\left(\langle s,\Theta Y_j\rangle\right)
    }
    \right\|_2^p
    \right].
\end{split}
\end{equation}

\noindent\textit{First term.}
On $\Omega\setminus\Omega_s$,
\begin{equation}
    \frac{1}{M}\sum_{j=1}^{M}\exp\left(\langle s,\Theta Y_j\rangle\right)
    >
    \frac{\mathcal{M}_{\nu}(\Theta^{\top}s)}{2}.
\end{equation}
Therefore,
\begin{equation}
\begin{split}
    &\bbE\left[
    1_{\Omega\setminus\Omega_s}
    \left\|
    \frac{
    \sum_{j=1}^{M}
    \left(Y_j-\fA(s;\nu)\right)
    \exp\left(\langle s,\Theta Y_j\rangle\right)
    }
    {
    \sum_{j=1}^{M}
    \exp\left(\langle s,\Theta Y_j\rangle\right)
    }
    \right\|_2^p
    \right]\\
    &\leq
    \left(
    \frac{2}{\mathcal{M}_{\nu}(\Theta^{\top}s)}
    \right)^p
    \bbE\left\|
    \frac{1}{M}\sum_{j=1}^{M}
    \left(Y_j-\fA(s;\nu)\right)
    \exp\left(\langle s,\Theta Y_j\rangle\right)
    \right\|_2^p.
\end{split}
\end{equation}
The summands are independent and mean zero because
\begin{equation}
\begin{split}
    &\bbE\left[
    \left(Y_j-\fA(s;\nu)\right)
    \exp\left(\langle s,\Theta Y_j\rangle\right)
    \right]\\
    =&
    \int_{\bbR^{d_v}}x\exp\left(\langle s,\Theta x\rangle\right)\nu(dx)-
    \fA(s;\nu)
    \int_{\bbR^{d_v}}\exp\left(\langle s,\Theta x\rangle\right)\nu(dx)
    =
    0.
\end{split}
\end{equation}
Hence the Marcinkiewicz-Zygmund inequality gives a constant $C_1=C_1(p)>0$ such that
\begin{equation}
\begin{split}
    &\bbE\left[
    1_{\Omega\setminus\Omega_s}
    \|\fA(s;\nu^M)-\fA(s;\nu)\|_2^p
    \right]\\
    &\leq
    \left(
    \frac{2}{\mathcal{M}_{\nu}(\Theta^{\top}s)}
    \right)^p
    C_1M^{-p/2}
    \bbE\left\|
    \left(Y_1-\fA(s;\nu)\right)
    \exp\left(\langle s,\Theta Y_1\rangle\right)
    \right\|_2^p.
\end{split}
\end{equation}

\noindent\textit{Second term.}
Using {Hölder's inequality} with exponents $(p+\epsilon)/\epsilon$ and $(p+\epsilon)/p$, we obtain
\begin{equation}
\begin{split}
    &\bbE\left[
    1_{\Omega_s}
    \|\fA(s;\nu^M)-\fA(s;\nu)\|_2^p
    \right]\\
    &\leq
    \bbP(\Omega_s)^{\epsilon/(p+\epsilon)}
    \left(
    \bbE\|\fA(s;\nu^M)-\fA(s;\nu)\|_2^{p+\epsilon}
    \right)^{p/(p+\epsilon)}.
\end{split}
\end{equation}
By Lemma~\ref{lem:untruncated_concentration}, there exists $C_2=C_2(p,\epsilon)>0$ such that
\begin{equation}
    \bbP(\Omega_s)^{\epsilon/(p+\epsilon)}
    \leq
    C_2M^{-q\epsilon/(2(p+\epsilon))}
    \left(
    1+
    \frac{\mathcal{M}_{\nu}(q\Theta^{\top}s)}
    {\mathcal{M}_{\nu}(\Theta^{\top}s)^q}
    \right)^{\epsilon/(p+\epsilon)}.
\end{equation}
It remains to bound the $(p+\epsilon)$-moment. Since $\fA(s;\nu^M)$ is a convex combination of $Y_1,\ldots,Y_M$,
\begin{equation}
    \|\fA(s;\nu^M)\|_2
    \leq
    \max_{j\in[M]}\|Y_j\|_2.
\end{equation}
Thus
\begin{equation}
\begin{split}
    \bbE\|\fA(s;\nu^M)-\fA(s;\nu)\|_2^{p+\epsilon}
    &\leq
    2^{p+\epsilon-1}
    \left(
    \bbE\|\fA(s;\nu^M)\|_2^{p+\epsilon}
    +
    \|\fA(s;\nu)\|_2^{p+\epsilon}
    \right)\\
    &\leq
    2^{p+\epsilon-1}
    \left(
    \bbE\max_{j\in[M]}\|Y_j\|_2^{p+\epsilon}
    +
    \|\fA(s;\nu)\|_2^{p+\epsilon}
    \right)\\
    &\leq
    2^{p+\epsilon-1}
    \left(
    \bbE\sum_{j=1}^{M}\|Y_j\|_2^{p+\epsilon}
    +
    \|\fA(s;\nu)\|_2^{p+\epsilon}
    \right)\\
    &=
    2^{p+\epsilon-1}
    \left(
    M\bbE\|Y_1\|_2^{p+\epsilon}
    +
    \|\fA(s;\nu)\|_2^{p+\epsilon}
    \right).
\end{split}
\end{equation}
Combining the last three displays gives a constant $C_3=C_3(p,\epsilon)>0$ such that
\begin{equation}
\begin{split}
    &\bbE\left[
    1_{\Omega_s}
    \|\fA(s;\nu^M)-\fA(s;\nu)\|_2^p
    \right]\\
    &\leq
    C_3
    \left(
    1+
    \frac{\mathcal{M}_{\nu}(q\Theta^{\top}s)}
    {\mathcal{M}_{\nu}(\Theta^{\top}s)^q}
    \right)^{\epsilon/(p+\epsilon)}
    M^{-q\epsilon/(2(p+\epsilon))}
    \left(
    M\bbE\|Y_1\|_2^{p+\epsilon}
    +
    \|\fA(s;\nu)\|_2^{p+\epsilon}
    \right)^{p/(p+\epsilon)}\\
    &\leq
    C_3
    \left(
    1+
    \frac{\mathcal{M}_{\nu}(q\Theta^{\top}s)}
    {\mathcal{M}_{\nu}(\Theta^{\top}s)^q}
    \right)^{\epsilon/(p+\epsilon)}
    M^{(-q\epsilon+2p)/(2(p+\epsilon))}
    \left(
    \bbE\|Y_1\|_2^{p+\epsilon}
    +
    \|\fA(s;\nu)\|_2^{p+\epsilon}
    \right)^{p/(p+\epsilon)}.
\end{split}
\end{equation}
Note that
\begin{equation}
    \frac{-q\epsilon+2p}{2(p+\epsilon)}=-\frac{p}{2}\Longrightarrow q=\frac{p^2+2p+p\epsilon}{\epsilon}.
\end{equation}
Therefore, both the good-event and bad-event terms are bounded by $C_4M^{-p/2}\mathcal G_{p,\epsilon}(\nu,s)$, where $C_4=C_4(p,\epsilon)>0$.
\end{proof}

\begin{lemma}[MGF Ratio]\label{lem:mgf_ratio_ssg}
Let $X\sim\mu\in\mathcal{SSG}_{a_{\mu}}(\bbR^{d_v})$. Then for every $t\geq 1$ and every $s\in\bbR^{d_u}$,
\begin{equation}
    \frac{\mathcal{M}_{\mu}(t\Theta^{\top}s)}{\mathcal{M}_{\mu}(\Theta^{\top}s)^t}
    \leq \exp\left(\frac{t(t-1)a_{\mu}^2\|\Theta^\top s\|_2^2}{2}\right).
\end{equation}
Moreover, for any $m\in\bbR^{d_v}$ and $\mu=\normal(m, a_{\mu}^2I)$, the equality holds.

\lemusage{Used in Lemma~\ref{lem:bound_G} to control exponential reweighting uniformly in the query.}
\end{lemma}

\begin{proof}
Set $u=\Theta^{\top}s$ and $f(r)=\psi_{\mu}(ru)$. Since $\psi_{\mu}(0)=0$ and $\mu\in\mathcal{SSG}_{a_{\mu}}(\bbR^{d_v})$,
\begin{equation}
    f''(r)=\langle u,\nabla^2\psi_{\mu}(ru)u\rangle\leq a_{\mu}^2\|u\|_2^2.
\end{equation}
Taylor's formula with integral remainder gives
\begin{subequations}
\begin{align}
    f(t)&=tf'(0)+\int_0^t(t-r)f''(r)dr,\\
    f(1)&=f'(0)+\int_0^1(1-r)f''(r)dr.
\end{align}
\end{subequations}
Therefore,
\begin{equation}
\begin{aligned}
    f(t)-tf(1)
    &=(t-1)\int_0^1rf''(r)dr+\int_1^t(t-r)f''(r)dr\\
    &\leq (t-1)a_{\mu}^2\|u\|_2^2\int_0^1rdr+a_{\mu}^2\|u\|_2^2\int_1^t(t-r)dr\\
    &=\frac{t(t-1)a_{\mu}^2\|u\|_2^2}{2}.
\end{aligned}
\end{equation}
Exponentiating proves the claim.
\end{proof}

\begin{lemma}[Bound on $\mathcal G$]\label{lem:bound_G}
Let $X\sim\mu\in\mathcal{SSG}_{a_{\mu}}(\bbR^{d_v})$, let $p\geq 2$ and $\epsilon>0$, and define
\begin{equation}
    c(p,\epsilon)=\frac{p}{2(p+\epsilon-1)}+\frac{p(p+\epsilon+2)}{2(p+\epsilon)}\left(\frac{p(p+\epsilon+2)}{\epsilon}-1\right).
\end{equation}
Then there exists $C=C(p,\epsilon)>0$ such that for every $s\in\bbR^{d_u}$,
\begin{equation}
    \mathcal G_{p,\epsilon}(\mu,s)\leq C\exp\left(c(p,\epsilon)a_{\mu}^2\|\Theta^\top s\|_2^2\right)\left(\bbE\|X\|_2^{p+\epsilon}\right)^{p/(p+\epsilon)}.
\end{equation}

\lemusage{Used in Lemma~\ref{lem:urbain_lemma_revised}.}
\end{lemma}

\begin{proof}
Set $u=\Theta^{\top}s$ and
\begin{equation}
\begin{aligned}
    r&=p+\epsilon, & \hat p&=\frac{pr}{\epsilon}, & q&=\frac{p(r+2)}{\epsilon}.
\end{aligned}
\end{equation}
$c(p,\epsilon)=\frac{p}{2(r-1)}+\frac{\epsilon q(q-1)}{2r}$.
The moment $\bbE\|X\|_2^r$ is finite by Lemma~\ref{lem:ssg_consequences}. Constants $C=C(p,\epsilon)>0$ may change from line to line.

By Hölder's inequality and Lemma~\ref{lem:mgf_ratio_ssg},
\begin{equation}
\begin{aligned}
    \|\fA(s;\mu)\|_2
    &\leq \frac{\bbE\left[\|X\|_2\exp\left(\langle u,X\rangle\right)\right]}{\mathcal{M}_{\mu}(u)}\\
    &\leq \left(\bbE\|X\|_2^r\right)^{1/r}\left(\frac{\mathcal{M}_{\mu}\left(\frac{r}{r-1}u\right)}{\mathcal{M}_{\mu}(u)^{r/(r-1)}}\right)^{(r-1)/r}\\
    &\leq \left(\bbE\|X\|_2^r\right)^{1/r}\exp\left(\frac{a_{\mu}^2\|u\|_2^2}{2(r-1)}\right).
\end{aligned}
\end{equation}
Consequently,
\begin{equation}\label{eq:bound_attention_by_moment}
    \|\fA(s;\mu)\|_2^p\leq \left(\bbE\|X\|_2^r\right)^{p/r}\exp\left(\frac{pa_{\mu}^2\|u\|_2^2}{2(r-1)}\right).
\end{equation}

Write
\begin{equation}
    \mathcal G_{p,\epsilon}(\mu,s)=G_1+G_2,
\end{equation}
where
\begin{subequations}
\begin{align}
    G_1&=\frac{\bbE\left[\|X-\fA(s;\mu)\|_2^p\exp\left(p\langle u,X\rangle\right)\right]}{\mathcal{M}_{\mu}(u)^p},\\
    G_2&=\left(1+\frac{\mathcal{M}_{\mu}(qu)}{\mathcal{M}_{\mu}(u)^q}\right)^{\epsilon/r}\left(\bbE\|X\|_2^r+\|\fA(s;\mu)\|_2^r\right)^{p/r}.
\end{align}
\end{subequations}

For $G_1$, Hölder's inequality with exponents $r/p$ and $r/\epsilon$ gives
\begin{equation}
\begin{aligned}
    G_1
    &\leq \left(\bbE\|X-\fA(s;\mu)\|_2^r\right)^{p/r}\frac{\mathcal{M}_{\mu}(\hat pu)^{\epsilon/r}}{\mathcal{M}_{\mu}(u)^p}\\
    &=\left(\bbE\|X-\fA(s;\mu)\|_2^r\right)^{p/r}\left(\frac{\mathcal{M}_{\mu}(\hat pu)}{\mathcal{M}_{\mu}(u)^{\hat p}}\right)^{\epsilon/r},
\end{aligned}
\end{equation}
where $\epsilon\hat p/r=p$. By \eqref{eq:bound_attention_by_moment},
\begin{equation}
\begin{aligned}
    \left(\bbE\|X-\fA(s;\mu)\|_2^r\right)^{p/r}
    &\leq C\left(\left(\bbE\|X\|_2^r\right)^{p/r}+\|\fA(s;\mu)\|_2^p\right)\\
    &\leq C\left(\bbE\|X\|_2^r\right)^{p/r}\exp\left(\frac{pa_{\mu}^2\|u\|_2^2}{2(r-1)}\right).
\end{aligned}
\end{equation}
Moreover, Lemma~\ref{lem:mgf_ratio_ssg} yields
\begin{equation}
    \left(\frac{\mathcal{M}_{\mu}(\hat pu)}{\mathcal{M}_{\mu}(u)^{\hat p}}\right)^{\epsilon/r}\leq \exp\left(\frac{\epsilon\hat p(\hat p-1)a_{\mu}^2\|u\|_2^2}{2r}\right).
\end{equation}
Since $q\geq\hat p>1$,
\begin{equation}
    \frac{p}{2(r-1)}+\frac{\epsilon\hat p(\hat p-1)}{2r}\leq\frac{p}{2(r-1)}+\frac{\epsilon q(q-1)}{2r}=c(p,\epsilon).
\end{equation}
Therefore,
\begin{equation}
    G_1\leq C\left(\bbE\|X\|_2^r\right)^{p/r}\exp\left(c(p,\epsilon)a_{\mu}^2\|u\|_2^2\right).
\end{equation}

For $G_2$, Lemma~\ref{lem:mgf_ratio_ssg} gives
\begin{equation}
\begin{aligned}
    \left(1+\frac{\mathcal{M}_{\mu}(qu)}{\mathcal{M}_{\mu}(u)^q}\right)^{\epsilon/r}
    &\leq \left(1+\exp\left(\frac{q(q-1)a_{\mu}^2\|u\|_2^2}{2}\right)\right)^{\epsilon/r}\\
    &\leq 2^{\epsilon/r}\exp\left(\frac{\epsilon q(q-1)a_{\mu}^2\|u\|_2^2}{2r}\right).
\end{aligned}
\end{equation}
Using \eqref{eq:bound_attention_by_moment} and $p/r<1$,
\begin{equation}
\begin{aligned}
    \left(\bbE\|X\|_2^r+\|\fA(s;\mu)\|_2^r\right)^{p/r}
    &\leq \left(\bbE\|X\|_2^r\right)^{p/r}+\|\fA(s;\mu)\|_2^p\\
    &\leq 2\left(\bbE\|X\|_2^r\right)^{p/r}\exp\left(\frac{pa_{\mu}^2\|u\|_2^2}{2(r-1)}\right).
\end{aligned}
\end{equation}
By the definition of $c(p,\epsilon)$,
\begin{equation}
    \frac{p}{2(r-1)}+\frac{\epsilon q(q-1)}{2r}=c(p,\epsilon).
\end{equation}
Thus,
\begin{equation}
    G_2\leq C\left(\bbE\|X\|_2^r\right)^{p/r}\exp\left(c(p,\epsilon)a_{\mu}^2\|u\|_2^2\right).
\end{equation}
Combining the estimates for $G_1$ and $G_2$ proves the claim.
\end{proof}

\begin{lemma}[Empirical Key Bound]\label{lem:urbain_lemma_revised}
Let $Y_1,\ldots,Y_M$ be i.i.d. samples from $\nu\in\mathcal{SSG}_{a_{\nu}}(\bbR^{d_v})$, and let $\nu^M$ be the empirical key/value measure. Then for every $s\in\bbR^{d_u}$ and every $p\geq 2$, there exists $C=C(p,\nu)>0$ such that
\begin{equation}
    \bbE\|\fA(s;\nu)-\fA(s;\nu^M)\|_2^p
    \leq CM^{-p/2}\exp\left(c(p)a_{\nu}^2\|\Theta^\top s\|_2^2\right),
\end{equation}
where $c(p)$ is defined in \eqref{eq:empirical_key_exponent}.

\lemusage{Used in Theorem~\ref{thm:cross_noncompact} and Proposition~\ref{prop:second_term}.}
\end{lemma}

\begin{proof}
Take $\epsilon=p(p+2)$. Then

\begin{equation}
    p+\epsilon=p(p+3),\qquad
    q=\frac{p(p+\epsilon+2)}{\epsilon}=p+1.
\end{equation}

Since $\nu\in\mathcal{SSG}_{a_{\nu}}(\bbR^{d_v})$, all polynomial and linear exponential moments of $\nu$ are finite. In particular, the assumptions of Lemma~\ref{lem:stronger_Urbain_lemma} are satisfied. Hence there exists $C_1=C_1(p)>0$ such that
\begin{equation}
    \bbE\|\fA(s;\nu)-\fA(s;\nu^M)\|_2^p\leq C_1M^{-p/2}\mathcal G_{p,\epsilon}(\nu,s).
\end{equation}
By Lemma~\ref{lem:bound_G}, there exists $C_2=C_2(p)>0$ such that
\begin{equation}
    \mathcal G_{p,\epsilon}(\nu,s)
    \leq C_2\exp\left(c(p)a_{\nu}^2\|\Theta^\top s\|_2^2\right)
    \left(\bbE\|Y_1\|_2^{p(p+3)}\right)^{1/(p+3)},
\end{equation}
where
\begin{equation}
    \begin{aligned}
        c(p)
        &=\frac{p}{2(p+\epsilon-1)}+\frac{\epsilon q(q-1)}{2(p+\epsilon)}\\
        &=\frac{p}{2(p(p+3)-1)}+\frac{p(p+1)(p+2)}{2(p+3)} = \frac{p^2}{2} + O(1).
    \end{aligned}
\end{equation}
Since $\bbE\|Y_1\|_2^{p(p+3)}<\infty$, its contribution can be absorbed into a constant $C=C(p,\nu)$. Combining the two estimates proves the claim.
\end{proof}

\subsection{Noncompact Self-Attention Estimates}
\label{app:noncompact_self_estimates}

The remaining noncompact estimates handle the feature that is unique to self-attention: each query particle also belongs to its own empirical key/value set. The leave-one-out step below separates the independent empirical key error from the diagonal self-interaction term.

\begin{lemma}[Diagonal Term]\label{lem:remove_XX_term}
Let $X_1,\ldots,X_N\sim\mu$ be i.i.d., let $p\geq 1$, and suppose that $\mu$ satisfies Assumption~\ref{ass:self_noncompact_input}. Define the diagonal weight
\begin{equation}
    D_N=\frac{\exp\left(\langle X_N,\Theta X_N\rangle\right)}{\sum_{n=1}^{N}\exp\left(\langle X_N,\Theta X_n\rangle\right)}.
\end{equation}
Then there exists $C=C(p,\Theta,\mu)>0$ such that for every $N\geq 2$,
\begin{equation}
    \bbE D_N^{2p}\leq CN^{-2p}.
\end{equation}

\lemusage{Used in Proposition~\ref{prop:second_term}. It controls the diagonal term caused by the query point also appearing in its own empirical key set.}
\end{lemma}

\begin{proof}
The Cauchy--Schwarz inequality gives
\begin{equation}
    D_N\leq\frac{1}{N^2}\sum_{i=1}^{N}\exp\left(\langle X_N,\Theta(X_N-X_i)\rangle\right).
\end{equation}
Since $\left(\sum_{i=1}^{N}y_i\right)^{2p}\leq N^{2p-1}\sum_{i=1}^{N}y_i^{2p}$ for $y_i\geq0$,
\begin{equation}
    \bbE D_N^{2p}\leq\frac{1}{N^{2p+1}}\sum_{i=1}^{N}\bbE\exp\left(2p\langle X_N,\Theta(X_N-X_i)\rangle\right).
\end{equation}
The term corresponding to $i=N$ equals one. For $i\neq N$, conditioning on $X_N$ and using independence gives
\begin{equation}
\begin{aligned}
    \bbE\exp\left(2p\langle X_N,\Theta(X_N-X_i)\rangle\right)
    &=\bbE\left[\exp\left(2p\langle X_N,\Theta X_N\rangle\right)\mathcal{M}_{\mu}\left(-2p\Theta^{\top}X_N\right)\right].
\end{aligned}
\end{equation}
The tail condition in Assumption~\ref{ass:self_noncompact_input}, together with the corresponding smallness condition in the $2$-exponential case, implies that
\begin{equation}
    \bbE\left[\exp\left(2p\langle X,\Theta X\rangle\right)\mathcal{M}_{\mu}\left(-2p\Theta^{\top}X\right)\right]<\infty.
\end{equation}
Consequently, there exists $C_1=C_1(p,\Theta,\mu)>0$ such that
\begin{equation}
    \bbE D_N^{2p}\leq\frac{C_1}{N^{2p}}.
\end{equation}
\end{proof}

\begin{proposition}[Noncompact Self Second Term]\label{prop:second_term}
Fix $p\geq 1$ and let $\mu$ satisfy Assumption~\ref{ass:self_noncompact_input}. Let
\begin{equation}
    \mu^N=\frac{1}{N}\sum_{j=1}^{N}\delta_{X_j}
\end{equation}
be the empirical measure of i.i.d. samples $X_1,\ldots,X_N$ from $\mu$. Then there exists $C=C(p,\Theta,\mu)>0$ such that, for every $N\geq 2$,

\begin{equation}\label{eq:self_key_particle_bound}
\begin{aligned}
    &\left(\bbE W_p^p\left(\fA(\cdot;\mu)_{\sharp}\mu^N,\fA(\cdot;\mu^N)_{\sharp}\mu^N\right)\right)^{1/p}\\
    &\leq\left(\frac{1}{N}\sum_{j=1}^{N}\bbE\|\fA(X_j;\mu)-\fA(X_j;\mu^N)\|_2^p\right)^{1/p}\\
    &\leq CN^{-1/2}.
\end{aligned}
\end{equation}
In particular, the same rate holds for the expected Wasserstein distance. The middle expression also controls residual maps: adding $X_j$ to both images leaves their difference unchanged, and applying a value matrix or a Lipschitz post-attention map multiplies the bound by its Lipschitz constant.

\end{proposition}

\begin{proof}
For each $j\in[N]$, define the leave-one-out empirical measure
\begin{equation}
    \mu_{-j}^{N-1}=\frac{1}{N-1}\sum_{\substack{k=1\\k\neq j}}^{N}\delta_{X_k}.
\end{equation}
The empirical coupling that pairs the two images of each $X_j$ gives
\begin{equation}
    W_p\left(\fA(\cdot;\mu)_{\sharp}\mu^N,\fA(\cdot;\mu^N)_{\sharp}\mu^N\right)^p
    \leq\frac{1}{N}\sum_{j=1}^{N}\|\fA(X_j;\mu)-\fA(X_j;\mu^N)\|_2^p.
\end{equation}
By exchangeability, it suffices to estimate the error at $X_N$. Define
\begin{equation}
\begin{aligned}
    A_N&=\frac{1}{N}\sum_{j=1}^{N}\bbE\|\fA(X_j;\mu)-\fA(X_j;\mu_{-j}^{N-1})\|_2^p\\
    &=\bbE\|\fA(X_N;\mu)-\fA(X_N;\mu_{-N}^{N-1})\|_2^p.
\end{aligned}
\end{equation}

We first estimate $A_N$. Set $r_p=\max\{p,2\}$. Since $X_j$ is independent of $\mu_{-j}^{N-1}$ and $\mu_{-j}^{N-1}$ is the empirical measure of $N-1$ i.i.d. samples from $\mu$, Lemma~\ref{lem:urbain_lemma_revised}, applied conditionally on $X_j$, gives
\begin{equation}
    \bbE\left[\|\fA(X_j;\mu)-\fA(X_j;\mu_{-j}^{N-1})\|_2^{r_p}\mid X_j\right]
    \leq C_0(N-1)^{-r_p/2}\exp\left(c(r_p)a_{\mu}^2\|\Theta^{\top}X_j\|_2^2\right),
\end{equation}
where $C_0=C_0(r_p,\mu)>0$. Since $p\leq r_p$, conditional Lyapunov's inequality yields
\begin{equation}
\begin{aligned}
    &\bbE\left[\|\fA(X_j;\mu)-\fA(X_j;\mu_{-j}^{N-1})\|_2^p\mid X_j\right]\\
    &\leq\left(\bbE\left[\|\fA(X_j;\mu)-\fA(X_j;\mu_{-j}^{N-1})\|_2^{r_p}\mid X_j\right]\right)^{p/r_p}\\
    &\leq C_0^{p/r_p}(N-1)^{-p/2}
    \exp\left(\frac{p}{r_p}c(r_p)a_{\mu}^2\|\Theta^{\top}X_j\|_2^2\right).
\end{aligned}
\end{equation}
Taking expectations and using exchangeability, we obtain
\begin{equation}
    A_N
    \leq C_0^{p/r_p}(N-1)^{-p/2}\mathcal I_{p,\Theta}(\mu,\mu)
    \leq C_1(N-1)^{-p/2},
\end{equation}
where $C_1=C_1(p,\Theta,\mu)>0$. The finiteness of $\mathcal I_{p,\Theta}(\mu,\mu)$ follows from \eqref{eq:self_leave_one_out_integrability}.

Recall the diagonal weight from Lemma~\ref{lem:remove_XX_term},
\begin{equation}
    D_N=\frac{\exp\left(\langle X_N,\Theta X_N\rangle\right)}{\sum_{n=1}^{N}\exp\left(\langle X_N,\Theta X_n\rangle\right)}\in(0,1].
\end{equation}
The definition of attention gives the identity
\begin{equation}
\begin{aligned}
    \fA(X_N;\mu^N)-\fA(X_N;\mu)
    &=(1-D_N)\left(\fA(X_N;\mu_{-N}^{N-1})-\fA(X_N;\mu)\right)\\
    &\quad+D_N\left(X_N-\fA(X_N;\mu)\right).
\end{aligned}
\end{equation}
By Proposition~\ref{prop:query_lipschitz} and $\fA(0;\mu)=\bbE X_1$,
\begin{equation}
    \|X_N-\fA(X_N;\mu)\|_2
    \leq\left(1+a_{\mu}^2\|\Theta\|_2\right)\|X_N\|_2+\|\bbE X_1\|_2.
\end{equation}
All polynomial moments of $\mu$ are finite by Lemma~\ref{lem:ssg_consequences}. Hence Cauchy--Schwarz and Lemma~\ref{lem:remove_XX_term} give
\begin{equation}
\begin{aligned}
    \bbE\left[D_N^p\|X_N-\fA(X_N;\mu)\|_2^p\right]
    &\leq\left(\bbE D_N^{2p}\right)^{1/2}
    \left(\bbE\|X_N-\fA(X_N;\mu)\|_2^{2p}\right)^{1/2}\\
    &\leq C_2N^{-p},
\end{aligned}
\end{equation}
where $C_2=C_2(p,\Theta,\mu)>0$.

Since $0\leq1-D_N\leq1$, Minkowski's inequality gives
\begin{equation}\label{eq:first_bound_EWp_prop2_revised}
\begin{aligned}
    &\left(\bbE W_p^p\left(\fA(\cdot;\mu)_{\sharp}\mu^N,\fA(\cdot;\mu^N)_{\sharp}\mu^N\right)\right)^{1/p}\\
    &\leq\left(\frac{1}{N}\sum_{j=1}^{N}\bbE\|\fA(X_j;\mu)-\fA(X_j;\mu^N)\|_2^p\right)^{1/p}\\
    &\leq A_N^{1/p}+\frac{C_2^{1/p}}{N}\\
    &\leq C_1^{1/p}(N-1)^{-1/2}+\frac{C_2^{1/p}}{N}\\
    &\leq CN^{-1/2},
\end{aligned}
\end{equation}

where we used $(N-1)^{-1/2}\leq\sqrt{2}N^{-1/2}$ for $N\geq2$, and one may take
\begin{equation}
    C=\sqrt{2}C_1^{1/p}+C_2^{1/p}.
\end{equation}
\end{proof}

\section{Experimental Details and Additional Results}
\label{app:experimental_details}

This appendix specifies the experiments in Section~\ref{sec:experiments}.
Subsections~\ref{app:experimental_generators} and
\ref{app:experimental_training} describe the random input and target laws,
network architectures, and training procedure.
Subsections~\ref{app:experimental_metrics} and
\ref{app:experimental_rate_fits} define the numerical distances, sampling
protocol, and rate estimators. Subsection~\ref{app:experimental_theory_scope}
explains how the chosen input families relate to the theorem assumptions,
and Subsections~\ref{app:experimental_sinkhorn_ablation} and
\ref{app:experimental_sampling_repeats} compare Sinkhorn settings and
sampling counts.

\subsection{Input Families and Target Measures}
\label{app:experimental_generators}

The families below are used for both training and testing the operators in
Section~\ref{sec:experiments}, with identical parameter distributions and
ranges. At every training step, each ensemble in the batch receives newly
drawn input-measure parameters. For testing, three new input measures per
model are drawn with seeds $s_\mu\in\{0,1,2\}$ and then held fixed across
particle counts and sampling replicates. Thus testing uses new laws from
the training family, as well as new particles. The true output law is
determined analytically for each chosen input law. Input and target
particles are sampled independently conditional on these law parameters;
the target law is not estimated from the input particles.

\paragraph{Torus doubling.}
We identify $\mathbb T^d$ with $[0,1)^d$ and draw product truncated Gaussian
densities
\begin{equation}
    \rho_{m,s}(x)=\prod_{j=1}^d
    \frac{\exp(-(x_j-m_j)^2/(2s_j^2))}
    {\sqrt{2\pi}s_j[\Phi((1-m_j)/s_j)-\Phi(-m_j/s_j)]},
    \qquad x\in[0,1)^d,
\end{equation}
where $\Phi$ is the standard normal distribution function.
The coordinates $m_j$ are uniform on $[0.2,0.8]$, and $\log s_j$ are uniform
on $[\log 0.04,\log 0.15]$, independently. The target is obtained by applying
$x\mapsto 2x\pmod 1$ to an independent sample from this density.

\paragraph{Gaussian and Gaussian-mixture operators.}
For $\normal(m,\Sigma)$, the mean coordinates are uniform on $[-1,1]$.
We write $\Sigma=O\operatorname{diag}(\lambda_1,\ldots,\lambda_d)O^\top$,
where the $\log\lambda_j$ are uniform on $[\log 0.5,\log 2]$ and $O$ is a
random orthogonal basis obtained by QR factorization of an independent
matrix of standard Gaussian entries (with $O=1$ in dimension one).
The rescaling and covariance-inversion targets are
$\normal(m/2,\Sigma/2)$ and $\normal(m,\Sigma^{-1})$, respectively.
For mixtures, the number of components is uniform on $\{1,2,3\}$, their
weights have a $\operatorname{Dirichlet}(1,\ldots,1)$ law, and component
parameters are drawn independently as above. The two operators act on each
component and preserve its weight.

\paragraph{Tilting.}
For each dimension, one base source--target pair defines the fixed
tilting potential. The actual input and output laws vary through the
common perturbation described below. In the Gaussian case, let
\begin{equation}
    V_{m,\Sigma}(x)=\tfrac12(x-m)^\top\Sigma^{-1}(x-m),\qquad
    u(x)=V_{m_1,\Sigma_1}(x)-V_{m_0,\Sigma_0}(x).
\end{equation}
Then the tilting operator in \eqref{eq:experimental_tilting} sends
$\normal(m_0,\Sigma_0)$ to $\normal(m_1,\Sigma_1)$, with normalizing constant
$Z=(\det\Sigma_1/\det\Sigma_0)^{1/2}$. Both parameter pairs are drawn from
the Gaussian family above.

The faster-tailed example uses the cubic generalized-normal density
\begin{equation}\label{eq:experimental_cubic_density}
    \rho_{m,C}(x)=\frac{1}{|\det C|}
       \left(\frac{3}{2\Gamma(1/3)}\right)^d
       \exp\bigl(-\|C^{-1}(x-m)\|_3^3\bigr).
\end{equation}
Here $m_j$ is uniform on $[-1,1]$ and
$C=O\operatorname{diag}(s_1,\ldots,s_d)$, with independent
$\log s_j$ uniform on $[\log 0.5,\log 2]$ and an independent random
orthogonal basis $O$. With
$V_{m,C}(x)=\|C^{-1}(x-m)\|_3^3$ and
$u=V_{m_1,C_1}-V_{m_0,C_0}$, the tilted source is exactly
$\rho_{m_1,C_1}$ and $Z=|\det C_1|/|\det C_0|$.
Sampling is direct: $X=m+C\xi$, where the coordinates of $\xi$ have independent
uniform signs and $|\xi_j|^3\sim\operatorname{Gamma}(1/3,1)$.

\paragraph{Seeds defining the tilting operators.}
For $d\in\{1,2,3,5,10\}$, the Gaussian base pair is generated with seed
$s_u=8001+100d$ and the cubic base pair with seed $s_u=7001+100d$.
Using the respective parameter generator above, we initialize a separate
PyTorch CPU random stream with $s_u$ and draw a batch of two parameter
sets: the first defines $V_s$ and the second defines $V_t$.
Within this draw, means are generated first, then log eigenvalues or
log scales, followed for $d>1$ by the Gaussian matrices
used for QR factorization. The resulting $u=V_t-V_s$ is fixed throughout
training and testing. These operator seeds are separate from the training,
test-measure, and particle-sampling seeds. The saved configurations retain
the generated base parameters at full precision.

\paragraph{Random input laws for a fixed tilt.}
\label{app:experimental_random_tilting}
To train an operator over a family of input measures while retaining the
same $u=V_t-V_s$, introduce a common random potential
\[
    h_{a,b}(x)=\tfrac a2\|x\|_2^2+b^\top x,\qquad
    a\sim\operatorname{Unif}[0,0.5],\quad
    b_j\sim\operatorname{Unif}[-0.5,0.5],
\]
with independent coefficients. The input and target densities are
\[
    \rho_{s,a,b}(x)\propto e^{-V_s(x)-h_{a,b}(x)},\qquad
    \rho_{t,a,b}(x)\propto e^{-V_t(x)-h_{a,b}(x)}.
\]
Their potential difference is always $u$. Each ensemble in a training
batch receives fresh $(a,b)$, then independent source and target particles;
test laws use separately drawn coefficients. For a Gaussian base,
completing the square gives covariance $(\Sigma^{-1}+aI)^{-1}$ and mean
$(\Sigma^{-1}+aI)^{-1}(\Sigma^{-1}m-b)$, so both vary across input laws.
For the cubic base, $C$ has orthogonal columns. In coordinates $x=m+Cz$,
the density factorizes into terms proportional to
\[
    \exp(-|z_j|^3-A_jz_j^2/2-B_jz_j),\quad
    A_j=a(C^\top C)_{jj},\quad B=C^\top(am+b).
\]
Independent rejection sampling uses the global bound
$|z|^3\geq cz^2-4c^3/27$, $c>0$, with a Gaussian proposal of precision
$A_j+2c$ and mean $-B_j/(A_j+2c)$. This supplies {i.i.d.} samples from the
specified law. All ten tilting models are trained on this random-input family.

\subsection{Architecture and Training}
\label{app:experimental_training}

All models consist of an affine input map, a stack of post-norm self-attention
blocks, and a pointwise output map. The blocks use {multihead} attention,
residual connections, GELU feed-forward layers of width four times the model
width, and {layer normalization with $\epsilon=10^{-5}$}. We use no positional
encoding or dropout. Euclidean inputs and outputs are represented directly
in $\bbR^d$. On the torus, the input representation is
\begin{equation}\label{eq:experimental_torus_embedding}
    \phi(x)=(\sin(2\pi x_1),\ldots,\sin(2\pi x_d),
             \cos(2\pi x_1),\ldots,\cos(2\pi x_d))\in\bbR^{2d}.
\end{equation}
The output layer produces a pair of coordinates for each circle, which is
converted to an angle in $[0,1)$. Table~\ref{tab:experimental_architecture}
gives the network sizes and training budgets for $d=1,2,3$;
Table~\ref{tab:experimental_h100_architecture} gives those for $d=5,10$.
Each operator and dimension has its own trained model, whose weights are then fixed across all $N$.

\begin{table}[htbp]
\centering
\caption{Architecture and training parameters for $d=1,2,3$. The Euclidean rows apply to
all six Euclidean operators. Batch size counts independent particle ensembles.}
\label{tab:experimental_architecture}
\begin{tabular}{lrrrrrr}
\hline
Family & $d$ & Width & Blocks & Heads & Steps & Batch size\\
\hline
Torus & 1 & 256 & 6 & 8 & 10,000 & 8\\
Torus & 2, 3 & 384 & 8 & 8 & 20,000 & 4\\
Euclidean & 1 & 192 & 6 & 6 & 15,000 & 8\\
Euclidean & 2 & 256 & 8 & 8 & 25,000 & 6\\
Euclidean & 3 & 384 & 10 & 12 & 35,000 & 4\\
\hline
\end{tabular}
\end{table}

\begin{table}[htbp]
\centering
\caption{Architecture and training parameters for the H100 experiments.
Every paired entry lists $d=5$ / $d=10$. All models use eight attention
heads, batch size four, and feed-forward width four times the model width.}
\label{tab:experimental_h100_architecture}
\small
\setlength{\tabcolsep}{4pt}
\begin{tabular}{lrrrr}
\toprule
Operator & Width & Blocks & Steps & Peak learning rate\\
\midrule
Torus doubling & 384 / 512 & 8 / 8 & 30,000 / 50,000 & $2\!\times\!10^{-4}$ / $1.5\!\times\!10^{-4}$\\
Gaussian rescaling & 512 / 512 & 10 / 12 & 45,000 / 65,000 & $1.5\!\times\!10^{-4}$ / $10^{-4}$\\
Gaussian inversion & 512 / 512 & 10 / 12 & 50,000 / 80,000 & $10^{-4}$ / $7.5\!\times\!10^{-5}$\\
Mixture rescaling & 512 / 512 & 10 / 12 & 55,000 / 85,000 & $10^{-4}$ / $7.5\!\times\!10^{-5}$\\
Mixture inversion & 512 / 512 & 10 / 12 & 70,000 / 110,000 & $10^{-4}$ / $5\!\times\!10^{-5}$\\
Gaussian tilting & 512 / 512 & 10 / 12 & 50,000 / 80,000 & $5\!\times\!10^{-5}$ / $3\!\times\!10^{-5}$\\
Cubic tilting & 512 / 512 & 10 / 12 & 50,000 / 80,000 & $5\!\times\!10^{-5}$ / $3\!\times\!10^{-5}$\\
\bottomrule
\end{tabular}
\end{table}

At every optimization step, we draw $B$ independent input laws
$\mu_1,\ldots,\mu_B$ from the appropriate family in
Appendix~\ref{app:experimental_generators}, where $B$ is the batch size in
Tables~\ref{tab:experimental_architecture} and
\ref{tab:experimental_h100_architecture}. For each law, we draw $n$ input
particles and $n$ independent target particles from its analytically
specified image under $\sG^\dagger$. The transformer processes each input
ensemble separately. For predicted particles $z_1,\ldots,z_n$ and target
particles $y_1,\ldots,y_n$, the per-ensemble loss is the empirical energy
distance
\citep{szekely2013energy}
\begin{equation}\label{eq:experimental_training_loss}
    \widehat{\mathcal E}(z,y)
    =\frac{2}{n^2}\sum_{i,j}\|z_i-y_j\|_2
     -\frac{1}{n^2}\sum_{i,j}\|z_i-z_j\|_2
     -\frac{1}{n^2}\sum_{i,j}\|y_i-y_j\|_2.
\end{equation}
The optimization loss is $B^{-1}\sum_{b=1}^B\widehat{\mathcal E}(z^{(b)},y^{(b)})$.
For torus data, $z_i$ and $y_j$ in this formula are replaced by their
embeddings under $\phi$. A common particle count for the batch is sampled
uniformly from
$\{128,256\}$ during the first $60\%$ of training, from
$\{128,256,512\}$ during the next $20\%$, and from $\{512,1024\}$ during
the final $20\%$ for $d\leq5$. For $d=10$, the same three particle-count
sets occupy the first $30\%$, next $30\%$, and final $40\%$ of training.
Thus the largest evaluation ensembles are substantially
larger than those used for training.

We use AdamW with weight decay $10^{-4}$. For $d=1,2,3$, the peak learning rate is
$3\times10^{-4}$ for torus doubling and $2\times10^{-4}$ for the Euclidean
problems, except that both three-dimensional tilting models use
$5\times10^{-5}$. A linear warm-up occupies $5\%$ of training, followed by
cosine decay to one tenth of the peak rate; gradients are clipped at norm
$1$. The cubic tilt in $d\in\{3,5,10\}$ instead uses $10\%$ warm-up and
clipping threshold $0.25$. The higher-dimensional peak rates are listed in
Table~\ref{tab:experimental_h100_architecture}; the other optimizer settings
are unchanged. Training uses single-precision parameters and
particles with native attention and no automatic mixed precision.

All 35 checkpoints use training seed $7$. This seed controls parameter
initialization, training measure and particle draws, and curriculum draws.
The final checkpoint at the stated training budget is retained per
operator and dimension; training randomness is not replicated in the
reported uncertainty. The Gaussian-tilting models in $d\in\{5,10\}$ use query--key
normalization: within each head, $q$ and $k$ are divided by
$\max\{\|q\|_2,10^{-6}\}$ and $\max\{\|k\|_2,10^{-6}\}$, respectively,
and the attention score is $g_h\langle\widehat q,\widehat k\rangle$.
The learned head gain is $g_h=8\operatorname{sigmoid}(a_h)$, initialized
at $g_h=4$. The remaining models use ordinary scaled dot-product attention.
Evaluation uses full attention over each ensemble,
including the $500{,}000$-particle reference, in single precision without
autocast or TF32; particles are not split into separate attention groups.

\paragraph{Hardware, software, and measured runtimes.}
The $d=1,2,3$ models were trained and evaluated on an NVIDIA GeForce
RTX 4080 SUPER. Their saved evaluation records specify Python 3.12.2,
PyTorch 2.4.0, CUDA 12.4, GeomLoss 0.3.1, and PyKeOps 2.3.
All fourteen $d=5,10$ models were trained and evaluated on NVIDIA H100
80GB HBM3 GPUs, with one GPU per model task. Their saved runtime records
specify Python 3.12.2, PyTorch 2.5.1, CUDA 12.1, GeomLoss 0.3.1, and
PyKeOps 2.3. These runs retain FP32 parameters, particles, and native
attention, with autocast and TF32 disabled.

Tables~\ref{tab:experimental_rtx_runtime} and
\ref{tab:experimental_h100_runtime} report RTX 4080 SUPER and H100
training and model-forward times. Training time is the elapsed time from
entry into the optimization loop to its final progress record; it includes
data generation, logging, and any intermediate validation and
recovery-checkpoint writes, but excludes model construction and final
model export. All fourteen H100 training attempts started without resuming.
For the fifteen older RTX models without individual training timers, we
measured independent full-budget training reruns on the same GPU type,
using the original architecture, hyperparameters, and seed $7$. These
reruns measure the training cost of the configurations; their separately
saved checkpoints do not replace the originals used for any numerical
result or forward timing. The reruns used Python 3.12.2,
PyTorch 2.4.0, and CUDA 12.4. The six RTX tilting models retain
their original optimization-loop timers. Each training duration represents
one run.

Forward times are the saved model-forward durations with CUDA
synchronization at completion. We report the median of fifteen calls at
$N=65{,}536$ (five realizations for each of the three input laws) and of
three reference calls at $M=500{,}000$ per model. These times exclude
the subsequent distance calculations. These measurements describe
the stated runs; the different model sizes and training budgets do not
constitute a comparison between GPU types.

\begin{table}[t]
\centering
\caption{RTX 4080 SUPER training and inference times for $d=1,2,3$.
Training times are minutes; median forward times are seconds. For the first
five operators, training times are measured in independent full-budget
reruns with the original configurations and seed $7$; the two tilting
operators retain their original training timers. All inference times use
the original evaluated checkpoints. References are computed once per input
law and reused across all sampling replicates.}
\label{tab:experimental_rtx_runtime}
\small
\setlength{\tabcolsep}{3pt}
\begin{tabular}{lrrrrrrrrr}
\toprule
& \multicolumn{3}{c}{Training (min)}
& \multicolumn{3}{c}{$N=65{,}536$ (s)}
& \multicolumn{3}{c}{$M=500{,}000$ (s)}\\
\cmidrule(lr){2-4}\cmidrule(lr){5-7}\cmidrule(lr){8-10}
Operator & $d=1$ & $d=2$ & $d=3$ & $d=1$ & $d=2$ & $d=3$ & $d=1$ & $d=2$ & $d=3$\\
\midrule
Torus doubling & 2.6 & 6.0 & 6.0 & 2.95 & 4.89 & 4.89 & 170.1 & 283.9 & 283.3\\
Gaussian rescaling & 2.9 & 6.3 & 12.7 & 2.55 & 3.92 & 7.36 & 127.6 & 227.1 & 426.0\\
Gaussian inversion & 2.9 & 6.5 & 12.8 & 2.56 & 3.92 & 7.37 & 127.3 & 227.2 & 425.9\\
Mixture rescaling & 3.2 & 7.0 & 13.5 & 2.53 & 3.92 & 7.36 & 127.5 & 227.1 & 425.4\\
Mixture inversion & 3.2 & 7.0 & 13.6 & 2.53 & 3.91 & 7.36 & 127.4 & 226.4 & 425.3\\
Gaussian tilting & 3.0 & 6.5 & 13.1 & 2.55 & 3.92 & 7.36 & 127.4 & 226.8 & 425.3\\
Cubic tilting & 4.9 & 9.8 & 17.4 & 2.49 & 3.92 & 7.36 & 127.5 & 226.7 & 425.7\\
\bottomrule
\end{tabular}
\end{table}

\begin{table}[htbp]
\centering
\caption{Measured H100 training and inference times. Each pair of columns
lists $d=5$ and $d=10$. Training times are minutes; median forward times
are seconds. References are computed once per input law and reused across
its five sampling replicates.}
\label{tab:experimental_h100_runtime}
\small
\setlength{\tabcolsep}{4pt}
\begin{tabular}{lrrrrrr}
\toprule
& \multicolumn{2}{c}{Training (min)}
& \multicolumn{2}{c}{$N=65{,}536$ (s)}
& \multicolumn{2}{c}{$M=500{,}000$ (s)}\\
Operator & $d=5$ & $d=10$ & $d=5$ & $d=10$ & $d=5$ & $d=10$\\
\midrule
Torus doubling & 8.7 & 16.7 & 2.08 & 2.13 & 127.3 & 128.8\\
Gaussian rescaling & 15.8 & 29.3 & 2.66 & 3.20 & 161.8 & 194.8\\
Gaussian inversion & 17.2 & 36.6 & 2.67 & 3.21 & 162.3 & 193.9\\
Mixture rescaling & 19.5 & 39.6 & 2.65 & 3.19 & 160.0 & 192.3\\
Mixture inversion & 26.2 & 53.3 & 2.68 & 3.17 & 162.6 & 191.4\\
Gaussian tilting & 20.0 & 39.9 & 2.65 & 3.21 & 159.4 & 192.4\\
Cubic tilting & 23.4 & 47.1 & 2.67 & 3.21 & 160.7 & 193.5\\
\bottomrule
\end{tabular}
\end{table}

\subsection{Numerical Distances}
\label{app:experimental_metrics}

\paragraph{One-dimensional Wasserstein distance.}
For measures $\alpha,\beta$ on $\bbR$ with distribution functions
$F_\alpha,F_\beta$, the quantile representation gives
\citep{manole2022minimax}
\begin{equation}\label{eq:experimental_quantile_w2}
    W_2^2(\alpha,\beta)=\int_0^1
       |F_\alpha^{-1}(u)-F_\beta^{-1}(u)|^2\,du.
\end{equation}
This reduces the transport problem to sorting. The Euclidean results use
a fixed quadrature grid $u_j=(j-1/2)/J$, $J=100{,}000$, and
\begin{equation}\label{eq:experimental_quantile_approximation}
    \widehat W_{2,J}(\alpha_n,\beta_m)
    =\left[\frac1J\sum_{j=1}^J
       |\widetilde F_{\alpha_n}^{-1}(u_j)
        -\widetilde F_{\beta_m}^{-1}(u_j)|^2\right]^{1/2}.
\end{equation}
The tildes denote linear interpolation between the sorted observations:
if $t=(n-1)u$, $k=\lfloor t\rfloor$ and $a=t-k$, then
$\widetilde F_{\alpha_n}^{-1}(u)=(1-a)x_{(k+1)}+a x_{(k+2)}$,
with the endpoint value defined by continuity.
All particles enter the order statistics, while the common grid avoids a
pairwise transport calculation for unequal ensemble sizes. This is a
numerical approximation to the empirical quantile integral, rather than its
exact step-function integration.

For $\mathbb T^1$, linear quantile matching would depend on the choice of
origin. We instead minimize over the relative shift of the lifted quantiles,
\begin{equation}\label{eq:experimental_circle_w2}
    W_{2,\mathbb T}^2(\alpha,\beta)
    =\inf_{a\in\bbR}\int_0^1
        |F_\alpha^{-1}(u)-F_\beta^{-1}(u-a)|^2\,du,
\end{equation}
where $F^{-1}(u+1)=F^{-1}(u)+1$. The circular transport algorithm of
\citet{delon2010circle}, implemented in POT \citep{flamary2021pot}, evaluates
the empirical integral using all particles with optimization tolerance
$10^{-7}$. This periodic $W_2$ is also the one-dimensional curve shown in
the sliced-distance figures for torus doubling.

\paragraph{Sliced Wasserstein distance.}
For measures on $\bbR^k$, let $P_\theta(x)=\langle\theta,x\rangle$ and
let $\sigma$ be the uniform probability measure on $\mathbb S^{k-1}$. Then
\citep{bonneel2015sliced}
\begin{equation}\label{eq:experimental_sliced_distance}
    \operatorname{SW}_2^2(\alpha,\beta)
    =\int_{\mathbb S^{k-1}}
        W_2^2((P_\theta)_\sharp\alpha,(P_\theta)_\sharp\beta)
      \,\sigma(d\theta).
\end{equation}
We approximate the angular integral by an average over $256$ independent
uniform directions and use \eqref{eq:experimental_quantile_approximation}
for every projection. Here $k=d$ for Euclidean data and $k=2d$ for torus
data embedded by $\phi$. Thus the torus sliced distance measures differences
between embedded distributions, whereas its Wasserstein and Sinkhorn
distances use the intrinsic torus geometry. In Euclidean dimension one,
both projection directions give the same value, so the sliced
calculation is precisely the one-dimensional quantile approximation above.

\paragraph{Sinkhorn approximation.}
For discrete measures $\alpha=\sum_i a_i\delta_{x_i}$ and
$\beta=\sum_j b_j\delta_{y_j}$, define
\begin{equation}\label{eq:experimental_regularized_ot}
    \operatorname{OT}_\varepsilon(\alpha,\beta)
    =\min_{\pi\in\Pi(a,b)}
       \left\{\sum_{i,j}C_{ij}\pi_{ij}
             +\varepsilon\sum_{i,j}\pi_{ij}
                    \log\frac{\pi_{ij}}{a_i b_j}\right\}.
\end{equation}
Here $\Pi(a,b)$ denotes nonnegative matrices with marginal constraints
$\sum_j\pi_{ij}=a_i$ and $\sum_i\pi_{ij}=b_j$.
The ground cost is $C_{ij}=\tfrac12\|x_i-y_j\|_2^2$ for Euclidean data and
\begin{equation}
    C_{ij}=\tfrac12\sum_{q=1}^d
       \min\{|x_{iq}-y_{jq}|,1-|x_{iq}-y_{jq}|\}^2
\end{equation}
for torus data. The particle weights are uniform. Sinkhorn iteration obtains
the regularized coupling by alternately enforcing its two marginals
\citep{cuturi2013sinkhorn}.
Writing $K_{ij}=\exp(-C_{ij}/\varepsilon)$, its scaling updates are
\begin{equation}
    u\leftarrow a/(Kv),\qquad v\leftarrow b/(K^\top u),\qquad
    \pi=\operatorname{diag}(u)K\operatorname{diag}(v),
\end{equation}
where division is componentwise. The numerical solvers use stabilized updates and anneal the regularization
parameter; the lower-dimensional solver also uses spatial multiscale acceleration.
We remove the self-interaction bias and report
\begin{align}
    S_\varepsilon(\alpha,\beta)
       &=\operatorname{OT}_\varepsilon(\alpha,\beta)
        -\tfrac12\operatorname{OT}_\varepsilon(\alpha,\alpha)
        -\tfrac12\operatorname{OT}_\varepsilon(\beta,\beta),\label{eq:experimental_sinkhorn_divergence}\\
    D_\varepsilon(\alpha,\beta)
       &=\sqrt{2\max\{S_\varepsilon(\alpha,\beta),0\}}.
       \label{eq:experimental_sinkhorn_distance}
\end{align}
The factor $2$ accounts for the half-squared ground cost
\citep{feydy2019sinkhorn}. We use GeomLoss with blur
$0.01$, hence $\varepsilon=10^{-4}$, and successive blur ratio $0.8$.
For $d=2,3$, the multiscale backend uses kernel truncation at five blur
lengths. For $d=5,10$, the online backend evaluates pairwise kernel
reductions through PyKeOps without constructing a dense transport matrix
or using spatial clustering; the multiscale truncation setting does not
apply to this backend. Both use the same ground costs and debiased
objective, and all particles contribute. For fixed discrete measures, $D_\varepsilon$ approaches
$W_2$ as $\varepsilon\downarrow0$, whereas fixed regularization and solver
accuracy can affect the slopes measured as $N$ increases. We therefore
treat $D_\varepsilon$ as a numerical surrogate and apply the finite-reference
triangle inequality only to the Wasserstein distances.

\subsection{Evaluation and Rate Fitting}
\label{app:experimental_rate_fits}

The three seed roles are $s_{\mathrm{train}}=7$,
$s_\mu\in\{0,1,2\}$, and
$s_{\mathrm{sample}}\in\{0,\ldots,R_d-1\}$, with $R_d$ specified in
Section~\ref{subsec:experimental_protocol}.
For each trained model, each input-measure seed selects one law from the
family in Appendix~\ref{app:experimental_generators}; its true output is
specified by the fixed target operator. The law is selected on the CPU
in an isolated random stream and its full-precision parameters are saved.
The particle stream is determined by the pair
$(s_\mu,s_{\mathrm{sample}})$, independently of the law-selection stream.
The implementation uses distinct SHA-256 namespaces for law selection and
particle sampling, so the repeated numeric seed labels identify different
streams. Restart state retains both distribution parameters and random
states. Changing a sampling seed never changes the selected law.

For each of the three input measures, one model reference of
$M=500{,}000$ particles is generated with reference seed
$s_{\mathrm{ref}}=s_\mu\in\{0,1,2\}$, in an isolated random stream.
Its predicted output is cached and reused by all $R_d$ sampling seeds at
all $N_i=2^{i+7}$, $i=1,\ldots,9$. Reference generation does not advance
the sampling stream; loading the cache therefore gives the same subsequent
particle draws as generating it. Each sampling replicate draws $100{,}000$
independent true-output particles, held fixed over $N$ within that replicate. Torus experiments use nested prefixes of one input ensemble;
Euclidean experiments draw a fresh input ensemble at each $N$.
Projection directions are drawn per replicate on the torus and per
replicate and particle count in Euclidean space. Both metrics and both
target comparisons use the same predicted ensembles within a replicate.

Let $e_{kri}$ denote the reference error for input-measure seed $k$,
sampling seed $r$, and particle count $N_i$, and set $x_i=\log N_i$.
We retain a separate ordinary least-squares fit with an intercept for
each sampling replicate:
\begin{equation}\label{eq:experimental_ols_exponent}
    \widehat b_{kr}
    =-\frac{\sum_i(x_i-\bar x)(\log e_{kri}-\overline{\log e_{kr}})}
            {\sum_i(x_i-\bar x)^2}.
\end{equation}
All nine particle counts receive equal weight. First average within each
measure, then summarize across the three measures:
\begin{equation}
    \widehat b_k=\frac1{R_d}\sum_{r=0}^{R_d-1}\widehat b_{kr},\qquad
    \bar b=\frac13\sum_{k=0}^2\widehat b_k,\qquad
    \operatorname{SE}(\bar b)
    =\frac1{\sqrt3}\left[\frac12\sum_{k=0}^2
                         (\widehat b_k-\bar b)^2\right]^{1/2}.
\end{equation}
Tables report $\bar b\pm\operatorname{SE}(\bar b)$, using the three
measure-level rates as the statistical units. This SE includes residual
particle-sampling variation within each rate, but excludes training
variability and systematic effects from finite particle counts, references,
and approximate metrics.

Figures use only the prespecified input-measure seed $k=0$. Their curves
are $\bar e_i=\frac1{R_d}\sum_{r=0}^{R_d-1}e_{0ri}$, with shading
$\bar e_i\pm s_{e,i}/\sqrt{R_d}$, where $s_{e,i}$ is the sample standard
deviation with denominator $R_d-1$. These sampling bands are conditional on
the one shared reference for that law; they do not estimate reference
sampling uncertainty. The same averaging is used for the
secondary true-operator comparison.
Appendix~\ref{app:experimental_sampling_repeats} compares sampling
counts in $d=1,2,3$.

For a prescribed guide $r(N)$, only its vertical position is fitted:
\begin{equation}\label{eq:experimental_guide_normalization}
    \log\widehat C=\frac1{9}\sum_{i=1}^9
       [\log\bar e_i-\log r(N_i)].
\end{equation}
The two-dimensional $W_2$ guide is $r_2(N)=(\log N/N)^{1/2}$.
Its effective power-law exponent over the displayed
grid is obtained by applying the same least-squares calculation:
\begin{equation}\label{eq:experimental_d2_effective_exponent}
    b_{2,\mathrm{grid}}
    =\frac12-\frac12
       \frac{\sum_i(x_i-\bar x)\log x_i}{\sum_i(x_i-\bar x)^2}
    =0.43814\ldots.
\end{equation}
Thus this table entry is a finite-grid summary of the logarithmically
corrected guide. The asymptotic algebraic exponent is $1/2$.
The other Wasserstein guide exponents are $1/2$ for $d=1$ and $1/d$ for
$d\geq3$; the sliced-distance guide has exponent $1/2$ in every dimension.

The fitted exponent can depend on the range of particle counts. The
saved summaries also apply the same per-repeat procedure to $N\leq8192$
and $N\geq8192$, with the same within-measure averaging and across-measure
SE. These sensitivity summaries do not estimate the reference bias itself.

At $N=65{,}536$, the benchmark ratios $r_d(M)/r_d(N)$ are approximately
$0.36$, $0.39$, $0.51$, $0.67$, and $0.82$ in dimensions
$d=1,2,3,5,10$, respectively. Although
these ratios do not measure the actual reference error, they indicate
that the reference contribution may affect the largest-$N$ observations.
The fitted exponents describe the tested finite range; assessing the
remaining reference bias requires larger-reference comparisons. Errors
relative to the analytic output also contain finite-target sampling error
and, in higher dimensions, regularization or projection error.

\subsection{Scope of the Comparisons}
\label{app:experimental_theory_scope}

Theorem~\ref{thm:transformer_stack_limit} transfers an input empirical
Wasserstein rate to the fixed transformer, up to the additional
$N^{-1/2}$ term. It does not supply a distribution-free sharp rate for
$\bbE W_2(\mu^N,\mu)$. The rates in
\eqref{eq:experimental_dimension_rates} are regular-distribution benchmarks;
compact-manifold results such as \citet{borda2023empirical} require density
assumptions, and unbounded laws can introduce further tail corrections.
For the sliced rate, a sufficient regularity condition is
\begin{equation}\label{eq:experimental_sliced_regularity}
    \int_{\mathbb S^{k-1}}\int_{\bbR}
       \frac{F_\theta(t)(1-F_\theta(t))}{f_\theta(t)}\,dt\,
       \sigma(d\theta)<\infty,
\end{equation}
where $F_\theta$ and $f_\theta$ are the distribution function and density
of the projected law. The integrand is defined as zero when both numerator
and denominator vanish, and as $+\infty$ when only the denominator vanishes.
Under this condition, the empirical $\operatorname{SW}_2$ rate is $N^{-1/2}$
\citep{manole2022minimax}. Sharp rates can include logarithmic corrections
when the condition fails, as for unbounded Gaussian projections
\citep{nietert2022statistical}. This condition is not asserted for all
learned output laws. A Wasserstein bound alone does not give the sharper
$N^{-1/2}$ sliced rate in higher dimensions.

The compact torus input avoids tail assumptions. To transfer the compact
stability estimate through its angle readout, the predicted coordinate pairs
must remain away from zero. For ordinary dot-product attention with a Gaussian
law after the input projection, both smallness conditions in
\eqref{eq:self_small_theta_sufficient} are sufficient for the noncompact
theorem. Training does not enforce these inequalities. The cubic family
satisfies both input conditions, as verified below. Strong sub-Gaussianity
still requires separate verification for unequal-covariance Gaussian
mixtures. The Gaussian smallness criterion for bilinear scores does not directly
apply to the query--key-normalized Gaussian-tilting variants described above.
These experiments provide evidence about the trained maps without
certifying every remaining hypothesis for every input family.

\paragraph{Strong sub-Gaussianity of the cubic family.}
\label{app:experimental_cubic_ssg}
For the one-dimensional potential $V(z)=|z|^3$, set
\[
 W(z)=\begin{cases}\tfrac32z^2-\tfrac12,&|z|\leq1,\\
                    |z|^3,&|z|>1.\end{cases}
\]
Then $W$ is continuously differentiable, $W''\geq3$ almost everywhere,
and $0\leq V-W\leq1/2$. Under any linear tilt $t$, the law
$\nu_t(dz)\propto e^{-W(z)+tz}dz$ has variance at most $1/3$:
integration by parts gives $\operatorname{Cov}_{\nu_t}(Z,W'(Z))=1$,
while strong convexity gives
$\operatorname{Cov}_{\nu_t}(Z,W'(Z))\geq3\operatorname{Var}_{\nu_t}(Z)$.
For $\mu_t(dz)\propto e^{-V(z)+tz}dz$, the density ratio
$d\mu_t/d\nu_t\leq e^{1/2}$. Centering at the mean of $\nu_t$ therefore
gives $\operatorname{Var}_{\mu_t}(Z)\leq e^{1/2}/3$, uniformly in $t$.
Since this variance is the second derivative of the log moment-generating
function, the cubic law is strong sub-Gaussian. Independence and the affine
map $X=m+C\xi$ give the bound
\[
    \|\nabla^2\psi_X(s)\|_2\leq(e^{1/2}/3)\|C\|_2^2
    \quad\text{for every }s.
\]
The faster-than-Gaussian tail condition is immediate from the cubic
potential, and both properties are preserved by the affine input lift.
The same proof applies to the random-input construction above: adding
$Az^2/2+Bz$, $A\geq0$, increases the convexity bound to $3+A$ and leaves
the bounded difference $V-W$ unchanged. Thus these particular noncompact
inputs meet the theorem's input hypotheses without matrix smallness.

\subsection{Sensitivity to Sinkhorn Settings}
\label{app:experimental_sinkhorn_ablation}

We perform a paired ablation using the same cubic-tilting model in $d=3$
and input-measure seed $0$ as in the main experiment. Holding the selected law
fixed, we use sampling seed $20260916$ to generate one learned reference
with $M=16{,}384$ particles and three independent input resamplings at each
$N\in\{256,512,1024,2048,4096\}$. All seven metric settings use exactly
the same saved model outputs and the same reference. Inference uses
full native attention in single precision, without autocast or TF32.
The baseline uses blur $0.01$, annealing ratio $0.8$, and truncation at
five blur lengths. Each alternative changes one of these settings.
Pairing the outputs isolates changes due to the distance calculation;
the smaller grid and reference give different finite-range exponents
from those in the main experiment.

\begin{table}[htbp]
\centering
\caption{Paired Sinkhorn ablation for cubic tilting in $d=3$ and
input-measure seed $0$. Exponents are mean $\pm$ sample standard deviation
of three per-realization fits. The last column gives the maximum absolute
relative change from the baseline over all $15$ paired errors. All settings
use the multiscale backend and the same debiased distance.}
\label{tab:experimental_sinkhorn_ablation}
\footnotesize
\begin{tabular}{lrrrrr}
\toprule
Setting & Blur & Scaling & Truncate & Exponent & Max. error change\\
\midrule
Baseline & $0.01$ & $0.8$ & $5$ & $0.3075 \pm 0.0084$ & $0.000\%$\\
Smaller blur & $0.005$ & $0.8$ & $5$ & $0.3069 \pm 0.0083$ & $0.397\%$\\
Larger blur & $0.02$ & $0.8$ & $5$ & $0.3106 \pm 0.0086$ & $1.672\%$\\
Largest blur & $0.05$ & $0.8$ & $5$ & $0.3334 \pm 0.0098$ & $12.080\%$\\
Slower annealing & $0.01$ & $0.9$ & $5$ & $0.3049 \pm 0.0084$ & $1.210\%$\\
Shorter truncation & $0.01$ & $0.8$ & $3$ & $0.3075 \pm 0.0084$ & $<0.001\%$\\
Longer truncation & $0.01$ & $0.8$ & $8$ & $0.3075 \pm 0.0084$ & $<0.001\%$\\
\bottomrule
\end{tabular}
\end{table}

Changing blur from $0.01$ to $0.005$ or $0.02$ gives exponents
$0.3069$ and $0.3106$, compared with the baseline $0.3075$, and changes
individual errors by at most $0.40\%$ and $1.67\%$, respectively.
The larger blur $0.05$ changes errors by up to $12.08\%$ and increases
the exponent to $0.3334$. Increasing the annealing ratio from $0.8$
to $0.9$ gives $0.3049$, with at most $1.21\%$ change in error.
Changing truncation from five blur lengths to three or eight has an
effect below $0.001\%$ in this example. For this model and input law,
halving or doubling the baseline blur therefore has a modest effect on
the measured decay, while stronger regularization changes both the
distances and their fitted exponent. The comparison is conditional on
one learned reference and does not estimate the difference between
Sinkhorn and unregularized $W_2$.

\begingroup
\makeatletter
\setlength{\@fpsep}{8pt}
\makeatother
\begin{figure}[!htbp]
    \centering
    \captionsetup{skip=5pt}
    \includegraphics[width=\textwidth,trim=0bp 12bp 0bp 8bp,clip]{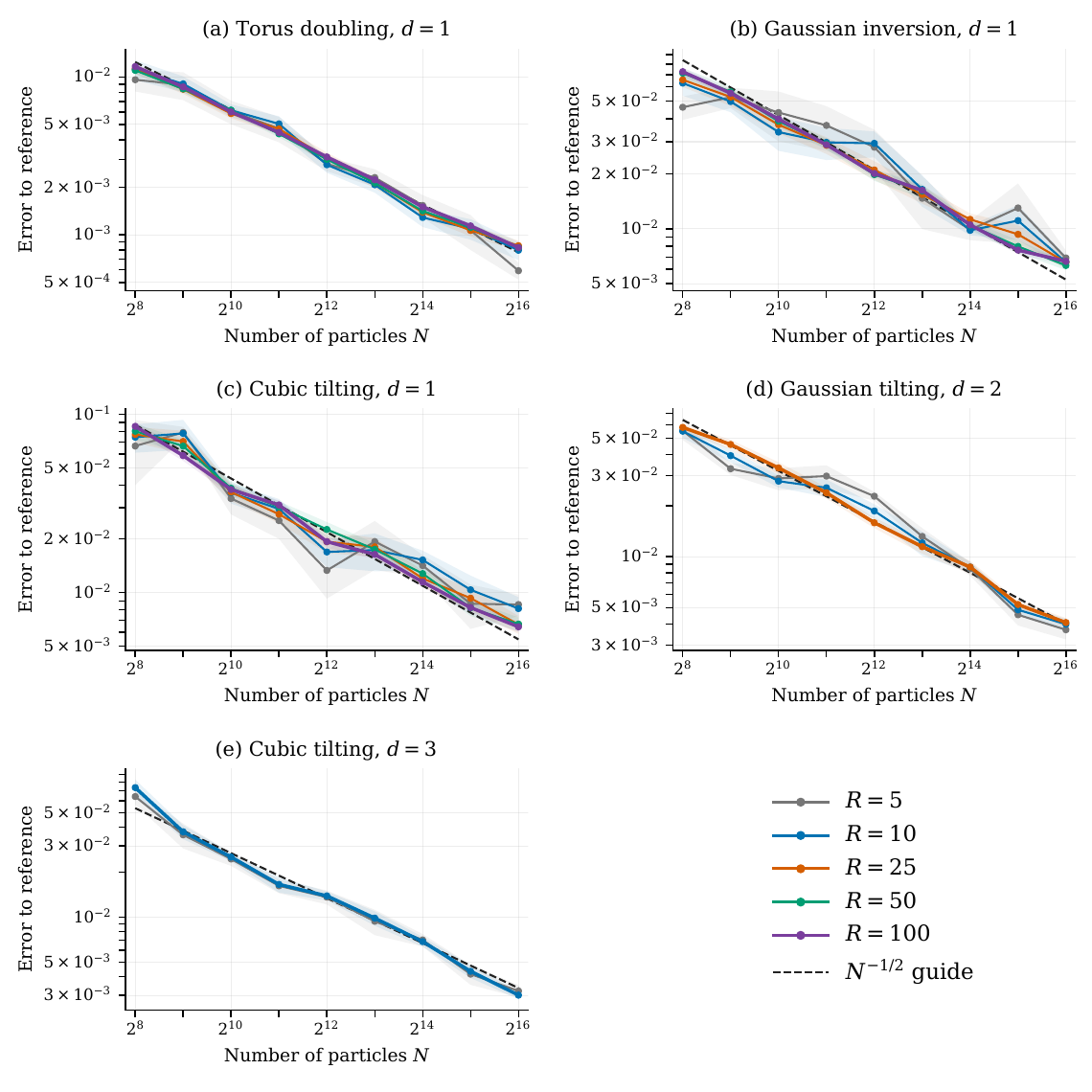}
    \caption{Mean reference errors with one sampling standard error for
    $R=5,10,25,50,100$ in panels (a)--(c), $R=5,10,25$ in panel (d), and
    $R=5,10$ in panel (e). Dashed $N^{-1/2}$ guides are fitted vertically
    to the curve with the largest $R$.}
    \label{fig:experimental_sampling_repeats}
\end{figure}

\subsection{Effect of the Number of Sampling Replicates}
\label{app:experimental_sampling_repeats}

Figure~\ref{fig:experimental_sampling_repeats} compares torus doubling,
Gaussian inversion, and cubic tilting in $d=1$ with
$R\in\{5,10,25,50,100\}$, Gaussian tilting in $d=2$ with
$R\in\{5,10,25\}$, and cubic tilting in $d=3$ with $R\in\{5,10\}$.
Each comparison uses input-measure seed $0$, particle counts
$N=2^8,\ldots,2^{16}$, and a fixed trained model and
$M=500{,}000$-particle learned reference. The metrics are the one-dimensional Wasserstein distances for $d=1$
and $\mathrm{SW}_2$ for $d=2,3$, as specified in
Subsection~\ref{subsec:experimental_protocol}.

For reference error $e_r(N)$ in realization $r$, the first $R$ sampling
seeds give
\[
    \overline e_R(N)=\frac1R\sum_{r=0}^{R-1}e_r(N),\qquad
    \operatorname{SE}_R(N)=
    \frac{1}{\sqrt R}
    \left[\frac1{R-1}\sum_{r=0}^{R-1}
    \bigl(e_r(N)-\overline e_R(N)\bigr)^2\right]^{1/2}.
\]
The average estimates the expected numerical reference error. Its
standard error is conditional on the input law, trained model, and
reference, and includes randomized projections for $\mathrm{SW}_2$.
We summarize relative uncertainty by the median of
$\operatorname{SE}_R(N)/\overline e_R(N)$ over the nine particle counts.

In $d=1$, increasing $R$ from $5$ to $100$ reduces the median relative
standard error from $16.2\%$ to $3.6\%$, from $24.5\%$ to $4.9\%$, and
from $21.0\%$ to $6.6\%$ for doubling, Gaussian inversion, and cubic tilting,
respectively. The corresponding changes are $12.0\%$ to $7.0\%$ for
Gaussian tilting in $d=2$ ($R=5$ to $25$) and $10.7\%$ to $9.8\%$ for
cubic tilting in $d=3$ ($R=5$ to $10$). Five realizations already capture
the decay with $N$; larger sampling budgets reduce fluctuations in the
estimated mean.

The estimates for different $R$ share realizations, and their standard
errors need not decrease at every increment of $R$.
\endgroup

\bibliography{references}

\end{document}